\documentclass[11pt]{article}

\usepackage[utf8]{inputenc} 
\usepackage[T1]{fontenc}    
\usepackage{url}
\usepackage[hidelinks]{hyperref}
\usepackage{amsmath,amsthm,amssymb,amsbsy,amsfonts,amscd,bm}
\usepackage{paralist}
\usepackage{xcolor}
\usepackage{color}
\usepackage{graphicx}
\graphicspath{{./figs/}}
\usepackage{algorithm}
\usepackage{algorithmic}
\usepackage{comment}
\usepackage{multirow}
\usepackage{enumitem}
\usepackage{fancyhdr}
\usepackage{authblk}
\usepackage{cleveref}
\usepackage{subcaption}
\usepackage{wrapfig}
\usepackage[toc, page]{appendix}
\usepackage[numbers]{natbib} 

\newtheorem{lemma}{Lemma}
\newtheorem{prop}{Proposition}
\newtheorem{definition}{Definition}

\newtheorem{theorem}{Theorem}

\theoremstyle{remark}

\newcommand{\R}{\mathbb{R}}

\newcommand{\e}{\begin{equation}}
\newcommand{\ee}{\end{equation}}
\newcommand{\en}{\begin{equation*}}
\newcommand{\een}{\end{equation*}}
\newcommand{\eqn}{\begin{eqnarray}}
\newcommand{\eeqn}{\end{eqnarray}}
\newcommand{\bmat}{\begin{bmatrix}}
\newcommand{\emat}{\end{bmatrix}}

\DeclareMathAlphabet\mathbfcal{OMS}{cmsy}{b}{n}

\hypersetup{
    colorlinks=true,%
    citecolor=blue,%
    filecolor=blue,%
    linkcolor=blue,%
    urlcolor=blue
}

\newcommand{\mO}{\mtx{O}}

\graphicspath{{./figs/}}

\newlength{\imgwidth}
\newboolean{twoColVersion}
\setboolean{twoColVersion}{false}
\newcommand{\twoCol}[2]{\ifthenelse{\boolean{twoColVersion}} {#1} {#2} }

\newtheorem{proposition}{\bf{Proposition}}

\long\def\comment#1{}

\newcommand{\xmath}[1] {\ensuremath{#1}\xspace}
\newcommand{\blmath}[1] {\xmath{\bm{#1}}}

\newcommand{\bD}{\boldsymbol{D}}

\newcommand{\bI}{\boldsymbol{I}}

\newcommand{\bP}{\boldsymbol{P}}
\newcommand{\bQ}{\boldsymbol{Q}}

\newcommand{\bU}{\boldsymbol{U}}
\newcommand{\bV}{\boldsymbol{V}}

\newcommand{\bX}{\boldsymbol{X}}

\newcommand{\bZ}{\boldsymbol{Z}}

\def\b{{\blmath b}}

\long\def\red#1{\bgroup\color{red}#1\egroup}

\definecolor{mich-blue}{HTML}{0027CC}
\definecolor{mich-blue-high}{HTML}{0027CC}
\definecolor{red-high}{HTML}{CA2020}
\definecolor{green-high}{HTML}{20A520}
\definecolor{mich-maize}{HTML}{FFCB05}
\definecolor{law-stone}{HTML}{655A52}
\definecolor{burton-beige}{HTML}{9B9A9D}
\definecolor{arch-ivy}{HTML}{7E732F}

 \colorlet{color1}{gray!15}

\usepackage{amssymb}
\usepackage{amsmath,amsfonts,amsthm}
\usepackage{latexsym}
\usepackage{mathrsfs}
\usepackage{color}
\usepackage{algorithm}
\usepackage{algorithmic}

\usepackage{hyperref}       
\hypersetup{
	colorlinks=true,
	linkcolor=blue,
	filecolor=magenta,
	urlcolor=cyan,
}
\usepackage{url}            
\usepackage{booktabs}       

\usepackage[top=1in, bottom=1in, left=1in, right=1in]{geometry}
\usepackage{bm}
\usepackage{graphicx}
\usepackage{makecell}
\usepackage{natbib}
 
\newcommand{\beqa}{\begin{eqnarray}}
\newcommand{\eeqa}{\end{eqnarray}}
\newcommand{\beqas}{\begin{eqnarray*}}
\newcommand{\eeqas}{\end{eqnarray*}}
\newcommand{\ea}{\end{array}}
\newcommand{\ei}{\end{itemize}}

\newcommand{\RN}[1]{%
  \textup{\uppercase\expandafter{\romannumeral#1}}%
}
\newcommand{\email}[1]{\protect\href{mailto:#1}{#1}}

\newcounter{spb}
\def\b0{\bm{0}}
\def\b1{\bm{1}}

\def\bD{\bm{D}}

\def\bI{\bm{I}}

\def\bP{\bm{P}}

\def\bQ{\bm{Q}}

\def\bU{\bm{U}}
\def\bV{\bm{V}}

\def\bX{\bm{X}}

\def\bZ{\bm{Z}}

\def\mO{\mathcal{O}}

\def\R{\mathbb{R}}
\def\S{\mathbb{S}}

\usepackage{tcolorbox}

\hypersetup{
colorlinks=true,%
citecolor=blue,%
filecolor=blue,%
linkcolor=blue,%
urlcolor=blue
}
\usepackage{footmisc}

\usepackage{titlesec}

\makeatletter
\renewcommand{\@biblabel}[1]{[#1]\hfill}
\makeatother

\begin{document}
\title{A Global Geometric Analysis of Maximal Coding Rate Reduction\thanks{This work has been accepted for publication in the Proceedings of the 41st International Conference on Machine Learning (ICML 2024). Correspondence to: Peng Wang (\email{peng8wang@gmail.com})}}  

\author[1]{Peng Wang}
\author[2]{Huikang Liu}
\author[3]{Druv Pai}
\author[3]{Yaodong Yu}
\author[4]{\\ Zhihui Zhu}
\author[1]{Qing Qu}
\author[3,5]{Yi Ma}

\affil[1]{\small Department of Electrical Engineering and Computer Science, University of Michigan, Ann Arbor}
\affil[2]{\small Antai College of Economics and Management, Shanghai Jiao Tong University, Shanghai}
\affil[3]{\small Department of Electrical Engineering and Computer Science, University of California, Berkeley}
\affil[4]{\small Department of Computer Science and Engineering, Ohio State University, Columbus}
\affil[5]{\small Institute of Data Science, University of Hong Kong, Hong Kong}

\date{\today}
\maketitle

\begin{abstract}
The maximal coding rate reduction (MCR$^2$) objective for learning structured and compact deep representations is drawing increasing attention, especially after its recent usage in the derivation of fully explainable and highly effective deep network architectures. However, it lacks a complete theoretical justification: only the properties of its global optima are known, and its global landscape has not been studied. In this work, we give a complete characterization of the properties of all its local and global optima, as well as other types of critical points. Specifically, we show that each (local or global) maximizer of the MCR$^2$ problem corresponds to a low-dimensional, discriminative, and diverse representation, and furthermore, each critical point of the objective is either a local maximizer or a strict saddle point. Such a favorable landscape makes MCR$^2$ a natural choice of objective for learning diverse and discriminative representations via first-order optimization methods. To validate our theoretical findings, we conduct extensive experiments on both synthetic and real data sets. 

\end{abstract}


\tableofcontents

\section{Introduction}\label{sec:intro}
\subsection{Background and Motivation}

In the past decade, deep learning has exhibited remarkable empirical success across a wide range of engineering and scientific applications \cite{lecun2015deep}, such as computer vision \cite{he2016deep,simonyan2014very}, natural language processing \cite{sutskever2014sequence,vaswani2017attention}, and health care \cite{esteva2019guide}, to name a few. As argued by \citet{bengio2013representation,ma2022principles}, one major factor contributing to the success of deep learning is the ability of deep networks to perform powerful nonlinear feature learning by converting the data distribution to a {\em compact} and {\em structured} representation. This representation greatly facilitates various downstream tasks, including classification \cite{dosovitskiy2020image}, segmentation \cite{kirillov2023segment}, and generation \cite{saharia2022photorealistic}.  

Based on the theory of data compression and optimal coding \citep{ma2007segmentation}, \citet{chan2022redunet,yu2020learning} proposed a principled and unified framework for deep learning to learn a compact and structured representation. Specifically, they proposed to maximize the difference between the {\em coding rate} of all features and the sum of coding rates of features in each class, which is referred to as {\em maximal coding rate reduction} (MCR$^2$). This problem is presented in Problem \eqref{eq:MCR1} and visualized in \Cref{fig:MCR}(a). Here, the coding rate measures the ``compactness" of the features, which is interpreted as the volume of a particular set spanned by the learned features: a lower coding rate implies a more compact feature set\footnote{Please refer to \citet[Section 2.1]{chan2022redunet} for more details on measuring compactness of feature sets via coding rates.}. Consequently, the MCR$^2$ objective aims to maximize the volume of the set of all features while minimizing the volumes of the sets of features from each class. Motivated by the structural similarities between deep networks and unrolled optimization schemes for sparse coding \cite{gregor2010learning,monga2021algorithm}, \citet{chan2022redunet} constructed a new deep network based on an iterative gradient descent scheme to maximize the MCR$^2$ objective.\footnote{When performing maximization, we actually mean that we use gradient \textit{ascent}. However, we write gradient descent to maintain consistency with existing optimization literature.} Notably, each component of this deep network has precise optimization and geometric interpretations. Moreover, it has achieved strong empirical performance on various vision and language tasks \cite{chu2023image,yu2023sparse}.  

\begin{figure*}[t]
\begin{center}
	\begin{subfigure}{0.52\textwidth}
		\centering    	
    	\includegraphics[width=1\textwidth]{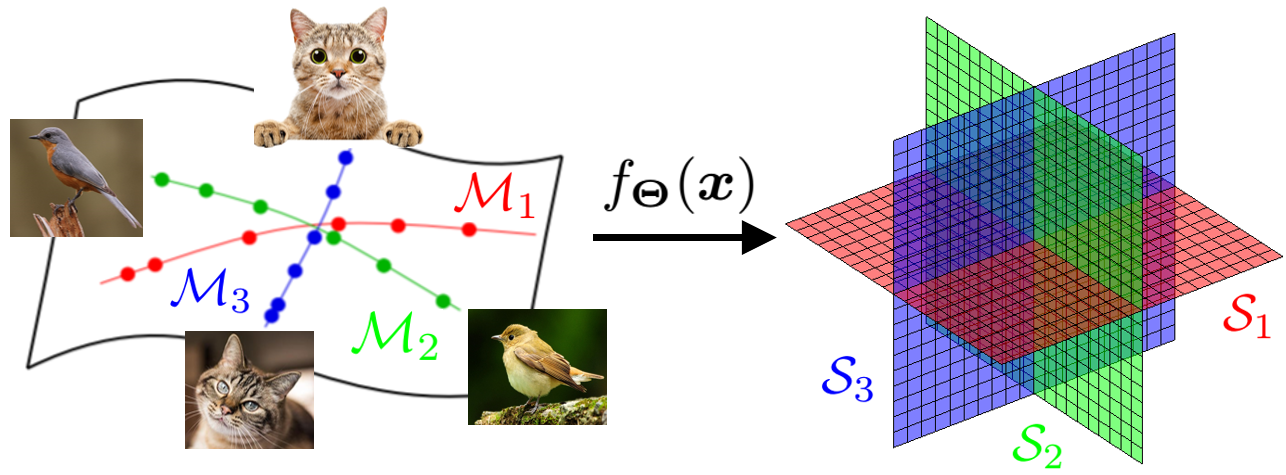} \vspace{0.01in}
    \caption{\bf Visualization of feature learning via MCR$^2$.} 
    \end{subfigure} \hspace{0.1in}
    \begin{subfigure}{0.44\textwidth}
    	\centering
    	\includegraphics[width = 0.7\linewidth]{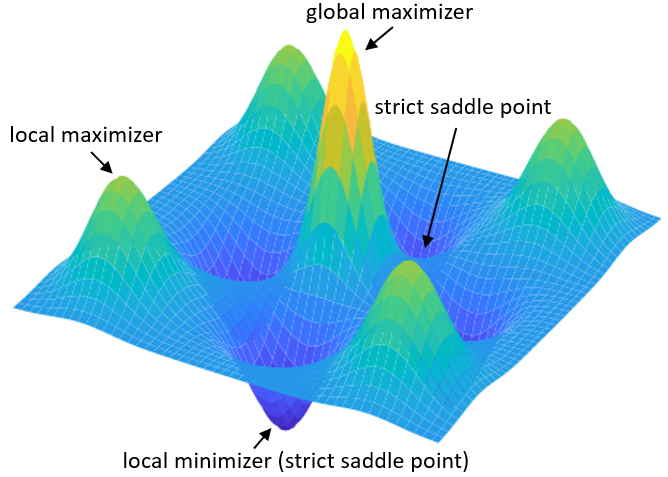} \vspace{-0.05in}
    \caption{\bf Ideal landscape of an MCR$^2$ problem.} 
    \end{subfigure}\vspace{-0.1in} 
    \caption{\textbf{An illustration of the properties of MCR$^2$}. (a) The high-dimensional data $\{\bm x_i\} \subseteq \mathbb{R}^n$ lies on a union of low-dimensional submanifolds. The objective of MCR$^2$ is to learn a feature mapping $f_{\bm \Theta}(\bm x) \in \mathbb{R}^d$ such that $\bm z_i = f_{\bm \Theta}(\bm x_i)$ for all $i$ are low-dimensional, discriminative, and diverse. (b) According to \Cref{thm:1} and \Cref{thm:2}, the regularized MCR$^2$ problem has a benign optimization landscape: each critical point is either a local maximizer or a strict saddle point. Furthermore, each local maximizer, just like the global maximizer,  corresponds to a feature representation that consists of a family of orthogonal subspaces, as illustrated in the middle.}
    \vspace{-0.2in} 
    \label{fig:MCR}
\end{center}
\end{figure*}

Although the MCR$^2$-based approach to deep learning is conceptually ``white-box" and has achieved remarkable empirical performance, its theoretical foundations have been relatively under-explored. In fact, the effective feature learning mechanism and “white-box” network architecture design based on MCR$^2$ are direct consequences of these foundations, and understanding them will pave the way to improving model interpretability and training efficiency of deep networks. Nevertheless, a comprehensive theoretical understanding of the MCR$^2$ problem remains lacking. In this work, we take a step towards filling this gap by studying its optimization properties. Notably, analyzing these properties, including local optimality and global landscape, of the MCR$^2$ objective is extremely challenging. To be precise, its objective function (see Problem \eqref{eq:MCR1}) is highly non-concave\footnote{We are maximizing the MCR$^2$ objective. Maximizing a concave function is equivalent to minimizing a convex function.} and complicated, as it involves quadratic functions and the difference between log-determinant functions. To the best of our knowledge, characterizing the local optimality and global optimization landscape of the MCR$^2$ problem remains an open question. 

\subsection{Our Contributions}\label{subsec:contri}

In this work, we study the optimization foundations of the MCR$^2$-based approach to deep learning. Towards this goal, we characterize the local and global optimality of the regularized MCR$^2$ problem and analyze its global optimization landscape  (see Problem \eqref{eq:MCR}). Our contributions can be highlighted as follows. 

\vspace{-0.1in}
\paragraph{Characterizing the local and global optimality.} For the regularized MCR$^2$ problem, we derive the closed-form expressions for its local and global optima for the first time. Our characterization shows that each local maximizer of the regularized MCR$^2$ problem is within-class compressible and between-class discriminative in the sense that features from the same class belong to a low-dimensional subspace, while features from different classes belong to different orthogonal subspaces. Besides these favorable properties, each global maximizer corresponds to a maximally diverse representation, which attains the highest possible dimension in the space.  

\vspace{-0.1in}
\paragraph{Studying the global optimization landscape.} Next, we show that the regularized MCR$^2$ function possesses a benign global optimization landscape, despite its complicated structures. More precisely, each critical point is either a local maximizer or strict saddle point of the regularized MCR$^2$ problem; see \Cref{fig:MCR}(b). Consequently, any gradient-based optimization, such as (stochastic) gradient descent, with random initialization can escape saddle points and at least converge to a local maximizer efficiently. 

Finally, we conduct extensive numerical experiments on synthetic data sets to validate our theoretical results. Moreover, we use the regularized MCR$^2$ objective to train deep networks on real data sets. These experimental results constitute an application of the rigorously derived MCR$^2$ theory to more realistic and complex deep learning problems.

Our results not only establish optimization foundations for the MCR$^2$ problem but also yield some important implications for the MCR$^2$-based approach to deep learning. Namely, our theoretical characterizations of local and global optimality offer a compelling explanation for the empirical observations that both deep networks constructed via gradient descent applied to the MCR$^2$ objective and over-parameterized deep networks trained by optimizing the MCR$^2$ objective learn low-dimensional, discriminative, and diverse representations. These results align with the motivations of \citet{chan2022redunet,yu2020learning} for employing the MCR$^2$ principle for deep learning, and elucidate the outstanding performance of MCR$^2$-based neural networks across a wide range of vision and language tasks \cite{chu2023image,yu2023emergence}. Moreover, our results underscore the potential of MCR$^2$-based approaches to serve as a cornerstone for future advancements in deep learning, offering a principled approach to pursuing structured and compact representations in practical applications.

\subsection{Related Work}\label{subsec:related} 

\paragraph{Low-dimensional structures in deep representation learning.} In the literature, it has long been believed that the role of deep networks is to learn certain (nonlinear) low-dimensional and informative representations of the data \cite{hinton2006reducing,ma2022principles}. For example, \citet{papyan2020prevalence} showed that the features learned by cross-entropy (CE) loss exhibit a neural collapse phenomenon during the terminal phase of training, where the features from the same class are mapped to a vector while the features from different classes are maximally linearly separable. \citet{ansuini2019intrinsic,recanatesi2019dimensionality} demonstrated that the dimension of the intermediate features first rapidly increases and then decreases from shallow to deep layers. \citet{masarczyk2023tunnel} concluded that the deep layers of neural networks progressively compress within-class features to learn low-dimensional features. Notably, \citet{wang2023understanding} proposed a theoretical framework to analyze hierarchical feature learning for learning low-dimensional representations. They showed that each layer of deep linear networks progressively compresses within-class features and discriminates between-class features in classification problems.   

\vspace{-0.1in} 
\paragraph{The MCR$^2$-based approach to deep learning.} The MCR$^2$-based approach to deep learning for seeking structured and compact representations was first proposed by \citet{yu2020learning}. Notably, they provided a global optimality analysis of the MCR$^2$ problem \eqref{eq:MCR1} with additional rank constraints on the feature matrix of each class. \citet{chan2022redunet} designed a new multi-layer deep network architecture, named ReduNet, based on an iterative gradient descent scheme for maximizing the MCR$^2$ objective. To learn self-consistent representations, \citet{dai2022ctrl} extended this approach to the closed-loop transcription (CTRL) framework, which is formulated as a max-min game to optimize a modified MCR$^2$ objective. This game was shown to have global equilibria corresponding to compact and structured representations \cite{pai2022pursuit}. Recently, \citet{yu2023white} showed that a transformer-like architecture named CRATE, which obtains strong empirical performance \cite{chu2023image,yu2023sparse,yu2023emergence}, can be naturally derived through an iterative optimization scheme for maximizing the sparse rate reduction objective, which is an adaptation to sequence data of the MCR$^2$ objective studied in this work.

\vspace{-0.1in}
\paragraph{Notation.} 
Given a matrix $\bm A \in \R^{m\times n}$, we use $\|\bm A\|$ to denote its spectral norm, $\|\bm{A}\|_{F}$ to denote its Frobenius norm, and $a_{ij}$ its $(i,j)$-th element. Given a vector $\bm a \in \R^d$, we use $\|\bm a\|$ to denote its $\ell_2$-norm, $a_i$ its $i$-th element, and $\mathrm{diag}(\bm a)$ the diagonal matrix with $\bm a$ on its diagonal. Given a positive integer $n$, we denote by $[n]$ the set $\{1,\ldots,n\}$. Given a set of integers $\{n_k\}_{k=1}^K$, let $n_{\max} = \max\{n_k: k \in [K]\}$. Let $\mathcal{O}^{m\times n} = \left\{\bm Z \in \R^{m\times n}: \bm Z^T\bm Z = \bm I_n \right\}$ denote the set of all $m\times n$ orthonormal matrices.

\section{Problem Setup}\label{sec:setup}

In this section, we first review the basic concepts of MCR$^2$ for deep representation learning in \Cref{subsec:MCR}, and then introduce our studied problem in \Cref{subsec:our}.  

\subsection{An Overview of MCR$^2$}\label{subsec:MCR}  

In deep representation learning, given data $\{\bm x_i\}_{i=1}^m \subseteq \R^n$ from multiple (say $K$) classes, the goal is to learn neural-network representations of these samples that facilitate downstream tasks. Recent empirical studies have shown that good features can be learned for tasks such as classification or autoencoding by using heuristics to promote either the contraction of samples in the same class \cite{rifai2011contractive} or the contrast of samples between different classes \cite{he2020momentum,oord2018representation} during the training of neural networks. Notably, \citet{chan2022redunet,yu2020learning} unified and formalized these practices and demonstrated that the MCR$^2$ objective is an effective objective to learn within-class compressible and between-class discriminative representations of the data.  

\vspace{-0.1in}
\paragraph{The formulation of MCR$^2$.} In this work, we mainly consider an MCR$^2$ objective for supervised learning problems. Specifically, let $\bm z_i = f_{\bm \Theta}(\bm x_i)$ for all $i \in [m]$ denote the features learned via the feature mapping $f_{\bm \Theta}(\cdot):\R^n \rightarrow \R^d$ parameterized by  $\bm \Theta$. 
For each $k \in [K]$, let $\bm \pi^k \in \{0,1\}^m$ be a label vector denoting membership of the samples in the $k$-th class, i.e., $\pi_i^k = 1$ if sample $i$ belongs to class $k$ and $\pi_i^k = 0$ otherwise for all $i \in [m]$, and $m_k := \sum_{i=1}^m \pi_i^k$ be the number of samples in the $k$-th class. 

For each $k \in [K]$, let $\bm{Z}_{k} \in \R^{d \times m_{k}}$ be the matrix whose columns are the features in the $k$-th class. Without loss of generality, we reorder the samples in a class-by-class manner, so that we can write the matrix of all features as
\begin{align}\label{eq:Z}
    \bm Z = \left[ \bm Z_1,\dots,\bm Z_K \right] \in \R^{d\times m}. 
\end{align}
On one hand, to make features between different classes discriminative or contrastive, one can maximize the lossy coding rate of all features in $\bm Z$, as argued in \cite{chan2022redunet,yu2020learning}, as follows:
\begin{align}\label{eq:R}
R(\bm Z) := \frac{1}{2}\log\det\left( \bm I + \frac{d}{m\epsilon^2}\bm Z\bm Z^T \right),  
\end{align}
where $\epsilon > 0$ is a prescribed quantization error.\footnote{Here, $R(\bm Z)$ is also known as the rate-distortion function in information theory \cite{cover1999elements}, which represents the average number of binary bits needed to encode the data $\bm Z$.}  On the other hand, to make features from the same class compressible or contractive, one can minimize the average lossy coding rate of features in the $k$-th class as follows: 
\begin{align}\label{eq:Rc1}
R_c\left(\bm Z; \bm \pi^k\right) 
& = \frac{m_k}{2m} \log\det\left( \bm I + \frac{d}{m_k\epsilon^2}\bm Z_k \bm Z_k^T \right).  
\end{align}
Consequently, a good representation tends to maximize the difference between the coding rate for the whole and that for each class as follows:   
\begin{align}\label{eq:MCR1}
    \max_{\bm Z \in \mathbb{R}^{d\times m}}\ R(\bm Z) - \sum_{k=1}^K R_c\left(\bm Z; \bm \pi^k \right)\quad \mathrm{s.t.}\ \|\bm Z_k\|_F^2 = m_k, \quad \forall k \in [K]. 
\end{align}
This is referred to as the principle of {\em maximal coding rate reduction} in \cite{chan2022redunet,yu2020learning}. It is worth mentioning that this principle can be extended to self-supervised and even unsupervised learning settings, where we learn the label vectors $\{\bm \pi^k\}_{k=1}^K$ during training. 

\subsection{The Regularized MCR$^2$ Problem}\label{subsec:our}
Due to the Frobenius norm constraints, it is a tremendously difficult task to analyze Problem \eqref{eq:MCR1} from an optimization-theoretic perspective, as all the analysis would occur on a product of spheres instead of on Euclidean space. Therefore, we consider the Lagrangian formulation of \eqref{eq:MCR1}. This can be viewed as a tight relaxation or even an equivalent problem of \eqref{eq:MCR1} whose optimal solutions agree under specific settings of the regularization parameter; see \Cref{prop:equi}. Specifically, the formulation we study, referred to henceforth as the \textit{regularized MCR$^2$ problem}, is as follows: 
\begin{align}\label{eq:MCR}
\max_{\bm Z}\ F(\bm Z) := R(\bm Z) - \sum_{k=1}^K  R_c(\bm Z; \bm \pi^k) - \frac{\lambda}{2}\|\bm Z\|_F^2,
\end{align}
where $\lambda > 0$ is the regularization parameter. Remark that our study on this problem applies meaningfully to at least two approaches to learning deep representations using the MCR$^2$ principle.

\vspace{-0.1in}
\paragraph{Applications of our formulation to deep representation learning via unrolled optimization.} The first approach, as argued by \citet{chan2022redunet,yu2023sparse}, is to construct a new deep network architecture, i.e., ReduNet \citep{chan2022redunet} or CRATE \citep{yu2023sparse}, based on an iterative gradient descent scheme to optimize the MCR$^2$-type objective. In this approach, each layer of the constructed network approximates a gradient descent step to optimize the MCR$^2$-type objective given the input representation. 
 The key takeaway is that these networks approximately implement gradient descent directly on the representations, so our analysis of the optimization properties of the MCR$^2$-type objective translates to explanations of the corresponding properties of the learned representations and architectures of these deep networks. In particular, our argument that the optima and optimization landscape of \eqref{eq:MCR} are favorable directly translates to a justification of the correctness of learned representations of the ReduNet and a validation of its architecture design. Moreover, this study enables principled improvement of deep network architectures constructed via unrolled optimization by leveraging more advanced optimization techniques better suited for problems with benign landscapes. This can improve model interpretability and efficiency. 

\vspace{-0.1in}
\paragraph{Applications of our formulation to deep representation learning with standard neural networks.}
In the second approach, one parameterizes the feature mapping  $f_{\bm{\Theta}}(\cdot)$ via standard deep neural networks such as a multi-layer perceptron or ResNet \citep{he2016deep}, and treats the MCR$^2$-type objective like other loss functions applied to outputs of a neural network, such as mean-squared error or cross-entropy loss. Studying Problem \eqref{eq:MCR} from this perspective would require us to optimize over $\bm{\Theta}$ instead of over \(\bm{Z}\). 
This new optimization problem would be extraordinarily difficult to analyze, because modern neural networks have nonlinear interactions across many layers, so the parameters $\bm{\Theta}$ would affect the final representation $\bm{Z}$ in a complex way. Fortunately, since modern neural networks are often highly over-parameterized, they can interpolate or approximate any continuous function in the feature space \cite{lu2017expressive}, so we may omit these constraints by assuming the unconstrained feature model, where $\bm z_{i}$ for all $i \in [N]$ are treated as free optimization variables \cite{mixon2020neural,yaras2022neural,zhu2021geometric,wang2022linear}. Consequently, studying the optimization properties of Problem \eqref{eq:MCR} provides valuable insights into the structures of learned representations and the efficiency of training deep networks using MCR$^2$-type objectives. 

\vspace{-0.1in}
\paragraph{Difficulties of analyzing Problem \eqref{eq:MCR}.} Although Problem \eqref{eq:MCR} has no constraints, one can observe that Problem \eqref{eq:MCR} is highly non-concave due to the quadratic form $\bm Z_k\bm Z_k^T$ and the difference of log-determinant functions. Notably, this problem shares similarities with low-rank matrix factorization problems. However, it employs the log-determinant function instead of the Frobenius norm, and the computation of the objective gradient involves matrix inverses. Therefore, from an optimization point of view, it is extremely challenging to analyze Problem \eqref{eq:MCR}. 

\section{Main Results}\label{sec:main}

In this section, we first characterize the local and global optimal solutions of Problem \eqref{eq:MCR} in \Cref{subsec:opti}, and then analyze the global landscape of the objective function in \Cref{subsec:land}.  

\subsection{Characterization of Local and Global Optimality}\label{subsec:opti} 

Although Problem \eqref{eq:MCR} is highly non-concave and involves matrix inverses in its gradient computation, we can still explicitly characterize its local and global optima as follows.  

\begin{theorem}[\bf Local and global optimality]\label{thm:1}
Suppose that the number of training samples in the $k$-th class is $m_k > 0$ for each $k \in [K]$. Given a coding precision $\epsilon > 0$, if the regularization parameter satisfies
\begin{align}\label{eq:lambda}
 \lambda \in \left(0, \frac{d(\sqrt{m/m_{\max}}-1)}{m(\sqrt{m/m_{\max}}+1)\epsilon^2} \right], 
\end{align} 
then the following statements hold: \\
(i) ({\bf Characterization of local maximizers}) $\bm Z = \left[\bm Z_1,\dots,\bm Z_K \right]$ is a local maximizer of Problem \eqref{eq:MCR} if and only if the $k$-th block admits the following decomposition
\begin{align}\label{eq:Zk opti}
\bm Z_k = \overline{\sigma}_{k} \bm U_k \bm V_k^T,
\end{align}
where (a) $r_k = \mathrm{rank}(\bm Z_k)$ satisfies $r_k \in [0,\min\{m_k,d\})$ and $\sum_{k=1}^K r_k \le \min\{m,d\}$, (b) $\bm U_k \in \mathcal{O}^{d \times r_k}$ satisfies $\bm U_k^T\bm U_l = \bm 0$ for all $l \neq k$, $\bm V_k \in \mathcal{O}^{m_k \times r_k}$, and (c) the singular value $\overline{\sigma}_{k}$ is given in \eqref{eq:sigma1} for each $k \in [K]$. \\
(ii) ({\bf Characterization of global maximizers}) $\bm Z = \left[\bm Z_1,\dots,\bm Z_K \right]$ is a global maximizer of Problem \eqref{eq:MCR} if and only if (a) it satisfies the above all conditions and $\sum_{k=1}^K r_k = \min\{m,d\}$, and (b) for all $k \neq l \in [K]$ satisfying $m_k < m_l$ and $r_l > 0$, we have $r_k = \min\{m_k,d\}$. 
\end{theorem}
We defer the proof to \Cref{subsec:pf opti} and \Cref{sec:pf thm1}. In this theorem, we explicitly characterize the local and global optima of Problem  \eqref{eq:MCR}. Intuitively, this demonstrates that the features represented by each local maximizer of Problem \eqref{eq:MCR} are low-dimensional and discriminative in the sense that \\
(i) {\em Within-Class Compressible:} According to \eqref{eq:Zk opti}, at each local maximizer, the features from the same class belong to the same low-dimensional linear subspace. \\
(ii) {\em Between-Class Discriminative:} It follows from \eqref{eq:Zk opti} and $\bm U_k^T\bm U_l = \bm 0$ for all $k \neq l$ that, at each local maximizer, the features from different classes belong to different subspaces that are orthogonal to each other.  \\
Moreover, the features represented by each global maximizer of Problem \eqref{eq:MCR} are not only low-dimensional and discriminative but also diverse in the sense that \\ 
(iii) {\em Maximally Diverse Representation:} According to $\sum_{k=1}^K r_k = \min\{m,d\}$, at each global maximizer, the total dimension of all features is maximized to match the highest dimension that it can achieve in the feature space.   

\vspace{-0.1in}
\paragraph{Quality of local versus global optima.} Our above discussion explains the merits of achieving both local and global optima. At each maximizer, the representations are within-class compressible and between-class discriminative (\Cref{thm:1} (i)). Moreover, global maximizers further satisfy that the representations are all maximally diverse (\Cref{thm:1} (ii)(a)). If all classes were balanced, i.e., $m_{1} = \cdots = m_{K}$, then \Cref{thm:1} (ii)(b) would not apply, and these properties would be all that \Cref{thm:1} asserts. In this case, global optima would clearly be desired over local optima. However, in the unbalanced case, the situation is more complex, because \Cref{thm:1} (ii)(b) would apply. It says that for global optima, the classes with the smallest numbers of samples would fill to the largest dimension possible, and the very largest classes could collapse to $\bm{0}$, an undesirable situation. A dramatic example of this is when $m_{1} > \cdots > m_{K} > d$, for then any global optimum would have $\operatorname{rank}(\bm{Z}_{K}) = d$ and $\bm{Z}_{1}, \dots, \bm{Z}_{K - 1}$ all collapse to $\bm{0}$. Overall, in the unbalanced case, global optima may not always correspond to the best representations. In particular, local optima with more equitable rank distributions (like bigger classes span more dimensions) which are still maximally diverse (i.e., ranks of each class sum to the dimension $d$) could be preferred in applications. As demonstrated in \Cref{subsec:exp1}, these kinds of potentially useful local optima are realized in experiments, even with unbalanced classes. 

\vspace{-0.1in}
\paragraph{Relation between Problems \eqref{eq:MCR1} and \eqref{eq:MCR}.} Based on the characterization of global optimality in \Cref{thm:1}, we show the following proposition that establishes the relationship between the constrained MCR$^2$ problem \eqref{eq:MCR1} and the regularized MCR$^2$ problem \eqref{eq:MCR} in terms of their global solutions under an appropriate choice of the regularization parameter. The proof of this result can be found in \Cref{sec:pf prop}. 

\begin{prop}\label{prop:equi}
    Suppose that the number of training samples in each class is the same, i.e., $m_1=\dots=m_K$, and the coding precision $\epsilon > 0$ satisfies 
    \begin{align}\label{eq:epsilon}
        \epsilon \le  \frac{1}{6}\sqrt{\frac{d}{m}}\exp\left(-\frac{1}{2}\right)K^{-\frac{1}{K-1}}\left( 1 + \frac{1}{\sqrt{K}}\right)^{-\frac{2m}{K-1}}.
    \end{align}
    The following statements hold: \\
    (i) If $m < d$ and the regularization parameter in Problem \eqref{eq:MCR} is set as 
    \begin{align}\label{eq:lambda m>d}
        \lambda = \frac{\alpha}{1+\alpha} - \frac{\alpha}{1+K\alpha},
    \end{align}
    Problems \eqref{eq:MCR1} and \eqref{eq:MCR} have the same global solution set. \\
    (ii) If $m \ge d$, $d/K$ is an integer, and the regularization parameter in Problem \eqref{eq:MCR} is set as
    \begin{align}\label{eq:lambda m<d}
        \lambda = \frac{\alpha}{1+\alpha m/d} -  \frac{\alpha}{1+\alpha K m/d},
    \end{align}
     the global solution set of Problem \eqref{eq:MCR1} is a subset of that of Problem \eqref{eq:MCR}.
\end{prop}
According to this proposition, if $\epsilon$ and $\lambda$ are appropriately chosen for Problem \eqref{eq:MCR}, when $m < d$, Problems \eqref{eq:MCR1} and \eqref{eq:MCR} are equivalent in terms of their global optimal solutions; when $m \le d$, Problem \eqref{eq:MCR} is a tight Lagrangian relaxation of Problem \eqref{eq:MCR1} such that the global solution set of the former contains that of the latter.

\subsection{Analysis of Global Optimization Landscape}\label{subsec:land} 

While we have characterized the local and global optimal solutions in \Cref{thm:1}, it remains unknown whether these solutions can be computed efficiently using GD to solve Problem \eqref{eq:MCR}, as GD may get stuck at a saddle point. Fortunately, \citet{sun2015nonconvex,lee2016gradient} showed that if a function is twice continuously differentiable and satisfies {\em strict saddle property}, i.e., each critical point is either a local minimizer or a strict saddle point\footnote{We say that a critical point is a strict saddle point of Problem \eqref{eq:MCR} if it has a direction with strictly positive curvature; see \Cref{def:saddle}. This includes classical saddle points with strictly positive curvature as well as local minimizers.}, GD converges to its local minimizer almost surely with random initialization. We investigate the global optimization landscape of Problem \eqref{eq:MCR} by characterizing all of its critical points as follows.  

\begin{theorem}[\bf Benign optimization landscape]\label{thm:2}
Suppose that the number of training samples in the $k$-th class is $m_k > 0$ for each $k \in [K]$. Given a coding precision $\epsilon > 0$, if the regularization parameter satisfies \eqref{eq:lambda},  
it holds that any critical point $\bm Z$ of Problem \eqref{eq:MCR} that is not a local maximizer is a strict saddle point. 
\end{theorem}
We defer the proof to \Cref{subsec:pf land} and \Cref{sec:pf thm2}. Here, we make some remarks on this theorem and also on the consequences of the results derived so far. 

\vspace{-0.1in}
\paragraph{Differences from existing results on the MCR$^2$ problem.}

\citet{chan2022redunet,yu2020learning} have characterized the global optimality of Problem \eqref{eq:MCR1} with Frobenius norm constraints on each $\bm Z_k$ in the UFM. 
However, their analysis requires an additional rank constraint on each $\bm Z_k$ and only characterizes globally optimal representations. In contrast, our analysis eliminates the need for the rank constraint, and we characterize local and global optimality in Problem \eqref{eq:MCR}, as well as its optimization landscape. Interestingly, we demonstrate that the features represented by each local maximizer --- not just global maximizers --- are also compact and structured. Furthermore, we demonstrate that the regularized MCR$^2$ objective \eqref{eq:MCR} is a strict saddle function. To the best of our knowledge, \Cref{thm:1,thm:2} constitute the first analysis of local optima and optimization landscapes for MCR$^2$ objectives. According to \citet{daneshmand2018escaping,lee2016gradient,xu2018first}, \Cref{thm:1,thm:2} imply that low-dimensional and discriminative representations can be efficiently found by (stochastic) GD on Problem \eqref{eq:MCR} from a random initialization.

\vspace{-0.1in} 
\paragraph{Comparison to existing landscape analyses in non-convex optimization.} In recent years, there has been a growing body of literature exploring optimization landscapes of non-convex problems in machine learning and deep learning. These include low-rank matrix factorization \cite{ge2017no,sun2018geometric,chi2019nonconvex,zhang2020symmetry}, community detection \cite{wang2021optimal,wang2022non}, dictionary learning \cite{sun2017complete_a,qu2020geometric}, and deep neural networks \cite{sun2020global,yaras2022neural,zhou2022optimization,zhu2021geometric,jiang2024generalized,li2024neural}. The existing analyses in the literature cannot be applied to the MCR$^2$ problem due to its special structure, which involves the log-determinant of all features minus the sum of the log-determinant of features in each class. Our work contributes to the literature on optimization landscape analyses of non-convex problems by showing that the MCR$^2$ problem has a benign optimization landscape. Our approach may be of interest to analyses of the landscapes of other intricate loss functions in practical applications.  

\section{Proofs of Main Results}\label{sec:pf}  

In this section, we sketch the proofs of our main theorems in \Cref{sec:main}. The complete proofs can be found in Sections \ref{sec:pf opti} and \ref{sec:pf land} of the appendix. For ease of exposition, let 
\begin{align}\label{eq:alp}
\alpha := \frac{d}{m\epsilon^2},\quad \alpha_k := \frac{d}{m_k\epsilon^2},\quad \forall k \in [K]. 
\end{align}  

\subsection{Analysis of Optimality Conditions}\label{subsec:pf opti}

Our goal in this subsection is to characterize the local and global optima of Problem \eqref{eq:MCR}. Towards this goal, we first provide an upper bound on the objective function $F$ in Problem \eqref{eq:MCR}. In particular, this upper bound is tight when the blocks $\{\bm Z_k\}_{k=1}^K$ are orthogonal to each other. This result is a direct consequence of \citep[Lemma 10]{chan2022redunet}.  

\begin{lemma}\label{lem:upper bound}
For any $\bZ = \left[\bZ_1,\ldots,\bZ_K\right] \in \R^{d \times m}$ with $\bm Z_k \in \R^{d \times m_k}$, we have 
\begin{align}\label{eq:F upper}
F(\bm Z) \le \sum_{k=1}^K  \left( \frac{1}{2} \log\det\left( \bm I_n + \alpha \bm Z_k \bm Z_k^T \right) - \frac{m_k}{2m} \log\det\left( \bm I_n + \alpha_k \bm Z_k \bm Z_k^T \right) - \frac{\lambda}{2} \|\bm Z_k\|_F^2 \right),  
\end{align}
where the equality holds if and only if  $\bm Z_k^T\bm Z_l = \bm 0$ for all $1 \le k \neq l \le K$. 
\end{lemma} 
Next, we study the following set of critical points, which are between-class discriminative (i.e., $\bm Z_k^T\bm Z_l = \bm 0$): 
\begin{align}\label{eq:set Z}
\mathcal{Z} := \left\{\bm Z: \nabla F(\bm Z) = 0,\ \bm Z_k^T\bm Z_l = \bm 0,\forall k \neq l \right\}. 
\end{align}
\begin{prop}\label{prop:set Z}
Consider the setting of \Cref{thm:1}. It holds that $\bm Z = \left[\bm Z_1,\dots,\bm Z_K \right] \in \cal Z$ if and only if each $\bm Z_k$ admits the following singular value decomposition 
\begin{align}\label{eq:Zk}
\bm Z_k = \bm U_k\tilde{\bm \Sigma}_k\bm V_k^T,\ \tilde{\bm \Sigma}_k = \mathrm{diag}\left(\sigma_{k,1},\dots,\sigma_{k,r_k}\right), 
\end{align}
where  
(i) $r_k \in \left[0,\min\{m_k,d\}\right)$ satisfies $\sum_{k=1}^K r_k \le d$, (ii) $\bm U_k \in \mathcal{O}^{n\times r_k}$ satisfies $\bm U_k^T\bm U_l = \bm 0$ for all $1 \le k \neq l \le K$, $\bm V_k \in \mathcal{O}^{m_k \times r_k}$ for all $k \in [K]$, and  (iii) the singular values satisfy  
\begin{align}\label{eq:gamma}
\sigma_{k,i} \in \left\{ \overline{\sigma}_k, \underline{\sigma}_{k} \right\},\ \forall i \in [r_k],
\end{align}
where $\eta_k = (\alpha_k - \alpha) - \lambda\left(m/m_k+1\right)$ and 
\begin{align}
\overline{\sigma}_k  = \left(\frac{ \eta_k + \sqrt{\eta_k^2 - 4\lambda^2m/m_k}}{2\lambda \alpha_k}\right)^{1/2},\quad \underline{\sigma}_{k}  = \left(\frac{ \eta_k - \sqrt{\eta_k^2 - 4\lambda^2m/m_k}}{2\lambda \alpha_k}\right)^{1/2}. \label{eq:sigma1}
\end{align}
\end{prop}
This proposition shows that each critical point that is between-class discriminative (i.e., $\bm Z_k^T\bm Z_l = \bm 0$) exhibits a specific structure: the singular values of $\bm Z_k$ can only take on two possible values, $\overline{\sigma}_k$ and $\underline{\sigma}_{k}$. We will leverage this structure and further show that $\bm{Z}$ is a strict saddle if there exists a $\bm{Z}_k$ with a singular value $\underline{\sigma}_k$.

\subsection{Analysis of Optimization Landscape}\label{subsec:pf land}

Our goal in this subsection is to show that the function $F$ in Problem \eqref{eq:MCR} has a benign optimization landscape. Towards this goal, we denote the set of critical point of $F$ by 
\begin{align}\label{eq:crit}
\mathcal{X} = \left\{ \bm Z \in \R^{d\times m}: \nabla F(\bm Z) = \bm 0 \right\}. 
\end{align} 
According to \eqref{eq:set Z}, we divide the critical point set $\cal X$ into two disjoint sets  $\mathcal{Z}$ and $\mathcal{Z}^c$, i.e., $\mathcal{X} = \mathcal{Z} \cup \mathcal{Z}^c$, where 
\begin{align}\label{eq:Zc}
\mathcal{Z}^c := \left\{\bm Z: \nabla F(\bm Z) = 0,\ \bm Z_k^T\bm Z_l \neq \bm 0,\ \exists k \neq l \right\}. 
\end{align} 
Moreover, according to \Cref{prop:set Z}, we further divide $\cal Z$ into two disjoint sets $\mathcal{Z}_1$ and $\mathcal{Z}_2$, i.e., $\mathcal{Z} = \mathcal{Z}_1 \cup \mathcal{Z}_2$. Here,  
\begin{align}\label{eq:Z12}
& \mathcal{Z}_1 := \mathcal{Z} \cap \left\{\bm Z : \sigma_{k,i}(\bm Z_k) = \overline{\sigma}_k, \forall i \in [r_k], k \in [K] \right\},\ \mathcal{Z}_2 := \mathcal{Z}\setminus \mathcal{Z}_1,
\end{align} 
where $\sigma_{k,i}(\bm Z_k)$ denotes the $i$-th largest singular value of $\bm Z_k$. Our first step is to show that any point belonging to $\mathcal{Z}_1$ is a local maximizer, while any point belonging to $\mathcal{Z}_2$ is a strict saddle point. 

\begin{prop}\label{prop:local max}
Consider the setting of \Cref{thm:2}. Suppose that  $\bm Z \in \cal Z$. Then, the following statements hold: \\
(i) If $\bm Z_k$ takes the form of \eqref{eq:Zk} with $\sigma_{k,i} = \overline{\sigma}_k$ for all $i \in [r_k]$ and all $k \in [K]$, i.e., $\bm Z \in \mathcal{Z}_1$, then $\bm Z$ is a local maximizer. \\
(ii) If there exists a $k \in [K]$ and $i \in [r_k]$ with $r_k \ge 1$ such that $\sigma_{k,i} = \underline{\sigma}_k$, i.e., $\bm Z \in \mathcal{Z}_2$, then $\bm Z$ is a strict saddle point. 
\end{prop} 

Next, we proceed to the second step to show that any point belonging to $\mathcal{Z}^c$ is a strict saddle point. It suffices to find a direction $\bm D \in \R^{d\times m}$ such that $\nabla^2 F(\bm Z)[\bm D, \bm D] > 0$ for each $\bm Z \in \mathcal{Z}^c$ according to \Cref{def:saddle}. 

\begin{prop}\label{prop:saddle}
Consider the setting of \Cref{thm:2}. If $\bm Z\in \R^{d\times m}$ is a critical point and there exists $1 \le k \neq l \le K$ such that $\bm Z_k^T\bm Z_l \neq \bm 0$, i.e., $\bm Z \in \mathcal{Z}^c$, then $\bm Z$ is a strict saddle point.  
\end{prop}

With the above preparations that characterize all the critical points, we can prove \Cref{thm:1} and \Cref{thm:2}. We refer the reader to \Cref{sec:pf main} for the detailed proof.

\section{Experimental Results}\label{sec:expe}

\begin{figure*}[!t]
\begin{center}
	\begin{subfigure}{0.21\textwidth}
    	\centering\includegraphics[width = 1 \linewidth]{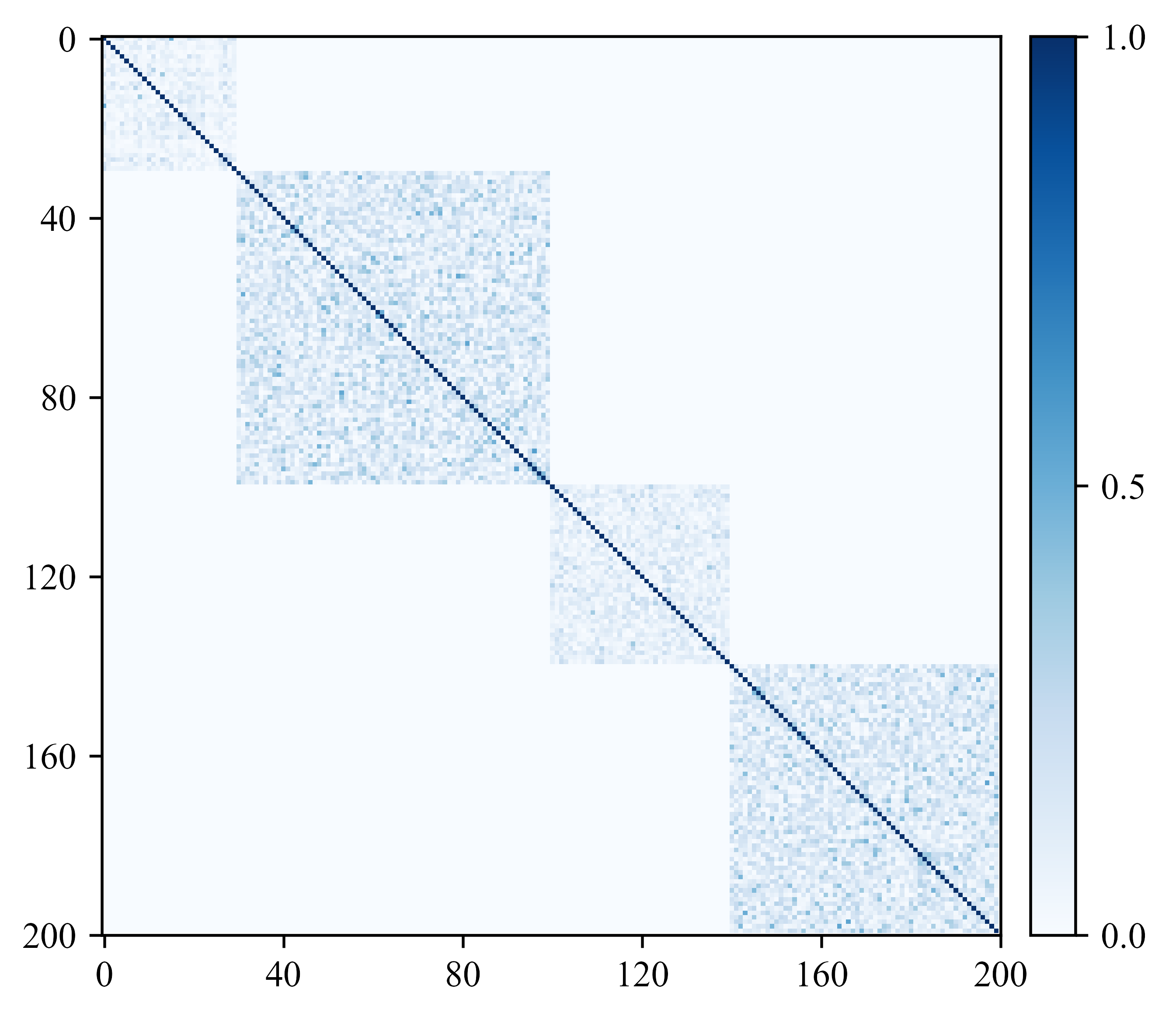}\vspace{-0.05in}
    \centering\caption{Cosine similarity} 
    \end{subfigure} \vspace{-0.05in} 
    \begin{subfigure}{0.78\textwidth}
    	\centering\includegraphics[width = 1 \linewidth]{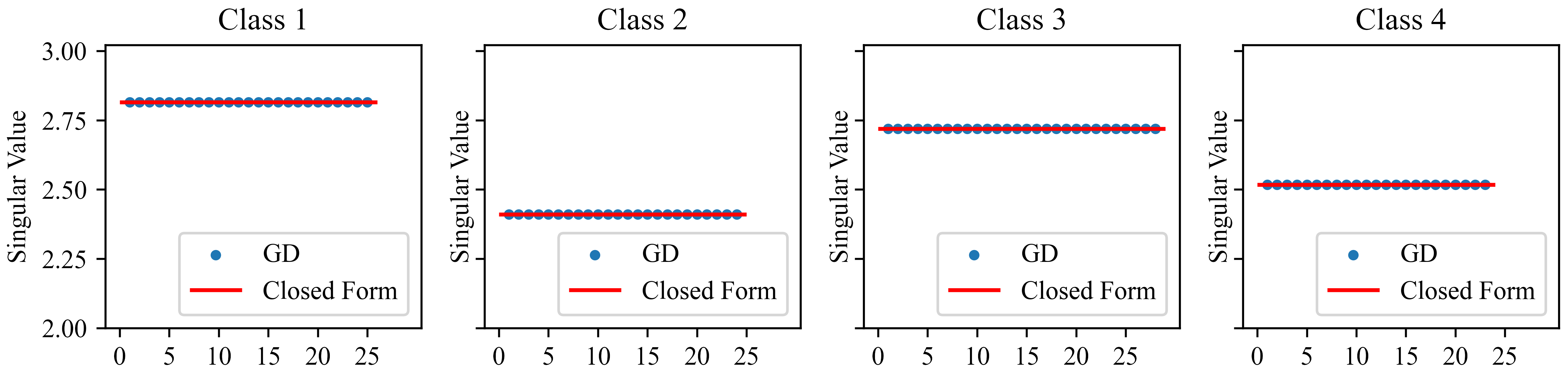} 
    \centering\caption{The number and magnitude of singular values of each class matrix.} 
    \end{subfigure}\vspace{-0.05in} 
    \caption{\textbf{Validation of theory for the MCR$^2$ problem.} (a) We visualize the heatmap of cosine similarity among learned features by GD for solving Problem \eqref{eq:MCR}. The lighter pixels represent lower cosine similarities between pairwise features. (b) The blue dots are plotted based on the singular values by applying SVD to the solution returned by GD, and the red line is plotted according to the closed-form solution in \eqref{eq:Zk opti}. The number of nonzero singular values in each subspace is $24, 23, 27, 26$, respectively.}\vspace{-0.2in} 
    \label{fig:opti}
\end{center}
\end{figure*}

In this section, we first conduct numerical experiments on synthetic data in \Cref{subsec:exp1} to validate our theoretical results, and then on real-world data sets using deep neural networks in \Cref{subsec:exp2} to further support our theory. All codes are implemented in Python mainly using NumPy and PyTorch. All of our experiments are executed on a computing server equipped with NVIDIA A40 GPUs. Due to space limitations, we defer some implementation details and additional experimental results to \Cref{app sec:exp}. 

\subsection{Validation of Theory for Solving Problem \eqref{eq:MCR}}\label{subsec:exp1}

In this subsection, we employ GD for solving Problem \eqref{eq:MCR} with different parameter settings. We visualize the optimization dynamics and structures of the solutions returned by GD to verify and validate Theorems \ref{thm:1} and \ref{thm:2}.

\vspace{-0.1in}
\paragraph{Verification of \Cref{thm:1}.} In this experiment, we set the parameters in Problem \eqref{eq:MCR} as follows: the dimension of features $d=100$, the number of classes $K=4$, the number of samples in each class is $m_1=30,m_2=70,m_3=40,m_4=60$, the regularization parameter $\lambda=0.1$, and the quantization error $\epsilon = 0.5$. Then, one can verify that $\lambda$ satisfies \eqref{eq:lambda}. For the solution $\bm Z$ returned by GD, we first plot the heatmap of the cosine similarity between pairwise columns of $\bm Z$ in \Cref{fig:opti}(a). We observe that the features from different classes are orthogonal to each other, while the features from the same class are correlated. Next, we compute the singular values of $\bm Z_k$ via singular value decomposition (SVD) and plot the singular values using blue dots for each $k \in [K]$ in \Cref{fig:opti}(b). According to the closed-form solution \eqref{eq:Zk opti} in \Cref{thm:1}, we also plot the theoretical bound of singular values in red in \Cref{fig:opti}(b). One can observe that the number of singular values of each block is respectively $24, 23, 27, 26$, summing up to $100$, and the red line perfectly matches the blue dots. These results all provide strong support for \Cref{thm:1}. 

\vspace{-0.1in}
\paragraph{Verification of \Cref{thm:2}.} In this experiment, we maintain the same setting as above, except that the number of samples in each class is equal. We first fix $m=200$ and vary $d \in \{40,80,120\}$, and then fix $d=50$ and vary $d \in \{100,200,400\}$  to run GD.  We plot the distances between function values of the iterates to the optimal value, which is computed according to \eqref{eq:Zk opti} in \Cref{thm:1}, against the iteration numbers in \Cref{fig:conv}. We observe that GD with random initialization converges to an optimal solution at a linear rate. This indicates that the MCR$^2$ has a benign global landscape, which supports \Cref{thm:2}. 

\begin{figure*}[t]
\begin{center}
	\begin{subfigure}{0.48\textwidth}
    	\centering\includegraphics[width = 1\linewidth]{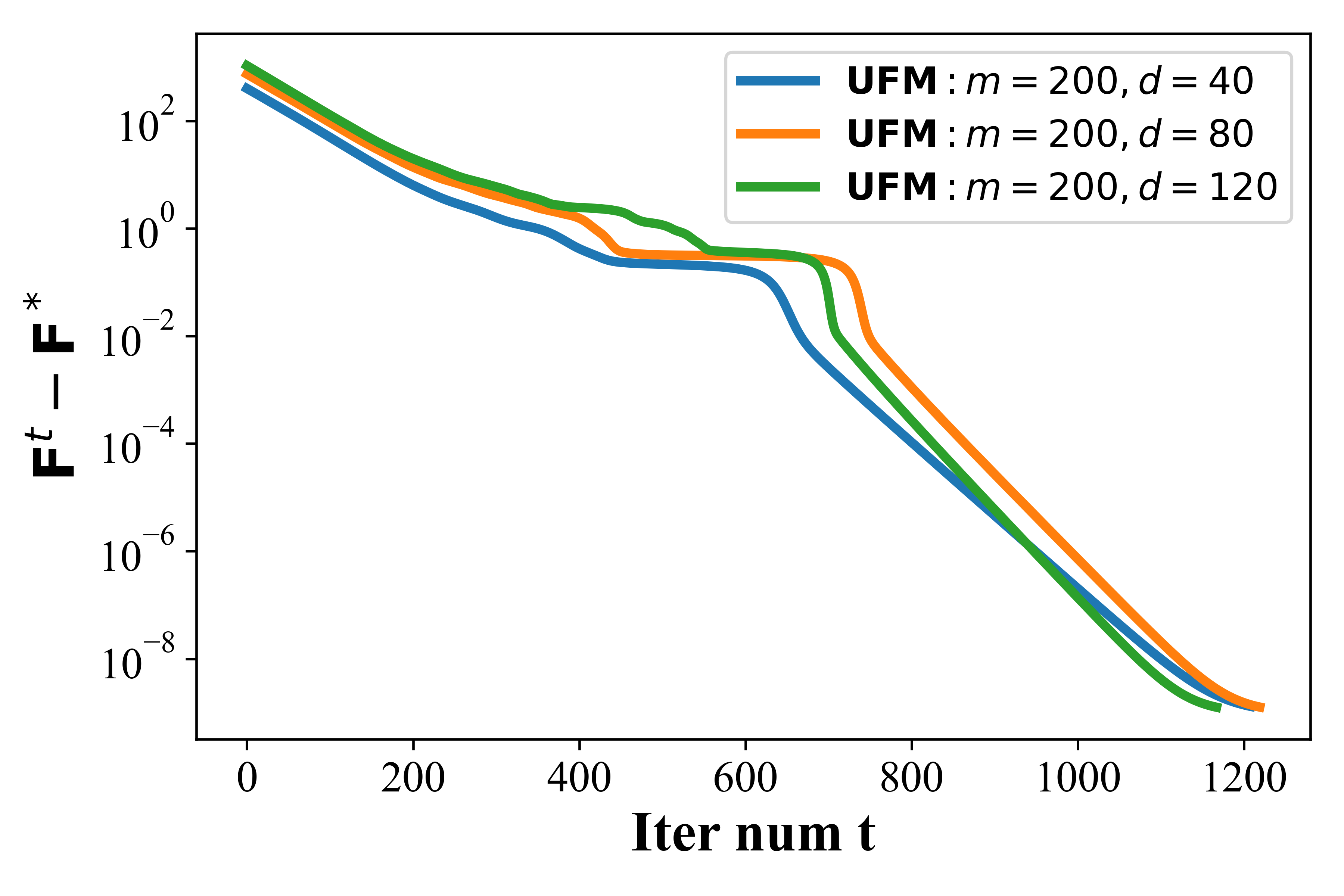}\vspace{-0.1in}
    \caption{Convergence for $m=200, d \in \{40,80,120\}$} 
    \end{subfigure} 
    \begin{subfigure}{0.48\textwidth}
    	\centering\includegraphics[width = 1\linewidth]{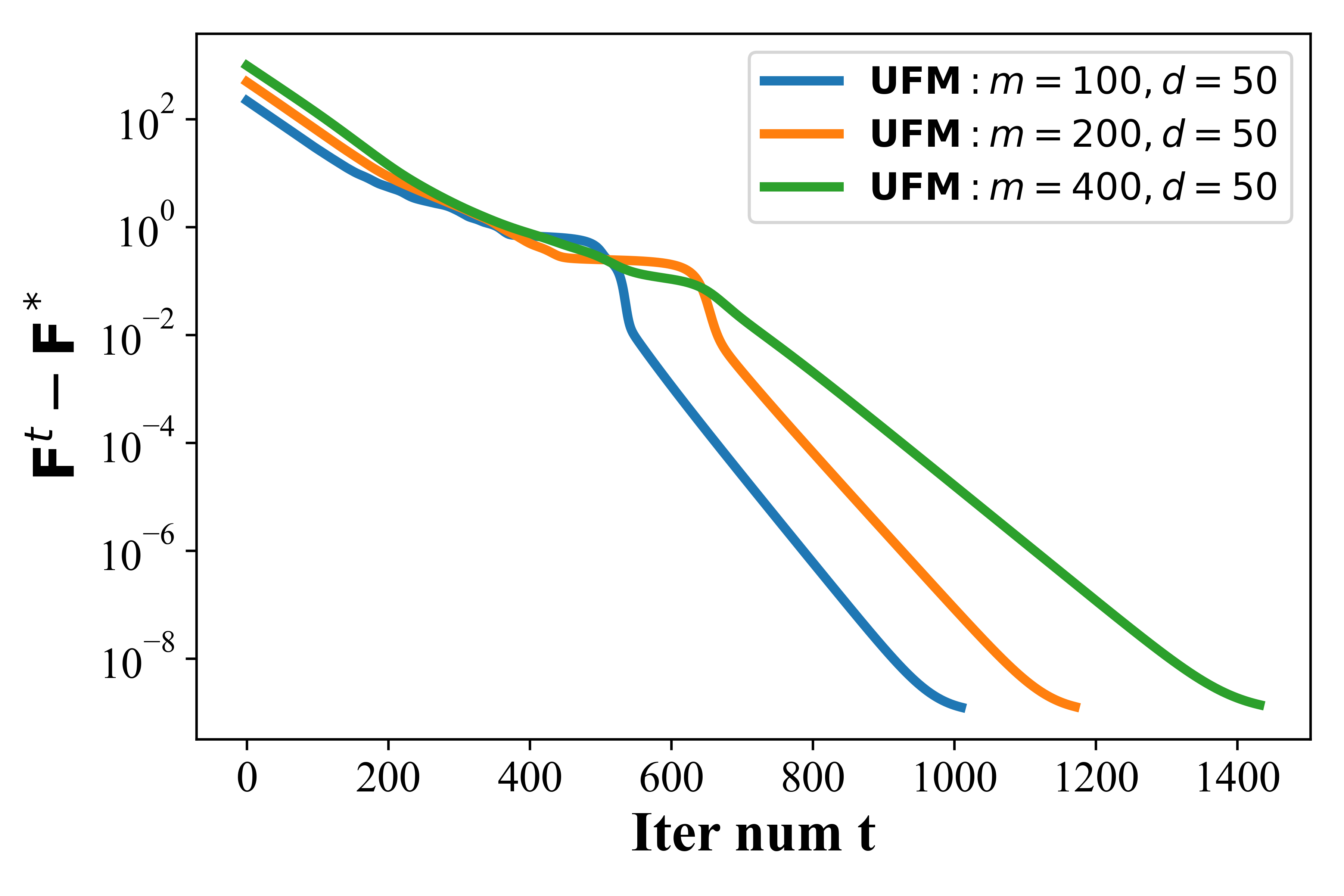}\vspace{-0.1in}
    \caption{Convergence for $d=50, m \in \{100,200,400\}$} 
    \end{subfigure}\vspace{-0.1in} 
    \caption{\textbf{Convergence performance of GD for solving the regularized MCR$^2$ problem.} Here, the $x$-axis is number of iterations (also denoted by $t$), and $y$-axis is the function value gap $F^t-F^*$, where $F^t=F(\bm Z^t)$ denotes the function value at the $t$-th iterate $\bm Z^t$ generated by GD, and $F^*$ is the optimal value of Problem \eqref{eq:MCR} computed according to \eqref{eq:Zk opti} in \Cref{thm:1}.}\vspace{-0.2in} 
    \label{fig:conv}
\end{center}
\end{figure*}

\subsection{Training Deep Networks Using Regularized MCR$^2$}\label{subsec:exp2}

In this subsection, we conduct numerical experiments on the image datasets MNIST \cite{lecun1998gradient} and CIFAR-10 \cite{krizhevsky2009learning} to provide evidence that our theory also applies to deep networks. More specifically, we employ a multi-layer perceptron network with ReLU activation as the feature mapping $\bm z = f_{\bm \Theta}(\bm x)$ with output dimension 32 for MNIST and 128 for CIFAR-10. Then, we train the network parameters $\bm \Theta$ via Adam \cite{kingma2014adam} by optimizing Problem \eqref{eq:MCR}.

\vspace{-0.1in}
 \paragraph{Experimental setting and results.} In the experiments, we randomly sample a balanced subset with $K$ classes and $m$ samples from MNIST or CIFAR-10, where each class has the same number of samples. We set $\lambda = 0.001$ and $\epsilon = 0.5$. For different subsets with corresponding values of $(m,K)$, we run experiments and report the function value $\hat{F}$ obtained by training deep networks and the optimal value $F^*$ computed using the closed-form solution in \Cref{thm:1} in \Cref{table:1}. To verify the discriminative nature of the features obtained by training deep networks across different classes, we measure the discrimination between features belonging to different classes by computing the cosine of the principal angle \citep{bjorck1973numerical} between the class subspaces: $s = \max\left\{ \|\bm U_k^T \bm U_l\|: k \neq l \in [K] \right\} \in [0,1]$, 
 where the columns of $\bm U_k \in \R^{d\times r_k}$ are the right singular vectors corresponding to the top $r_k$ singular values of $\bm Z_k$ defined in \eqref{eq:Zk} and $r_k$ is its rank\footnote{We estimate the rank of a matrix by rounding its ``stable rank'' \citep{horn2012matrix}: $r_{k} = \operatorname{round}(\|\bm{Z}_{k}\|_{F}^{2}/\|\bm{Z}_{k}\|^{2})$.} for each $k \in [K]$. In particular, when $s$ is smaller, the spaces spanned by each pair $\bm Z_k$ and $\bm Z_l$ for $k \neq l$ are closer to being orthogonal to each other. 
Then, we record the value $s$ in \Cref{table:1} in different settings. Moreover, we visualize the pairwise cosine similarities between learned features on MNIST and CIFAR-10 when $(m, K) = (1500, 6)$ and $(2500, 10)$ in \Cref{fig:4}. 
\begin{table}[!tbp]
\caption{Function value $\hat{F}$ obtained by training deep networks, the optimal value  $F^*$ computed by our theory on subsets of MNIST or CIFAR-10, and discrimination metric $s$ of features. }\label{table:1}
\vskip -0.1in
\begin{center}
\begin{tabular}{ccccc}
\toprule
& MNIST $(m,K)$ & $\hat{F}$ & $F^*$ & $s$ \\ 
\midrule
& (1000, 4) & 37.38 & 37.38 & $ 5.9 \cdot 10^{-6}$ \\
& (1500, 6) & 38.96 & 38.96 & $ 3.8 \cdot 10^{-6}$ \\ 
& (2000, 8) & 38.48 & 38.48 & 0.011 \\
& (2500, 10) &  37.41 & 37.41 & 0.008 \\ 
\bottomrule
\toprule
& CIFAR-10 $(m,K)$ & $\hat{F}$ & $F^*$ & $s$ \\ 
\midrule
& (1000, 4) & 215.61 & 215.61 & 0.004 \\
& (1500, 6) & 229.14 & 229.14 & 0.029 \\
& (2000, 8) & 230.70 & 230.70 & 0.059 \\
& (2500, 10) & 228.48 & 228.49 & 0.171 \\ 
\bottomrule
\end{tabular}
\end{center}
\vskip -0.1in
\end{table}

We observe from \Cref{table:1} that the function value returned by training deep networks is extremely close to the global optimal value of Problem \eqref{eq:MCR} and from the value $s$ and \Cref{fig:4} that the features from different classes are nearly orthogonal to each other. These observations, together with \Cref{thm:1,thm:2}, indicate that Problem \eqref{eq:MCR} retains its optimization properties even when $\bm{Z}$ is parameterized by a neural network. Our theoretical analysis of Problem \eqref{eq:MCR} thus illustrates a qualitative picture of training deep networks with the regularized MCR$^2$ objective.

\begin{figure*}[t]
\begin{center}
    \begin{subfigure}{0.24\textwidth}
    	\includegraphics[width = 1\linewidth]{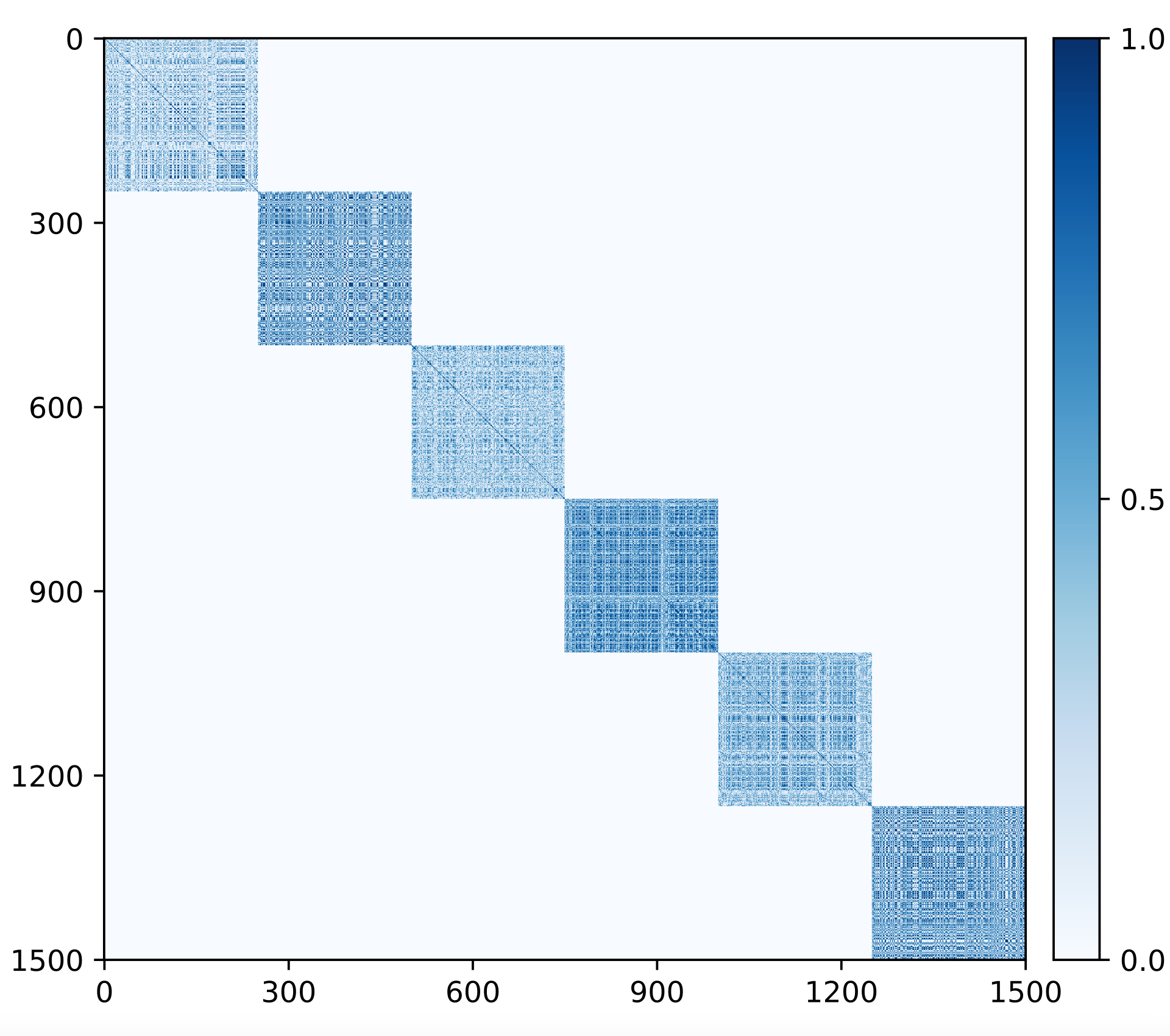}\vspace{-0.05in}
    \caption*{\footnotesize MNIST: $m=1500,K=6$} 
    \end{subfigure} 
    \begin{subfigure}{0.24\textwidth}
    	\includegraphics[width = 1\linewidth]{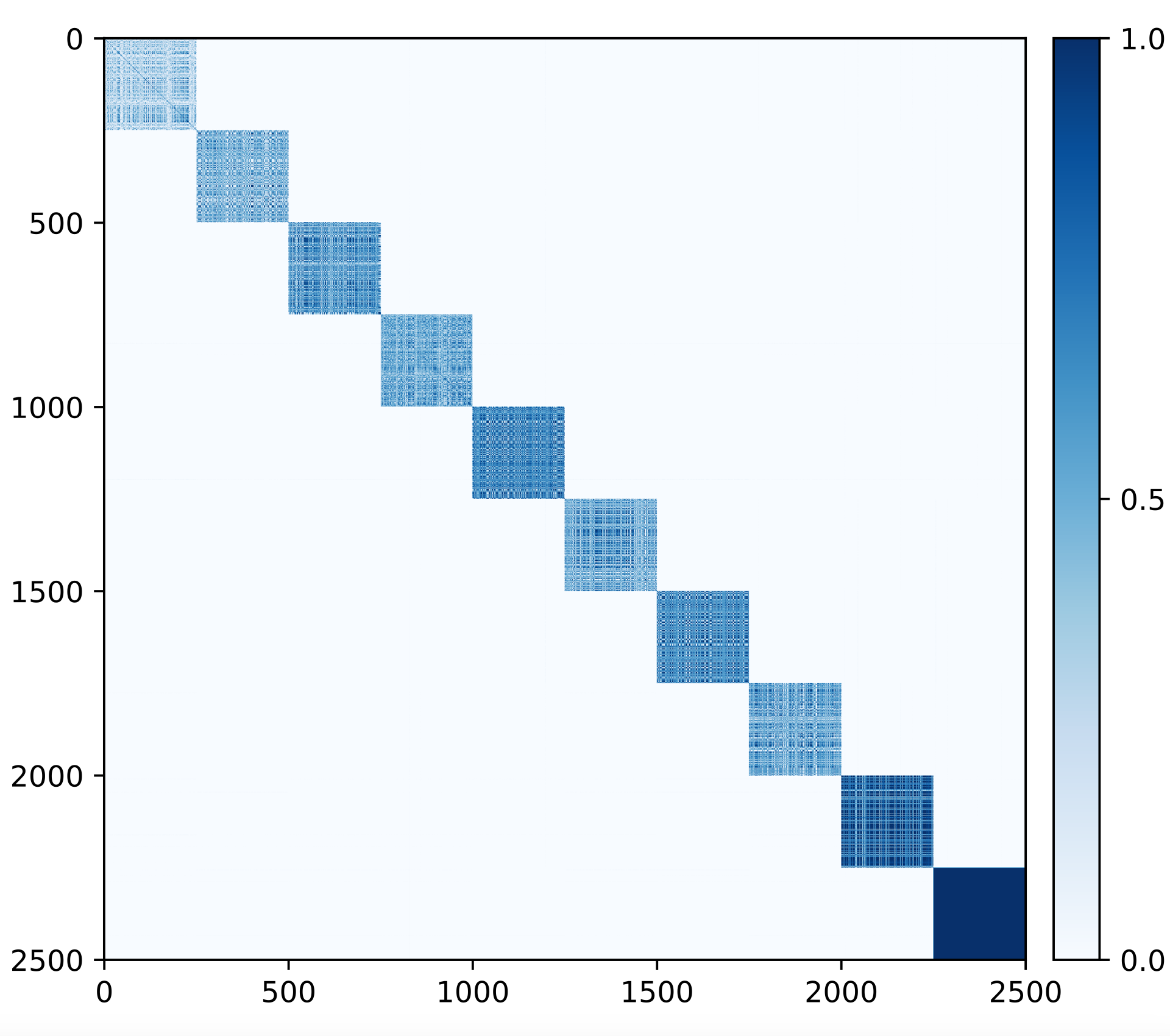}\vspace{-0.05in}
    \caption*{\footnotesize MNIST: $m=2500,K=10$} 
    \end{subfigure}\vspace{-0.15in} 
    \begin{subfigure}{0.24\textwidth}
    	\includegraphics[width = 1\linewidth]{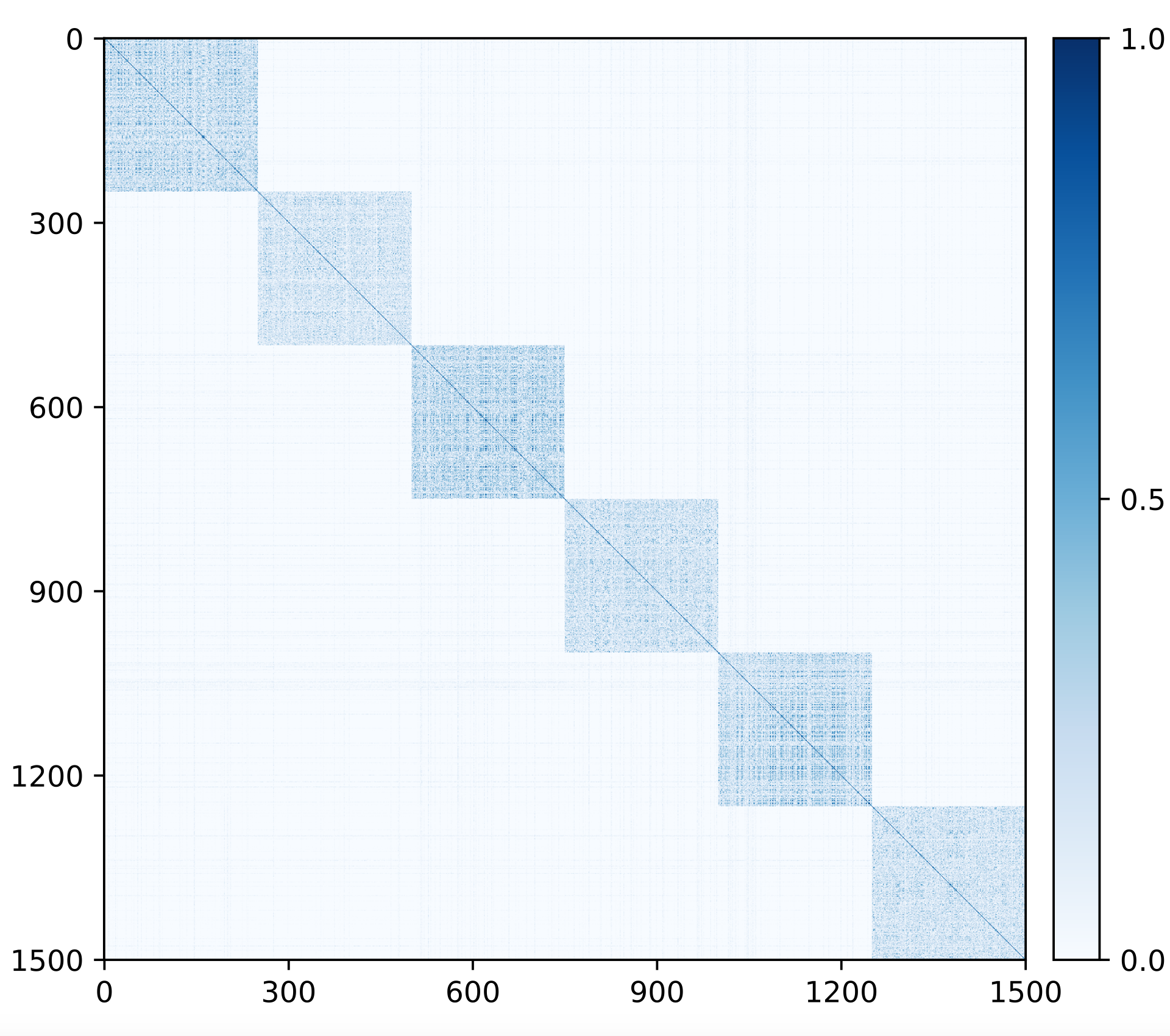}\vspace{-0.05in}
    \caption*{\footnotesize CIFAR: $m=1500,K=6$} 
    \end{subfigure} 
    \begin{subfigure}{0.24\textwidth}
    	\includegraphics[width = 1\linewidth]{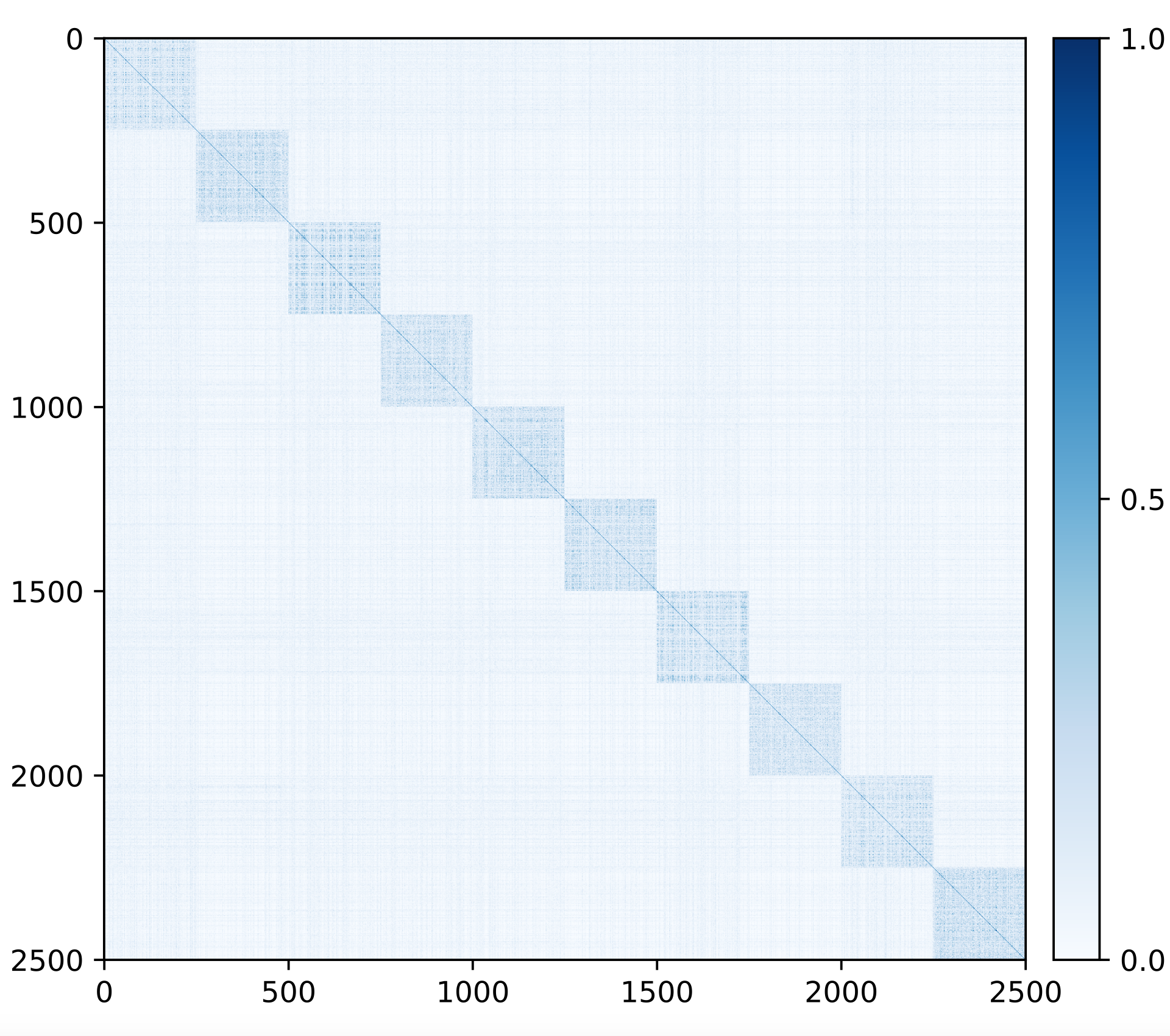}\vspace{-0.05in}
    \caption*{\footnotesize CIFAR: $m=2500,K=10$} 
    \end{subfigure} 
    \caption{\textbf{Heatmap of cosine similarity among features produced by deep networks trained on MNIST and CIFAR-10.} The darker pixels represent higher absolute cosine similarity between features.
    }
    \vspace{-0.2in} 
    \label{fig:4} 
\end{center}
\end{figure*}

\section{Conclusion}\label{sec:conc}

In this work, we provided a complete characterization of the global landscape of the MCR$^2$ objective, a highly nonconcave and nonlinear function used for representation learning. We characterized all critical points, including the local and global optima, of the MCR$^2$ objective, and showed that --- surprisingly --- it has a benign global optimization landscape. These characterizations provide rigorous justifications for why such an objective can be optimized well using simple algorithms such as gradient-based methods. In particular, we show that even local optima of the objective leads to geometrically meaningful representations. Our experimental results on synthetic and real-world datasets clearly support this new theoretical characterization. With the global landscape clearly revealed, our work paves the way for exploring better optimization strategies, hence better deep neural network architectures, for optimizing the MCR$^2$ objective more efficiently and effectively. For future work, it is natural to extend our analysis to Problem \eqref{eq:MCR1} with deep network parameterizations. It is also interesting to study the sparse MCR$^2$ objective, which has led to high-performance transformer-like architectures  \cite{yu2023sparse,yu2023white}.

\section*{Acknowledgements} 

The work of P.W. is supported in part by ARO YIP award W911NF1910027 and DoE award DE-SC0022186. The work of H.L. is supported in part by NSF China under Grant 12301403 and the Young Elite Scientists Sponsorship Program by CAST 2023QNRC001. 
The work of D.P. is supported by a UC Berkeley College of Engineering Fellowship. 
Y.Y. acknowledges support from the joint Simons Foundation-NSF DMS grant \#2031899.
Q.Q. acknowledges support from NSF CAREER CCF-2143904, NSF CCF-2212066, NSF CCF-2212326, NSF IIS 2312842, and Amazon AWS AI Award, and a gift grant from KLA. Z.Z. acknowledges support from NSF grants CCF-2240708 and IIS-2312840. Y.M. acknowledges support from the joint Simons Foundation-NSF DMS grant \#2031899, the ONR grant N00014-22-1-2102, and the University of Hong Kong.

\bibliographystyle{abbrvnat} 
\bibliography{MCR_reference}

@Misc{chan2022redunet,
  author   = {Chan, Kwan Ho Ryan and Yu, Yaodong and You, Chong and Qi, Haozhi and Wright, John and Ma, Yi},
  title    = {Redunet: A white-box deep network from the principle of maximizing rate reduction},
  year     = {2022},
  journal  = {Journal of Machine Learning Research},
  number   = {114},
  pages    = {1--103},
  volume   = {23},
}

@article{zhang2020symmetry,
  title={From Symmetry to Geometry: Tractable Nonconvex Problems},
  author={Zhang, Yuqian and Qu, Qing and Wright, John},
  journal={arXiv preprint arXiv:2007.06753},
  year={2020}
}

@inproceedings{wang2022linear,
  title={Linear convergence analysis of neural collapse with unconstrained features},
  author={Wang, Peng and Liu, Huikang and Yaras, Can and Balzano, Laura and Qu, Qing},
  booktitle={OPT 2022: Optimization for Machine Learning (NeurIPS 2022 Workshop)},
  year={2022}
}

@inproceedings{li2024neural,
title={Neural Collapse in Multi-label Learning with Pick-all-label Loss},
author={Li, Pengyu and Li, Xiao and Wang, Yutong and Qu, Qing},
booktitle={Forty-first International Conference on Machine Learning},
year={2024},
url={https://openreview.net/forum?id=y8NevOhrnW}
}

@inproceedings{zhou2022optimization,
  title={On the optimization landscape of neural collapse under mse loss: Global optimality with unconstrained features},
  author={Zhou, Jinxin and Li, Xiao and Ding, Tianyu and You, Chong and Qu, Qing and Zhu, Zhihui},
  booktitle={International Conference on Machine Learning},
  pages={27179--27202},
  year={2022},
  organization={PMLR}
}

@article{sun2017complete_a,
   title={Complete Dictionary Recovery Over the Sphere I: Overview and the Geometric Picture},
   author={Sun, Ju and Qu, Qing and Wright, John}, 
   journal={IEEE Transactions on Information Theory}, 
   year={2017}, 
   volume={63}, 
   number={2}, 
   pages={853-884},
}

@inproceedings{qu2020geometric,
title={Geometric Analysis of Nonconvex Optimization Landscapes for Overcomplete Learning},
author={Qing Qu and Yuexiang Zhai and Xiao Li and Yuqian Zhang and Zhihui Zhu},
booktitle={International Conference on Learning Representations},
year={2020},
url={https://openreview.net/forum?id=rygixkHKDH}
}

@article{sun2015nonconvex,
  title={When are nonconvex problems not scary?},
  author={Sun, Ju and Qu, Qing and Wright, John},
  journal={arXiv preprint arXiv:1510.06096},
  year={2015}
}

@inproceedings{jiang2024generalized,
title={Generalized Neural Collapse for a Large Number of Classes},
author={Jiang, Jiachen and Zhou, Jinxin and Wang, Peng and Qu, Qing and Mixon, Dustin and You, Chong and Zhu, Zhihui},
booktitle={International Conference on Machine Learning},
year={2024},
url={https://openreview.net/forum?id=D4B7kkB89m}
}

@book{nocedal1999numerical,
  title={Numerical optimization},
  author={Nocedal, Jorge and Wright, Stephen J},
  year={1999},
  publisher={Springer}
}

@InProceedings{yu2023emergence,
  title =    {Emergence of Segmentation with Minimalistic White-Box Transformers},
  author =       {Yu, Yaodong and Chu, Tianzhe and Tong, Shengbang and Wu, Ziyang and Pai, Druv and Buchanan, Sam and Ma, Yi},
  booktitle =    {Conference on Parsimony and Learning},
  pages =    {72--93},
  year =     {2024},
  volume =   {234},
  series =   {Proceedings of Machine Learning Research},
  publisher =    {PMLR},
}

@Article{lecun2015deep,
  author    = {LeCun, Yann and Bengio, Yoshua and Hinton, Geoffrey},
  journal   = {nature},
  title     = {Deep learning},
  year      = {2015},
  number    = {7553},
  pages     = {436--444},
  volume    = {521},
  publisher = {Nature Publishing Group UK London},
}

@InProceedings{he2016deep,
  author    = {He, Kaiming and Zhang, Xiangyu and Ren, Shaoqing and Sun, Jian},
  booktitle = {Proceedings of the IEEE conference on computer vision and pattern recognition},
  title     = {Deep residual learning for image recognition},
  year      = {2016},
  pages     = {770--778},
}

@Article{simonyan2014very,
  author  = {Simonyan, Karen and Zisserman, Andrew},
  journal = {arXiv preprint arXiv:1409.1556},
  title   = {Very deep convolutional networks for large-scale image recognition},
  year    = {2014},
}

@Article{sutskever2014sequence,
  author  = {Sutskever, Ilya and Vinyals, Oriol and Le, Quoc V},
  journal = {Advances in neural information processing systems},
  title   = {Sequence to sequence learning with neural networks},
  year    = {2014},
  volume  = {27},
}

@Article{vaswani2017attention,
  author  = {Vaswani, Ashish and Shazeer, Noam and Parmar, Niki and Uszkoreit, Jakob and Jones, Llion and Gomez, Aidan N and Kaiser, {\L}ukasz and Polosukhin, Illia},
  journal = {Advances in neural information processing systems},
  title   = {Attention is all you need},
  year    = {2017},
  volume  = {30},
}

@Article{esteva2019guide,
  author    = {Esteva, Andre and Robicquet, Alexandre and Ramsundar, Bharath and Kuleshov, Volodymyr and DePristo, Mark and Chou, Katherine and Cui, Claire and Corrado, Greg and Thrun, Sebastian and Dean, Jeff},
  journal   = {Nature medicine},
  title     = {A guide to deep learning in healthcare},
  year      = {2019},
  number    = {1},
  pages     = {24--29},
  volume    = {25},
  publisher = {Nature Publishing Group US New York},
}

@Article{papyan2020prevalence,
  author    = {Papyan, Vardan and Han, XY and Donoho, David L},
  journal   = {Proceedings of the National Academy of Sciences},
  title     = {Prevalence of neural collapse during the terminal phase of deep learning training},
  year      = {2020},
  number    = {40},
  pages     = {24652--24663},
  volume    = {117},
  publisher = {National Acad Sciences},
}

@Article{chu2023image,
  author  = {Chu, Tianzhe and Tong, Shengbang and Ding, Tianjiao and Dai, Xili and Haeffele, Benjamin David and Vidal, Rene and Ma, Yi},
  journal = {arXiv preprint arXiv:2306.05272},
  title   = {Image Clustering via the Principle of Rate Reduction in the Age of Pretrained Models},
  year    = {2023},
}

@Article{yu2020learning,
  author  = {Yu, Yaodong and Chan, Kwan Ho Ryan and You, Chong and Song, Chaobing and Ma, Yi},
  journal = {Advances in Neural Information Processing Systems},
  title   = {Learning diverse and discriminative representations via the principle of maximal coding rate reduction},
  year    = {2020},
  pages   = {9422--9434},
  volume  = {33},
}

@inproceedings{
yu2023white,
title={White-Box Transformers via Sparse Rate Reduction},
author={Yaodong Yu and Sam Buchanan and Druv Pai and Tianzhe Chu and Ziyang Wu and Shengbang Tong and Benjamin David Haeffele and Yi Ma},
booktitle={Thirty-seventh Conference on Neural Information Processing Systems},
year={2023},
}

@Article{ma2007segmentation,
  author    = {Ma, Yi and Derksen, Harm and Hong, Wei and Wright, John},
  journal   = {IEEE transactions on pattern analysis and machine intelligence},
  title     = {Segmentation of multivariate mixed data via lossy data coding and compression},
  year      = {2007},
  number    = {9},
  pages     = {1546--1562},
  volume    = {29},
  publisher = {IEEE},
}

@Article{pai2022pursuit,
  author  = {Pai, Druv and Psenka, Michael and Chiu, Chih-Yuan and Wu, Manxi and Dobriban, Edgar and Ma, Yi},
  journal = {arXiv preprint arXiv:2206.09120},
  title   = {Pursuit of a discriminative representation for multiple subspaces via sequential games},
  year    = {2022},
}

@InProceedings{gregor2010learning,
  author    = {Gregor, Karol and LeCun, Yann},
  booktitle = {Proceedings of the 27th international conference on international conference on machine learning},
  title     = {Learning fast approximations of sparse coding},
  year      = {2010},
  pages     = {399--406},
}

@Article{dai2022ctrl,
  author    = {Dai, Xili and Tong, Shengbang and Li, Mingyang and Wu, Ziyang and Psenka, Michael and Chan, Kwan Ho Ryan and Zhai, Pengyuan and Yu, Yaodong and Yuan, Xiaojun and Shum, Heung-Yeung and others},
  journal   = {Entropy},
  title     = {Ctrl: Closed-loop transcription to an ldr via minimaxing rate reduction},
  year      = {2022},
  number    = {4},
  pages     = {456},
  volume    = {24},
  publisher = {MDPI},
}

@inproceedings{wang2021optimal,
  title={Optimal non-convex exact recovery in stochastic block model via projected power method},
  author={Wang, Peng and Liu, Huikang and Zhou, Zirui and So, Anthony Man-Cho},
  booktitle={International Conference on Machine Learning},
  pages={10828--10838},
  year={2021},
  organization={PMLR}
}

@article{wang2022non,
  title={Non-convex exact community recovery in stochastic block model},
  author={Wang, Peng and Zhou, Zirui and So, Anthony Man-Cho},
  journal={Mathematical Programming},
  volume={195},
  number={1},
  pages={1--37},
  year={2022},
  publisher={Springer}
}

@article{xu2018first,
  title={First-order stochastic algorithms for escaping from saddle points in almost linear time},
  author={Xu, Yi and Jin, Rong and Yang, Tianbao},
  journal={Advances in neural information processing systems},
  volume={31},
  year={2018}
}

@inproceedings{daneshmand2018escaping,
  title={Escaping saddles with stochastic gradients},
  author={Daneshmand, Hadi and Kohler, Jonas and Lucchi, Aurelien and Hofmann, Thomas},
  booktitle={International Conference on Machine Learning},
  pages={1155--1164},
  year={2018},
  organization={PMLR}
}

@Article{zhu2021geometric,
  author  = {Zhu, Zhihui and Ding, Tianyu and Zhou, Jinxin and Li, Xiao and You, Chong and Sulam, Jeremias and Qu, Qing},
  journal = {Advances in Neural Information Processing Systems},
  title   = {A Geometric Analysis of Neural Collapse with Unconstrained Features},
  year    = {2021},
  volume  = {34},
}

@Article{oord2018representation,
  author  = {Oord, Aaron van den and Li, Yazhe and Vinyals, Oriol},
  journal = {arXiv preprint arXiv:1807.03748},
  title   = {Representation learning with contrastive predictive coding},
  year    = {2018},
}

@InProceedings{he2020momentum,
  author    = {He, Kaiming and Fan, Haoqi and Wu, Yuxin and Xie, Saining and Girshick, Ross},
  booktitle = {Proceedings of the IEEE/CVF conference on computer vision and pattern recognition},
  title     = {Momentum contrast for unsupervised visual representation learning},
  year      = {2020},
  pages     = {9729--9738},
}

@InProceedings{rifai2011contractive,
  author    = {Rifai, Salah and Vincent, Pascal and Muller, Xavier and Glorot, Xavier and Bengio, Yoshua},
  booktitle = {Proceedings of the 28th international conference on international conference on machine learning},
  title     = {Contractive auto-encoders: Explicit invariance during feature extraction},
  year      = {2011},
  pages     = {833--840},
}

@Book{cover1999elements,
  author    = {Cover, Thomas M},
  publisher = {John Wiley \& Sons},
  title     = {Elements of information theory},
  year      = {1999},
}

@Article{yaras2022neural,
  author  = {Yaras, Can and Wang, Peng and Zhu, Zhihui and Balzano, Laura and Qu, Qing},
  journal = {Advances in neural information processing systems},
  title   = {Neural collapse with normalized features: A geometric analysis over the riemannian manifold},
  year    = {2022},
  pages   = {11547--11560},
  volume  = {35},
}

@Article{lu2017expressive,
  author  = {Lu, Zhou and Pu, Hongming and Wang, Feicheng and Hu, Zhiqiang and Wang, Liwei},
  journal = {Advances in neural information processing systems},
  title   = {The expressive power of neural networks: A view from the width},
  year    = {2017},
  volume  = {30},
}

@Article{mixon2020neural,
  author  = {Mixon, Dustin G and Parshall, Hans and Pi, Jianzong},
  journal = {arXiv preprint arXiv:2011.11619},
  title   = {Neural collapse with unconstrained features},
  year    = {2020},
}

@Article{wang2023understanding,
  author  = {Wang, Peng and Li, Xiao and Yaras, Can and Zhu, Zhihui and Balzano, Laura and Hu, Wei and Qu, Qing},
  journal = {arXiv preprint arXiv:2311.02960},
  title   = {Understanding Deep Representation Learning via Layerwise Feature Compression and Discrimination},
  year    = {2023},
}

@InProceedings{jin2017escape,
  author       = {Jin, Chi and Ge, Rong and Netrapalli, Praneeth and Kakade, Sham M and Jordan, Michael I},
  booktitle    = {International conference on machine learning},
  title        = {How to escape saddle points efficiently},
  year         = {2017},
  organization = {PMLR},
  pages        = {1724--1732},
}

@Article{lee2019first,
  author    = {Lee, Jason D and Panageas, Ioannis and Piliouras, Georgios and Simchowitz, Max and Jordan, Michael I and Recht, Benjamin},
  journal   = {Mathematical programming},
  title     = {First-order methods almost always avoid strict saddle points},
  year      = {2019},
  pages     = {311--337},
  volume    = {176},
  publisher = {Springer},
}

@Article{yu2023sparse,
  author  = {Yu, Yaodong and Buchanan, Sam and Pai, Druv and Chu, Tianzhe and Wu, Ziyang and Tong, Shengbang and Bai, Hao and Zhai, Yuexiang and Haeffele, Benjamin D and Ma, Yi},
  journal = {arXiv preprint arXiv:2311.13110},
  title   = {White-Box Transformers via Sparse Rate Reduction: Compression Is All There Is?},
  year    = {2023},
}

@Article{krizhevsky2009learning,
  author    = {Krizhevsky, Alex and Hinton, Geoffrey and others},
  title     = {Learning multiple layers of features from tiny images},
  year      = {2009},
  publisher = {Toronto, ON, Canada},
}

@Article{ma2022principles,
  author    = {Ma, Yi and Tsao, Doris and Shum, Heung-Yeung},
  journal   = {Frontiers of Information Technology \& Electronic Engineering},
  title     = {On the principles of parsimony and self-consistency for the emergence of intelligence},
  year      = {2022},
  number    = {9},
  pages     = {1298--1323},
  volume    = {23},
  publisher = {Springer},
}

@Article{masarczyk2023tunnel,
  author  = {Masarczyk, Wojciech and Ostaszewski, Mateusz and Imani, Ehsan and Pascanu, Razvan and Mi{\l}o{\'s}, Piotr and Trzci{\'n}ski, Tomasz},
  journal = {arXiv preprint arXiv:2305.19753},
  title   = {The Tunnel Effect: Building Data Representations in Deep Neural Networks},
  year    = {2023},
}

@InProceedings{lee2016gradient,
  author       = {Lee, Jason D and Simchowitz, Max and Jordan, Michael I and Recht, Benjamin},
  booktitle    = {Conference on learning theory},
  title        = {Gradient descent only converges to minimizers},
  year         = {2016},
  organization = {PMLR},
  pages        = {1246--1257},
}

@Article{dosovitskiy2020image,
  author  = {Dosovitskiy, Alexey and Beyer, Lucas and Kolesnikov, Alexander and Weissenborn, Dirk and Zhai, Xiaohua and Unterthiner, Thomas and Dehghani, Mostafa and Minderer, Matthias and Heigold, Georg and Gelly, Sylvain and others},
  journal = {arXiv preprint arXiv:2010.11929},
  title   = {An image is worth 16x16 words: Transformers for image recognition at scale},
  year    = {2020},
}

@Article{kirillov2023segment,
  author  = {Kirillov, Alexander and Mintun, Eric and Ravi, Nikhila and Mao, Hanzi and Rolland, Chloe and Gustafson, Laura and Xiao, Tete and Whitehead, Spencer and Berg, Alexander C and Lo, Wan-Yen and others},
  journal = {arXiv preprint arXiv:2304.02643},
  title   = {Segment anything},
  year    = {2023},
}

@Article{saharia2022photorealistic,
  author  = {Saharia, Chitwan and Chan, William and Saxena, Saurabh and Li, Lala and Whang, Jay and Denton, Emily L and Ghasemipour, Kamyar and Gontijo Lopes, Raphael and Karagol Ayan, Burcu and Salimans, Tim and others},
  journal = {Advances in Neural Information Processing Systems},
  title   = {Photorealistic text-to-image diffusion models with deep language understanding},
  year    = {2022},
  pages   = {36479--36494},
  volume  = {35},
}

@InProceedings{ge2017no,
  author       = {Ge, Rong and Jin, Chi and Zheng, Yi},
  booktitle    = {International Conference on Machine Learning},
  title        = {No spurious local minima in nonconvex low rank problems: A unified geometric analysis},
  year         = {2017},
  organization = {PMLR},
  pages        = {1233--1242},
}

@Article{chi2019nonconvex,
  author    = {Chi, Yuejie and Lu, Yue M and Chen, Yuxin},
  journal   = {IEEE Transactions on Signal Processing},
  title     = {Nonconvex optimization meets low-rank matrix factorization: An overview},
  year      = {2019},
  number    = {20},
  pages     = {5239--5269},
  volume    = {67},
  publisher = {IEEE},
}

@Article{sun2018geometric,
  author    = {Sun, Ju and Qu, Qing and Wright, John},
  journal   = {Foundations of Computational Mathematics},
  title     = {A geometric analysis of phase retrieval},
  year      = {2018},
  pages     = {1131--1198},
  volume    = {18},
  publisher = {Springer},
}

@Article{sun2020global,
  author    = {Sun, Ruoyu and Li, Dawei and Liang, Shiyu and Ding, Tian and Srikant, Rayadurgam},
  journal   = {IEEE Signal Processing Magazine},
  title     = {The global landscape of neural networks: An overview},
  year      = {2020},
  number    = {5},
  pages     = {95--108},
  volume    = {37},
  publisher = {IEEE},
}

@Article{ansuini2019intrinsic,
  author  = {Ansuini, Alessio and Laio, Alessandro and Macke, Jakob H and Zoccolan, Davide},
  journal = {Advances in Neural Information Processing Systems},
  title   = {Intrinsic dimension of data representations in deep neural networks},
  year    = {2019},
  volume  = {32},
}

@Article{recanatesi2019dimensionality,
  author  = {Recanatesi, Stefano and Farrell, Matthew and Advani, Madhu and Moore, Timothy and Lajoie, Guillaume and Shea-Brown, Eric},
  journal = {arXiv preprint arXiv:1906.00443},
  title   = {Dimensionality compression and expansion in deep neural networks},
  year    = {2019},
}

@Article{hinton2006reducing,
  author    = {Hinton, Geoffrey E and Salakhutdinov, Ruslan R},
  journal   = {science},
  title     = {Reducing the dimensionality of data with neural networks},
  year      = {2006},
  number    = {5786},
  pages     = {504--507},
  volume    = {313},
  publisher = {American Association for the Advancement of Science},
}

@Article{monga2021algorithm,
  author    = {Monga, Vishal and Li, Yuelong and Eldar, Yonina C},
  journal   = {IEEE Signal Processing Magazine},
  title     = {Algorithm unrolling: Interpretable, efficient deep learning for signal and image processing},
  year      = {2021},
  number    = {2},
  pages     = {18--44},
  volume    = {38},
  publisher = {IEEE},
}

@article{bengio2013representation,
  title={Representation learning: A review and new perspectives},
  author={Bengio, Yoshua and Courville, Aaron and Vincent, Pascal},
  journal={IEEE transactions on pattern analysis and machine intelligence},
  volume={35},
  number={8},
  pages={1798--1828},
  year={2013},
  publisher={IEEE}
}

@article{lecun1998gradient,
  title={Gradient-based learning applied to document recognition},
  author={LeCun, Yann and Bottou, L{\'e}on and Bengio, Yoshua and Haffner, Patrick},
  journal={Proceedings of the IEEE},
  volume={86},
  number={11},
  pages={2278--2324},
  year={1998},
  publisher={Ieee}
}

@article{loshchilov2016sgdr,
  title={{SGDR}: Stochastic gradient descent with warm restarts},
  author={Loshchilov, Ilya and Hutter, Frank},
  journal={arXiv preprint arXiv:1608.03983},
  year={2016}
}

@article{kingma2014adam,
  title={Adam: A method for stochastic optimization},
  author={Kingma, Diederik P and Ba, Jimmy},
  journal={arXiv preprint arXiv:1412.6980},
  year={2014}
}

@article{bjorck1973numerical,
  title={Numerical methods for computing angles between linear subspaces},
  author={Bj\"{o}rck, {{\AA}ke} and Golub, Gene H},
  journal={Mathematics of computation},
  volume={27},
  number={123},
  pages={579--594},
  year={1973}
}

@book{horn2012matrix,
  title={Matrix analysis},
  author={Horn, Roger A and Johnson, Charles R},
  year={2012},
  publisher={Cambridge university press}
}

\newpage 

\begin{appendix}
\begin{center}
{\LARGE \bf Supplementary Material}
\end{center} 
\setcounter{section}{0}
\renewcommand\thesection{\Alph{section}}

The organization of the supplementary material is as follows: In  \Cref{sec:grad}, we introduce preliminary setups and auxiliary results for studying the MCR$^2$ problem. Then, we prove the technical results concerning the global optimality of Problem \eqref{eq:MCR}  in \Cref{sec:pf opti} and the optimization landscape of Problem \eqref{eq:MCR} in  \Cref{sec:pf land}, respectively. In \Cref{sec:pf main}, we prove the main theorems in \Cref{thm:1}  and \Cref{thm:2}. Finally, we provide more experimental setups and results in \Cref{app sec:exp}. 

Besides the notions introduced earlier, we shall use $\mathrm{BlkDiag}(\bm X_1,\dots,\bm X_K)$ to denote the block diagonal matrix whose diagonal blocks are $\bm X_1,\dots,\bm X_K$. 

\section{Preliminaries}\label{sec:grad}

In this section, we first introduce the first-order optimality condition and the concept of a strict saddle point for $F(\cdot)$ in Problem \eqref{eq:MCR} in Section \ref{subsec:pre}, and finally present auxiliary results about matrix computations and properties of the log-determinant function in Section \ref{subsec:auxi}. Recall that $\bZ = \left[\bZ_1,\ldots,\bZ_K\right] \in \R^{d \times m}$ with $\bm Z_k \in \R^{d \times m_k}$ for each $k \in [K]$, and $\alpha,\alpha_k$ are defined in \eqref{eq:alp}. To simplify our development, we write $R_c(\bm Z, \bm \pi_k)$ in \eqref{eq:Rc1} as 
\begin{align}\label{eq:Rc}
R_c(\bm Z, \bm \pi_k) := \frac{m_k}{m}R_c(\bm Z_k),\ \text{where}\ R_c(\bm Z_k) := \frac{1}{2}\log\det\left( \bm I + \alpha_k\bm Z_k \bm Z_k^T \right). 
\end{align} 
Therefore, we can write $F(\bm Z)$ in Problem \eqref{eq:MCR} into
\begin{align}
    F(\bm Z) = R(\bm Z) - \sum_{k=1}^K \frac{m_k}{m} R_c(\bm Z_k) - \frac{\lambda}{2}\|\bm Z\|_F^2. 
\end{align}




\subsection{Optimality Conditions and Strict Saddle Points}\label{subsec:pre}

To begin, we compute the gradient and Hessian (in bilinear form along a direction $\bm D \in \R^{d \times m}$) of $R(\cdot)$ in \eqref{eq:R} as follows: 
\begin{align}
	\nabla R(\bm Z) &= \alpha \bm X^{-1}\bm Z, \label{eq:grad R}\\
	\nabla^2 R(\bm Z)[\bm D,\bm D] &= \alpha \langle \bX^{-1}, \bm D\bm D^T \rangle - \frac{\alpha^2}{2} \mathrm{Tr}\left(\bm X^{-1} (\bm Z\bm D^T + \bm D\bm Z^T )\bm X^{-1}(\bm Z\bm D^T + \bm D\bm Z^T )\right),\label{eq:Hess R} 
\end{align}
where $\bm X := \bm I_d + \alpha \bm Z\bm Z^T$ and $\alpha$ is defined in \eqref{eq:alp}. Note that we can compute the gradient and Hessian of $R_c(\cdot)$ in \eqref{eq:Rc} using the same approach. Based on the above setup, we define the first-order optimality condition of Problem \eqref{eq:MCR} as follows. 

\begin{definition}\label{def:FONC}
We say that $\bm Z \in \R^{d \times m}$ is a critical point of Problem \eqref{eq:MCR} if $\nabla F(\bm Z) = \bm 0$, i.e.,
\begin{align}\label{eq:FONC} 
  \alpha(\bm I + \alpha \bm Z\bm Z^T)^{-1}\bm Z_k - \alpha(\bm I + \alpha_k \bm Z_k\bm Z_k^T)^{-1}\bm Z_k - \lambda \bm Z_k  = \bm{0},\ \forall k \in [K], 
\end{align}
where $\alpha$ and $\alpha_k$ are defined in \eqref{eq:alp}. 
\end{definition}

According to \citet{jin2017escape,lee2019first}, we define the strict saddle point, i.e., a critical point that has a direction with strictly positive curvature\footnote{Note that Problem \eqref{eq:MCR} is not a minimization problem but a maximization problem.}, of Problem \eqref{eq:MCR} as follows: 
\begin{definition}\label{def:saddle}
Suppose that $\bZ \in \R^{d\times m}$ is a critical point of Problem \eqref{eq:MCR}. We say that $\bZ$ is its strict saddle point  if there exists a direction $\bD = \left[\bD_1,\ldots,\bD_K\right] \in \R^{d \times m}$ with $\bm D_k \in \R^{d \times m_k}$ such that 
\begin{align*}
\nabla^2 F(\bZ)[\bD ,\bD ] > 0,
\end{align*} 
where 
\begin{align}\label{eq:Hess}
\nabla^2 F(\bZ)[\bD ,\bD] = \nabla^2 R(\bm Z)[\bm D,\bm D] - \sum_{k=1}^K \frac{m_k}{m}\nabla^2 R_c(\bm Z_k)[\bm D_k, \bm D_k] - \lambda\|\bm D\|_F^2.   
\end{align}
\end{definition}
Remark that for the MCR$^2$ problem, strict saddle points include saddle points with strictly positive curvature as well as local maximizers. 

\subsection{Auxiliary Results}\label{subsec:auxi}

We provide a matrix inversion lemma, which is also known as Sherman–Morrison–Woodbury formula. 

\begin{lemma}[Matrix inversion lemma] \label{lem:matrix inv}
For any $\bm Z \in \R^{d\times m}$, we have 
\begin{align}\label{eq:matrix inv Z}
(\bI + \alpha\bZ\bZ^T)^{-1} = \bI - \bZ\left(\frac{1}{\alpha}\bI + \bZ^T\bZ \right)^{-1}\bZ^T. 
\end{align}
\end{lemma}

We next present the commutative property for the log-determinant function and the upper bound for the coding rate function. We refer the reader to \citep[Lemma 8 \& Lemma 10]{chan2022redunet} for the detailed proofs. Here, let $\bm Z = \bm U\bm \Sigma\bm V^T$ be a singular value decompositon of $\bm Z \in \R^{d\times m}$, where $r = \mathrm{rank}(\bm Z) \le \min\{m,d\}$, $\bm U \in \mathcal{O}^{d\times r}$, $\bm \Sigma \in \R^{r\times r}$ is a diagonal matrix, and $\bm V \in \mathcal{O}^{m \times r}$. 

\begin{lemma}[Commutative property]\label{lem:commu} 
For any $\bm Z \in \R^{d\times m}$ and $\alpha  > 0$, we have 
\begin{align}\label{eq:commu}
\frac{1}{2}\log \det\left(\bm I + \alpha\bm Z\bm Z^T \right) = \frac{1}{2}\log \det\left(\bm I + \alpha\bm Z^T\bm Z \right) = \frac{1}{2}\log\det\left(\bm I + \alpha \bm \Sigma^2\right). 
\end{align}

\end{lemma}

\begin{lemma}\label{lem:MCR}
Let $\bm Z = \left[\bm Z_1,\dots,\bm Z_K \right] \in \R^{d\times m}$. Given $\alpha > 0$, it holds that 
\begin{align}\label{eq:MCR upper}
\log \det\left(\bm I + \alpha\bm Z\bm Z^T \right) \le \sum_{k=1}^K \log \det\left(\bm I + \alpha\bm Z_k\bm Z_k^T \right),
\end{align}
where the equality holds if and only if $\bm Z_k^T\bm Z_l = \bm 0$ for all $k \neq l \in [K]$. 
\end{lemma}

Finally, we show that the objective function of Problem \eqref{eq:MCR} is invariant under the block diagonal orthogonal matrices.  
\begin{lemma}\label{lem:equi}
For any $\bm O = \mathrm{BlkDiag}\left( \bm O_1,\dots,\bm O_K \right)$, where $\bm O_k \in \mathcal{O}^{m_k}$ for each $k \in [K]$, we have 
\begin{align}\label{eq:equi}
F(\bm Z\bm O) = F(\bm Z),\quad \nabla F(\bm Z\bm O) = \nabla F(\bm Z)\bm O,\quad \nabla^2 F(\bm Z\bm O)[\bm D\bm O, \bm D\bm O] = \nabla^2 F(\bm Z)[\bm D, \bm D]. 
\end{align}
\end{lemma}
\begin{proof}[Proof of \Cref{lem:equi}]
Let $\bm O_k \in \mathcal{O}^{m_k}$ be arbitrary for each $k \in [K]$ and $\bm O = \mathrm{BlkDiag}\left( \bm O_1,\dots,\bm O_K \right)$. According to \eqref{eq:R} and \eqref{eq:Rc}, we have $R(\bm Z\bm O) = R(\bm Z)$ and $R_c(\bm Z_k\bm O_k) = R(\bm Z_k)$. This, together with \eqref{eq:MCR}, yields that $F(\bm Z\bm O) = F(\bm Z)$. Moreover, it follows from \eqref{eq:grad R} that $\nabla R(\bm Z\bm O) = \nabla R(\bm Z)\bm O$ and $\nabla R_c(\bm Z_k\bm O_k) = \nabla R_c(\bm Z_k)\bm O_k$. This implies $\nabla F(\bm Z\bm O) = \nabla F(\bm Z)\bm $. Finally, using \eqref{eq:Hess R}, we have $\nabla^2 R(\bm Z\bm O)[\bm D\bm O,\bm D\bm O] = \nabla^2 R(\bm Z)[\bm D,\bm D]$ and $\nabla^2 R_c(\bm Z_k\bm O_k)[\bm D_k\bm O_k,\bm D_k\bm O_k] = \nabla^2 R(\bm Z_k)[\bm D_k,\bm D_k]$. This, together with \eqref{eq:Hess}, implies $\nabla^2 F(\bm Z\bm O)[\bm D\bm O, \bm D\bm O] = \nabla^2 F(\bm Z)[\bm D, \bm D]$. 
\end{proof}

\section{Proofs in \Cref{subsec:pf opti}}\label{sec:pf opti}

\subsection{Proof of \Cref{lem:upper bound}}

\begin{proof}[Proof of Lemma \ref{lem:upper bound}]
It follows from \eqref{eq:MCR upper} in \Cref{lem:MCR} that 
\begin{align}
 \log \det\left( \bm I_d + \alpha \bm Z \bm Z^T \right) \le \sum_{k=1}^K \log\det\left( \bm I_d + \alpha \bm Z_k \bm Z_k^T \right),
\end{align}
where the equality holds if and only if  $\bm Z_k^T\bm Z_l = \bm 0$ for all $1 \le k \neq l \le K$. Substituting this into \eqref{eq:MCR} directly yields \eqref{eq:F upper}. 
\end{proof}

\subsection{Proof of \Cref{prop:set Z}}

\begin{proof}[Proof of \Cref{prop:set Z}]
Let $\bm Z \in \cal Z$ be arbitrary, where $\mathcal{Z}$ is defined in \eqref{eq:set Z}. It follows from $\bZ = \left[\bZ_1,\ldots,\bZ_K\right] \in \R^{d \times m}$ that $\sum_{k=1}^K r_k \le d$. According to Lemma \ref{lem:upper bound} and  $\bm Z_k^T\bm Z_l = \bm 0$ for all $k \neq l$ due to $\bm Z \in \cal Z$, we have $F(\bm Z) = \sum_{k=1}^K f_k(\bm Z_k)$, where $f_k:\R^{d\times m_k}\rightarrow \R$ takes the form of
\begin{align}\label{eq:fk}
f_k(\bm Z_k) :=  \frac{1}{2} \log\det\left( \bm I_d + \alpha \bm Z_k \bm Z_k^T \right) - \frac{m_k}{2m} \log\det\left( \bm I_d + \alpha_k \bm Z_k \bm Z_k^T \right) - \frac{\lambda}{2} \|\bm Z_k\|_F^2. 
\end{align}
This, together with \eqref{eq:grad R}, yields that $\nabla F(\bm Z) = \bm 0$ is equivalent to
\begin{align}\label{eq1:prop set Z}
  \alpha\left(\bm I_d + \alpha \bm Z_k\bm Z_k^T \right)^{-1}\bm Z_k - \alpha\left(\bm I_d + \alpha_k \bm Z_k\bm Z_k^T \right)^{-1}\bm Z_k = \lambda \bm Z_k,\ \forall k \in [K]. 
\end{align}
Obviously, $\bm Z_k = \bm 0$ is a solution of the above equation for each $k \in [K]$, which satisfies $\bm Z_k^T\bm Z_l = \bm 0$ for all $l \neq k$. Now, we consider $\bm Z_k \neq \bm 0$, and thus $1 \le r_k = \mathrm{rank}(\bm Z_k) \le \min\{m_k,d\}$.  Let
\begin{align}\label{eq:SVD Zk}
\bm Z_k = \bP_k\bm{\Sigma}_k\bQ_k^T = \begin{bmatrix}
\bP_{k,1} & \bP_{k,2} 
\end{bmatrix} \begin{bmatrix}
\tilde{\bm{\Sigma}}_k & \bm{0} \\
\bm{0} & \bm{0} 
\end{bmatrix} \begin{bmatrix}
\bQ_{k,1}^T \\ \bQ_{k,2}^T 
\end{bmatrix}
\end{align}
be a singular value decomposition (SVD) of $\bm Z_k \in \R^{d\times m_k}$, where $\tilde{\bm{\Sigma}}_k = \mathrm{diag}(\sigma_{k,1},\dots,\sigma_{k,r_k})$ with $\sigma_{k,1}\ge\dots\ge \sigma_{k,r_k} > 0$ being positive singular values of $\bZ_k$, $\bP_k \in \mO^d$ with $\bP_{k,1} \in \R^{d \times r_k}$ and $\bP_{k,2} \in \R^{d \times (d-r_k)}$, and $\bm{Q}_k \in \mO^{m_k}$ with $\bQ_{k,1} \in \R^{m_k\times r_k}$ and $\bQ_{k,2} \in \R^{m_k\times (m_k-r_k)}$. Substituting this SVD into \eqref{eq1:prop set Z} yields for all $k \in [K]$,  
\begin{align*}
  \alpha \bm P_k (\bI_d + \alpha\bm{\Sigma}_k\bm{\Sigma}_k^T)^{-1} \bm{\Sigma}_k \bm Q_k^T - \alpha \bm P_k (\bI_d + \alpha_k\bm{\Sigma}_k\bm{\Sigma}_k^T)^{-1} \bm{\Sigma}_k \bm Q_k^T  =  \lambda \bm P_k\bm{\Sigma}_k\bQ_k^T,
\end{align*}
which is equivalent to
\begin{align*}
\alpha  (\bI_d + \alpha\bm{\Sigma}_k\bm{\Sigma}_k^T)^{-1} \bm{\Sigma}_k  - \alpha  (\bI_d + \alpha_k\bm{\Sigma}_k\bm{\Sigma}_k^T)^{-1} \bm{\Sigma}_k  =  \lambda \bm{\Sigma}_k. 
\end{align*}
Using $\bm{\Sigma}_k = \mathrm{BlkDiag}(\tilde{\bm \Sigma}_k,\bm 0)$, we further obtain
\begin{align*}
\alpha  (\bI_{r_k} + \alpha\tilde{\bm{\Sigma}}_k^2)^{-1} \tilde{\bm{\Sigma}}_k  - \alpha  (\bI_{r_k} + \alpha_k\tilde{\bm{\Sigma}}_k^2)^{-1} \tilde{\bm{\Sigma}}_k =  \lambda \tilde{\bm{\Sigma}}_k. 
\end{align*}
Since $\tilde{\bm{\Sigma}}_k$ is a diagonal matrix with diagonal entries being positive, we have for all $k \in [K]$,  
\begin{align}\label{eq2:prop set Z}
(\bI_{r_k} + \alpha\tilde{\bm{\Sigma}}_k^2)^{-1} - (\bI_{r_k} + \alpha_k\tilde{\bm{\Sigma}}_k^2)^{-1} = \frac{\lambda}{\alpha}\bm I_{r_k}. 
\end{align}
This implies for each $i \in [r_k]$ and $k \in [K]$,
\begin{align}\label{eq4:prop set Z}
\frac{1}{1+\alpha\sigma_{k,i}^2} - \frac{1}{1+\alpha_k\sigma_{k,i}^2} = \frac{\lambda}{\alpha}. 
\end{align}
Therefore, we obtain that $\sigma_{k,i}^2 > 0$ for each $i \in [r_k]$ is a positive root of the following quadratic equation with a variable $x \in \R$: 
\begin{align*}
\lambda \alpha_k x^2 -\eta_k x  + \frac{\lambda}{\alpha} = 0,
\end{align*}
where 
\begin{align}\label{eq:eta k}
\eta_k := (\alpha_k-\alpha) - \lambda\left(1 + \frac{\alpha_k}{\alpha} \right),\ \forall k \in [K]. 
\end{align}
According to \eqref{eq:lambda}, one can verify that for each $k \in [K]$,
\begin{align*}
\eta_k > 0,\quad \eta_k^2 - \frac{4\alpha_k}{\alpha}\lambda^2 \ge 0. 
\end{align*}
This yields that the above quadratic equation has positive roots as follows. For each $i \in [r_k]$ and $k \in [K]$, we have 
\begin{align}\label{eq:sig}
\sigma_{k,i}^2 = \frac{ \eta_k \pm \sqrt{\eta_k^2 - 4\lambda^2m/m_k}}{2\lambda \alpha_k}. 
\end{align}
Finally, using $\bm Z_k^T\bm Z_l = \bm 0$ and \eqref{eq:SVD Zk}, we obtain $\bm P_{k,1}^T \bm P_{l,1} = \bm 0$  for all $1 \le l \neq k \le K$. These, together with \eqref{eq:SVD Zk}, yields \eqref{eq:Zk}.  

Conversely, suppose that each block $\bm Z_k$ of $\bm Z$ satisfies $\bm Z_k=\bm 0$ or takes the form \eqref{eq:Zk} for some $\bm U_k \in \mO^{d\times r_k}$ satisfying $\bm U_k^T\bm U_l = \bm 0$ for all $l \neq k$, $\bm V_k \in \mO^{m_k \times r_k}$ for all $k \in [K]$, and $\sigma_{k,i} > 0$ satisfying \eqref{eq:gamma}. We are devoted to showing $\bm Z \in \cal Z$. It is straightforward to verify that $\bm Z_k^T\bm Z_l = \bm 0$ for all $1 \le k \neq l \le K$. This, together with Lemma \ref{lem:upper bound}, implies $F(\bm Z) = \sum_{k=1}^K f_k(\bm Z_k)$. Therefore, it suffices to verify that $\nabla f_k(\bm Z_k) = \bm 0$ for each $k \in [K]$ in the rest of the proof. For each $k \in [K]$, if $\bm Z_k = \bm 0$, it is obvious to verify $\nabla f(\bm Z_k) = \bm 0$. Otherwise, $\bm Z_k$ takes the form \eqref{eq:Zk} for some $\bm U_k \in \mO^{d \times r_k}$, $\tilde{\bm \Sigma}_k \in \R^{r_k\times r_k}$ satisfying \eqref{eq:gamma}, and $\bm V_k \in \mO^{m_k \times r_k}$, where $r_k \ge 1$. Now, we compute for all $i\in [r_k]$,
\begin{align}\label{eq3:prop set Z}
\frac{\sigma_{k,i}^2}{1/\alpha_k + \sigma_{k,i}^2} - \frac{\sigma_{k,i}^2}{1/\alpha + \sigma_{k,i}^2} & = \frac{\alpha_k\sigma_{k,i}^2}{1 + \alpha_k\sigma_{k,i}^2} - \frac{\alpha\sigma_{k,i}^2}{1 + \alpha\sigma_{k,i}^2} = \frac{(\alpha_k-\alpha)\sigma_{k,i}^2}{\left( 1 + \alpha_k\sigma_{k,i}^2 \right)\left( 1 + \alpha\sigma_{k,i}^2 \right)} \notag\\
&= \frac{1}{1 + \alpha\sigma_{k,i}^2} - \frac{1}{1 + \alpha_k\sigma_{k,i}^2} = \frac{\lambda}{\alpha},
\end{align}
where the last equality is due to \eqref{eq:gamma}, \eqref{eq:sigma1}, and \eqref{eq4:prop set Z}. Then, we compute	
\begin{align}\label{eq5:prop set Z}
 \left(\bm I_d + \alpha \bm Z_k\bm Z_k^T \right)^{-1} = \left(\bm I_d + \alpha \bm U_k\tilde{\bm \Sigma}_k^2 \bm U_k^T \right)^{-1} = \bm I_d -  \bm U_k\tilde{\bm \Sigma}_k\left(\frac{1}{\alpha}\bm I_{r_k} + \tilde{\bm \Sigma}_k^2 \right)^{-1}  \tilde{\bm \Sigma}_k\bm U_k^T. 
\end{align}
where the second equality follows from \eqref{eq:matrix inv Z}. This, together with \eqref{eq:grad R}, yields 
\begin{align*}
\nabla f_k(\bm Z_k) & =  \alpha\left(\bm I_d + \alpha \bm Z_k\bm Z_k^T \right)^{-1}\bm Z_k - \alpha\left(\bm I_d + \alpha_k \bm Z_k\bm Z_k^T \right)^{-1}\bm Z_k - \lambda \bm Z_k \\
& = \alpha \bm U_k\tilde{\bm \Sigma}_k \left( \left(\frac{1}{\alpha_k}\bm I_{r_k} + \tilde{\bm \Sigma}_k^2 \right)^{-1}  - \left(\frac{1}{\alpha}\bm I_{r_k} + \tilde{\bm \Sigma}_k^2 \right)^{-1}  \right)\tilde{\bm \Sigma}_k ^2 \bm V_k^T - \lambda\bm Z_k = \bm 0,
\end{align*}
where the last equality follows from \eqref{eq:Zk} and \eqref{eq3:prop set Z}. Therefore, we have $\nabla F(\bm Z) = \bm 0$ as desired. This, together with $\bm Z_k^T\bm Z_l = \bm 0$, implies $\bm Z \in \cal Z$. 
\end{proof}  

\section{Proofs in \Cref{subsec:pf land}}\label{sec:pf land}


\subsection{Proof of \Cref{prop:local max}} 

\begin{proof}[Proof of \Cref{prop:local max}]

For each $\bm Z \in \cal Z$, it follows from \Cref{lem:upper bound} that 
\begin{align}\label{eq0:prop local max}
F(\bm Z) = \sum_{k=1}^K f_k(\bm Z_k),
\end{align} 	
where $f_k$ is defined in \eqref{eq:fk}. Suppose that there exists $k \in [K]$ such that $r_k = 0$, i.e., $\bm Z_k = \bm 0$. According to \eqref{eq:Hess R} and \eqref{eq:fk}, we compute for any $\bm D_k \neq \bm 0$, 
\begin{align*}
\nabla f_k(\bm Z_k)[\bm D_k, \bm D_k] = \left( \frac{\alpha}{2} - \frac{m_k}{2m}\alpha_k - \lambda \right) \|\bm D_k\|_F^2 = -\lambda  \|\bm D_k\|_F^2  < 0,
\end{align*} 
where the second equality follows from $m_k\alpha_k/m  = \alpha$ according to \eqref{eq:alp}. This implies $\bm 0$ is a local maximizer of $f_k(\bm Z_k)$. Suppose to the contrary that $r_k > 0$ for all $k \in [K]$. For each $\bm Z \in \cal Z$, using \Cref{lem:upper bound} with $\bm Z_k^T\bm Z_l = \bm 0$ for all $k \neq l$, \eqref{eq:Zk}, and \eqref{eq:fk}, we have 
\begin{align}\label{eq1:prop local max} 
F(\bm Z) & = \sum_{k=1}^K  \left( \frac{1}{2} \log\det\left( \bm I_n + \alpha \bm U_k\tilde{\bm \Sigma}_k^2 \bm U_k^T \right) - \frac{m_k}{2m} \log\det\left( \bm I_n + \alpha_k \bm U_k\tilde{\bm \Sigma}_k^2 \bm U_k^T \right) - \frac{\lambda}{2} \|\bm Z_k\|_F^2 \right) \notag\\
& = \sum_{k=1}^K \left( \frac{1}{2} \log\det \left(\bm I + \alpha \tilde{\bm \Sigma}_k^2 \right) - \frac{m_k}{2m} \log\det \left(\bm I + \alpha_k \tilde{\bm \Sigma}_k^2 \right) - \frac{\lambda}{2}\|\tilde{\bm \Sigma}_k\|_F^2  \right) \notag\\
& = \frac{1}{2} \sum_{k=1}^K \sum_{i=1}^{r_k} \left(  \log\left( 1 + \alpha\sigma_{k,i}^2 \right) - \frac{m_k}{m} \log\left( 1 + \alpha_k\sigma_{k,i}^2 \right) - \lambda \sigma_{k,i}^2  \right),
\end{align}
where the second equality is due to \eqref{eq:Zk} and Lemma \ref{lem:commu}. For ease of exposition, let 
\begin{align}\label{eq:h}
    h_k(x) =\log\left( 1 + \alpha x \right) - \frac{m_k}{m}\log\left( 1 + \alpha_k x \right) - \lambda x,\ \forall k \in [k]. 
\end{align}
Using \eqref{eq:gamma}, \eqref{eq4:prop set Z}, \eqref{eq:eta k}, and \eqref{eq:sig}, one can verify that $h_k^\prime(x) \le 0$ for $x \in (0,\underline{\sigma}_k)$, $h_k^\prime(x) \ge 0$ for $x \in [\underline{\sigma}_k,\overline{\sigma}_k)$, and $h_k^\prime(x) \le 0$ for $x \in [\overline{\sigma}_k,\infty)$ for all $k \in [K]$. This yields that $h_k(\underline{\sigma}_k)$ is a local minimizer and $h(\overline{\sigma}_k)$ is a local maximizer. This, together with \eqref{eq1:prop local max} and the fact that $\bm 0$ is a local maximizer of $f_k(\bm Z_k)$, implies (i) and (ii). 

\end{proof}

\subsection{Proof of \Cref{prop:saddle}}

\begin{proof}[Proof of \Cref{prop:saddle}]
Note that $\bm Z \in \R^{d\times m}$ is a critical point that satisfies \eqref{eq:FONC}. Suppose that $\mathrm{rank}\left(\bm Z\right) = r$ and $\mathrm{rank}\left(\bm Z_k\right) = r_k$ for all $k \in [K]$. Obviously, we have $r_k \le \min\{m_k,d\}$ for all $k \in [K]$ and $\sum_{k=1}^K r_k \le r \le \min\{m,d\}$. Now, let $ \bm Z\bm Z^T = \bm Q \bm \Lambda \bm Q^T $ be an eigenvalue decomposition of $\bm Z\bm Z^T \in \S_+^d$, where $\bm Q \in \mathcal{O}^{d \times r}$ and $\bm \Lambda \in \R^{r\times r}$ is a  diagonal matrix with diagonal entries being positive eigenvalues of $\bm Z\bm Z^T$. Suppose that $\bm Z\bm Z^T$ has $p$ distinct positive eigenvalues, where $1 \le p \le r$. Let $\lambda_1 > \dots > \lambda_p > 0$ be its distinct eigenvalue values with the corresponding multiplicities being $h_1,\dots,h_{p} \in \mathbb{N}_+$, respectively. Obviously, we have $\sum_{i=1}^p h_i = r$. Therefore, we write 
\begin{align}\label{eq1:prop saddle}
\bm \Lambda = \mathrm{BlkDiag}\left(\lambda_1\bm I_{h_1},\dots, \lambda_p\bm I_{h_p}\right),\ \bm Q = \begin{bmatrix}
\bm Q_1,\dots, \bm Q_p
\end{bmatrix}, 
\end{align}
where $\bm Q_i \in \mathcal{O}^{d \times h_i}$ for all $i \in [p]$. 

According to \Cref{lem:equi}, we can see that $\bm Z$ is a critical point with curvature if and only if $\bm Z\bm O$ is a critical point with the same curvature for each $\bm O = \mathrm{BlkDiag}\left(\bm O_1,\dots,\bm O_K \right)$ with $\bm O_k \in \mathcal{O}^{m_k}$ for all $k \in [K]$. According to the SVD of $\bm Z_k$ in \eqref{eq:SVD Zk}, we can take $\bm O_k = \bm Q_k$ for each $k \in [K]$. Therefore, it suffices to study $\bm Z_k = \bm P_k \bm \Sigma_k $ for each $k \in [K]$.    Substituting this into \eqref{eq:FONC} in \Cref{def:FONC} gives   
\begin{align*}
\alpha(\bm I + \alpha \bm Z\bm Z^T)^{-1}\bm P_k\bm \Sigma_k - \alpha\bm P_k (\bm I + \alpha_k\bm \Sigma_k\bm \Sigma_k^T)^{-1}\bm \Sigma_k - \lambda \bm P_k\bm \Sigma_k  = \bm{0},\ \forall k \in [K].
\end{align*}
This is equivalent to 
\begin{align*}
\alpha (\bm I + \alpha \bm Z\bm Z^T)^{-1}\bm Z_k = \bm Z_k \left(\alpha  (\bm I + \alpha_k\bm \Sigma_k\bm \Sigma_k^T)^{-1} + \lambda \bm I \right),\ \forall k \in [K]. 
\end{align*}
This yields that each column of $\bm Z_k$ is an eigenvector of $\bm Z$ for each $k \in [K]$. This, together with the decomposition in \eqref{eq1:prop saddle}, yields that we can permute the columns of $\bm Z_k$ such that the columns belonging to the space spanned by $\bm Q_i$ are rearranged together. Let $s_{k,i} \in \mathbb{N}$ denote the number of columns of $\bm Z_k$ that belong to the space spanned by $\bm Q_i$ for each $i \in [p]$. Obviously, we have $\sum_{i=1}^p s_{k,i} = m_k$. Consequently, for each $k \in [K]$, there exists an a column permutation matrix $\bm \Pi_k \in \R^{m_k\times m_k}$ such that 
\begin{align}\label{eq:Zki}
\bm Z_k \bm \Pi_k = \begin{bmatrix}
\bm Z_k^{(1)} & \dots & \bm Z_k^{(p)}
\end{bmatrix}. 
\end{align}
where  $\bm Q_i\bm Q_i^T\bm Z_k^{(i)} = \bm Z_k^{(i)} \in \R^{d \times s_{k,i}}$. Since $\bm Q_i^T\bm Q_j = \bm 0$, we have $\bm Z_k^{(i)^T} \bm Z_k^{(j)} = \bm 0$ for all $i \neq j$. This, together with \eqref{eq:Rc} and Lemma \ref{lem:MCR}, yields 
\begin{align}\label{eq:prop Rc}
R_c(\bm Z_k) = \frac{m_k}{2m} \sum_{i=1}^p \log\det\left( \bm I_n + \alpha_k \bm Z_k^{(i)}\bm Z_k^{(i)^T} \right). 
\end{align}   
Moreover, let  $s_i := \sum_{k=1}^K s_{k,i}$ and 
\begin{align}\label{eq:Zi}
\bm Z^{(i)} := \left[ \bm Z_{1}^{(i)}\  \dots\  \bm Z_{K}^{(i)} \right] \in \R^{d\times s_i},\ \forall i \in [p].
\end{align}
Using this and \eqref{eq:Zki}, we have 
\begin{align*}
\bm Z \bm Z^T = \sum_{k=1}^K \bm Z_k\bm Z_k^T = \sum_{k=1}^K\sum_{i=1}^p \bm Z_k^{(i)}\bm Z_k^{(i)^T} = \sum_{i=1}^p \bm Z^{(i)}\bm Z^{(i)^T}. 
\end{align*}
This, together with \eqref{eq:R}, \Cref{lem:MCR}, and $\bm Z^{(i)^T}\bm Z^{(j)} = \bm 0$, yields that 
\begin{align}\label{eq9:prop saddle}
R(\bm Z) = \frac{1}{2}\sum_{i=1}^p \log\det\left( \bm I + \alpha \bm Z^{(i)}\bm Z^{(i)^T} \right). 
\end{align} 

\paragraph{Characterize the structure of critical points.} Now, for each $k \in [K]$ and $i \in [p]$, let $r_{k,i} = \mathrm{rank}(\bm Z_k^{(i)})$, where $r_{k,i} \le \min\{d,s_{k,i}\}$. Moreover, let 
\begin{align}\label{eq:SVD Zk1}
\bm Z_{k}^{(i)} = \bU_k^{(i)}\bm{\Sigma}_k^{(i)}\bV_k^{(i)^T} = \begin{bmatrix}
\bU_{k,1}^{(i)} & \bU_{k,2}^{(i)} 
\end{bmatrix} \begin{bmatrix}
\tilde{\bm{\Sigma}}_k^{(i)} & \bm{0} \\
\bm{0} & \bm{0} 
\end{bmatrix} \begin{bmatrix}
\bV_{k,1}^{(i)^T} \\ \bV_{k,2}^{(i)^T} 
\end{bmatrix}
\end{align}
be a singular value decomposition (SVD) of $\bm Z_k^{(i)}$, where $\tilde{\bm{\Sigma}}_k^{(i)} \in \R^{r_{k,i}\times r_{k,i}}$ is a diagonal matrix with diagonal entries being positive singular values of $\bZ_k^{(i)}$; $\bU_k^{(i)} \in \mO^d$ with $\bU_{k,1}^{(i)} \in \R^{d\times r_{k,i}}$ and $\bU_{k,2}^{(i)} \in \R^{d \times (d-r_{k,i})}$; $\bm{V}_k^{(i)} \in \mO^{s_{k,i}}$ with $\bV_{k,1}^{(i)} \in \R^{s_{k,i}\times r_{k,i}}$ and $\bV_{k,2}^{(i)} \in \R^{s_{k,i}\times (s_{k,i}-r_{k,i})}$. This, together with $\bm Q_i\bm Q_i^T \bm Z_{k}^{(i)} = \bm Z_{k}^{(i)}$, implies for all $k \in [K]$ and $i \in [p]$,
\begin{align}\label{eq0:prop saddle}
\bm Q_i\bm Q_i^T \bU_{k,1}^{(i)} = \bU_{k,1}^{(i)}. 
\end{align}
According to  \eqref{eq:MCR}, \eqref{eq:grad R}, \eqref{eq:Zki}, and \eqref{eq:prop Rc}, we have for all $k \in [K]$ and $i \in [p]$,
\begin{align}\label{eq2:prop saddle}
 \alpha \bm X^{-1}\bm Z_{k}^{(i)}  - \alpha \left(\bm I + \alpha_k\bm Z_{k}^{(i)}\bm Z_{k}^{(i)^T}  \right)^{-1} \bm Z_{k}^{(i)} = \lambda \bm Z_{k}^{(i)}.
\end{align}
Substituting the block forms of $\bU_k^{(i)}$ and $\bm{\Sigma}_k^{(i)}$ in \eqref{eq:SVD Zk1} into the above equation and rearranging the terms, we obtain for all $k \in [K]$ and $i \in [p]$, 
\begin{align*} 
\bm X^{-1}\bm U_{k,1}^{(i)}  -  \bm U_{k,1}^{(i)} \left(\bm I + \alpha_k\tilde{\bm{\Sigma}}_k^{(i)^2}\right)^{-1} = \frac{\lambda}{\alpha}\bm U_{k,1}^{(i)}. 
\end{align*}
Using $\bm X = \bm I + \alpha \bm Z \bm Z^T$ and rearranging the terms, we have for all $k \in [K]$ and $i \in [p]$, 
\begin{align}\label{eq3:prop saddle}
\bm U_{k,1}^{(i)} \left( \left( 1- \frac{\lambda}{\alpha}\right) \bm I - \left(\bm I + \alpha_k\tilde{\bm \Sigma}_k^{(i)^2} \right)^{-1} \right) =  \alpha \bm Z\bm Z^T\bm U_{k,1}^{(i)} \left( \left(\bm I + \alpha_k\tilde{\bm \Sigma}_k^{(i)^2} \right)^{-1} + \frac{\lambda}{\alpha}\bm I \right).  
\end{align} 
Since $\bm Z\bm Z^T = \sum_{i=1}^p \lambda_i\bm Q_i\bm Q_i^T$, $\bm Q_i^T\bm Q_j = \bm 0$ for all $i \neq j$, and \eqref{eq0:prop saddle}, we have for all $k \in [K]$ and $i \in [p]$,
\begin{align*} 
\bm U_{k,1}^{(i)} \left( \left( 1- \frac{\lambda}{\alpha}\right) \bm I - \left(\bm I + \alpha_k\tilde{\bm \Sigma}_k^{(i)^2} \right)^{-1} \right) =  \alpha \lambda_i \bm U_{k,1}^{(i)} \left( \left(\bm I + \alpha_k\tilde{\bm \Sigma}_k^{(i)^2} \right)^{-1} + \frac{\lambda}{\alpha}\bm I \right).  
\end{align*}
Rearranging the terms in the above equation, we obtain for each $k \in [K]$ and $i \in [p]$,
\begin{align}\label{eq4:prop saddle}
\tilde{\bm \Sigma}_{k}^{(i)} = \beta_{i}  \bm I,\ \text{where}\ \beta_{i} := \frac{1}{\sqrt{\alpha K}}\sqrt{\frac{1+\alpha\lambda_i}{1-\lambda/\alpha-\lambda\lambda_i} -1}. 
\end{align} 
Substituting this back to \eqref{eq3:prop saddle} yields for each $k \in [K]$ and $i \in [p]$, 
\begin{align*}
\lambda_i \bm U_{k,1}^{(i)} = \bm Z\bm Z^T \bm U_{k,1}^{(i)}. 
\end{align*}
This, together with \eqref{eq:SVD Zk1} and \eqref{eq4:prop saddle}, yields $\lambda_i \bm Z_{k}^{(i)} = \bm Z\bm Z^T \bm Z_{k}^{(i)}$ for all $k \in [K]$ and $i \in [p]$. Using this and $\bm Z_k^{(i)^T}\bm Z_k^{(j)} = \bm 0$ for all $i \neq j$, we have for all $i \in [p]$ and $k \in [K]$, 
\begin{align*}
\lambda_i \bm Z_{k}^{(i)} = \sum_{l=1}^K \bm Z_{l}^{(i)}\bm Z_{l}^{(i)^T}  \bm Z_{k}^{(i)}.
\end{align*}
It follows from this and \eqref{eq:Zi} that 
\begin{align}\label{eq5:prop saddle}
\lambda_i \bm Z^{(i)}  = \bm Z^{(i)} \bm Z^{(i)^T} \bm Z^{(i)}.  
\end{align} 
Since there exists $k \neq l \in [K]$ such that $\bm Z_k^T\bm Z_l \neq \bm 0$, we can assume without loss of generality that $\bm Z_1^T\bm Z_2 \neq \bm 0$. Then, there exist $i_1 \in [m_1]$ and $i_2 \in [m_2]$ such that $\bm z_{1,i_1}^T\bm z_{2,i_2} \neq 0$. This, together with $\bm Z^{(i)^T}\bm Z^{(j)} = \bm 0$ for all $i \neq j$, implies that there exists $i^* \in [p]$ such that $\bm z_{1,i_1}, \bm z_{2,i_2}$ are both columns of $\bm Z^{(i^*)}$. Without loss of generality,  suppose that $\bm z_{1,i_1}$ and $\bm z_{2,i_2}$ are the $u$-th and $v$-th columns of $\bm Z^{(i^*)}$, respectively. Therefore, we have $\bm z_u^{(i^*)^T}\bm z_v^{(i^*)} \neq 0$. Using this, $\bm z_u^{(i^*)^T}\bm z_v^{(i^*)} \neq 0$, and \eqref{eq5:prop saddle}, we have 
\begin{align}\label{eq11:prop saddle}
\lambda_{i^*} \bm z_u^{(i^*)} = \bm Z^{(i^*)} \bm Z^{(i^*)^T} \bm z_u^{(i^*)}.
\end{align}
This is equivalent to
\begin{align}\label{eq10:prop saddle}
\sum_{j\neq u}^{s_{i^*}} \bm z_j^{(i^*)^T} \bm z_u^{(i^*)} \bm z_j^{(i^*)} + \left( \|\bm z_u^{(i^*)} \|^2 - \lambda \right) \bm z_u^{(i^*)} = \bm 0
\end{align}
This, together with $\bm z_u^{(i^*)^T}\bm z_v^{(i^*)} \neq 0$, implies that the columns of $\bm Z^{(i^*)}$ are linearly dependent. By letting $t_{i^*} = \mathrm{rank}(\bm Z^{(i^*)})$, we have $t_{i^*} < s_{i^*}$ due to linear dependence of columns of $\bm Z^{(i^*)}$. Then, let $\bm Z^{(i^*)} = \bm U \bm \Sigma \bm V^T$ be an SVD of $\bm Z^{(i^*)}$, where $\bm U \in \mO^{d \times t_{i^*}}$, $\bm \Sigma \in \R^{t_{i^*}\times t_{i^*}}$, and $\bm V \in \mO^{s_{i^*}\times t_{i^*}}$. Substituting this into \eqref{eq5:prop saddle} yields $\lambda_{i^*} \bm \Sigma  =  \bm\Sigma^3$, which implies $\bm \Sigma = \sqrt{\lambda_{i^*}} \bm I$ and 
\begin{align}\label{eq6:prop saddle}
\bm Z^{(i^*)} =  \sqrt{\lambda_{i^*}} \bm U \bm V^T. 
\end{align}   

\paragraph{Construct an ascent direction.}
For ease of exposition, we simply write $i^*$ as $i$ from now on. According to \eqref{eq:Zi} and \eqref{eq11:prop saddle}, we have 
\begin{align}\label{eq12:prop saddle}
\sum_{k=1}^K \bm Z_k^{(i)}\bm Z_k^{(i)^T} \bm z_u^{(i)} = \lambda_i\bm z_u^{(i)}.  
\end{align}
Recall that $\bm z_u^{(i)}$ and $\bm z_v^{(i)}$ are a column of $\bm Z_1$ and $\bm Z_2$, respectively. Without loss of generality, suppose that $\bm z_u^{(i)}$ are the first column of $\bm Z_1^{(i)}$, i.e., $\bm z_u^{(i)} = \bm Z_1^{(i)} \bm e_1$. Then, let $\bm c = \left( \bm c_1\ \dots\ \bm c_K \right) \in \R^{s_i}$ with $\bm c_1 = \bm Z_1^{(i)^T}\bm z_u^{(i)} - \lambda_i\bm e_1$ and $\bm c_k =  \bm Z_k^{(i)^T} \bm z_u^{(i)}$ for all $k \neq 1$. This, together with $\bm z_u^{(i)^T}\bm z_v^{(i)} \neq 0$ and \eqref{eq12:prop saddle}, implies $\bm c_2 \neq \bm 0$ and $\bm Z^{(i)}\bm c = \bm 0$.  Now, we set $\bm q_k := \bm V_{k,1}^{(i)}\bm V_{k,1}^{(i)^T}\bm c_k$ for each $k \in [K]$ and $\bm q := \left(\bm q_1\ \dots\ \bm q_K \right)$. According to $\bm Z_k^{(i)} = \beta_i \bm U_{k,1}^{(i)} \bm V_{k,1}^{(i)^T} $ by \eqref{eq:SVD Zk1} and \eqref{eq4:prop saddle}, we have for all $k \neq 1$, 
\begin{align*}
\bm q_k = \bm V_{k,1}^{(i)}\bm V_{k,1}^{(i)^T} \bm Z_k^{(i)^T} \bm z_u^{(i)} = \bm Z_k^{(i)^T} \bm z_u^{(i)} = \bm c_k. 
\end{align*}
Moreover, using $\bm Z_k^{(i)} = \beta_i \bm U_{k,1}^{(i)} \bm V_{k,1}^{(i)^T} $ by \eqref{eq:SVD Zk1} and \eqref{eq4:prop saddle}, we have 
\begin{align}\label{eq7:prop saddle}
\bm Z^{(i)}\bm q = \sum_{k=1}^K \bm Z_k^{(i)}\bm q_k = \beta_i \sum_{k=1}^K \bm U_{k,1}^{(i)} \bm V_{k,1}^{(i)^T} \bm c_k = \sum_{k=1}^K \bm Z_k^{(i)}\bm c_k = \bm Z^{(i)}\bm c = \bm 0
\end{align}
and 
\begin{align}\label{eq8:prop saddle}
\|\bm Z_k^{(i)}\bm q_k\| = \beta_i \| \bm V_{k,1}^{(i)^T}\bm c_k\| = \beta_i \| \bm q_k\|.  
\end{align}
Let $\bm u = \bm U \bm a$, where $\bm U$ is given in \eqref{eq6:prop saddle} and $\bm a \in \R^{t_i}$ is chosen such that $\bm a \in \mathrm{span}(\bm U^T\bm U_{k,2}^{(i)})$ and $\|\bm a\|=1$. We construct $\bm D = \left[\bm D^{(1)}\ \dots\ \bm D^{(p)}\right]$ with $\bm D^{(i)} = \bm u \bm q^T$ and $\bm D^{(j)} = \bm 0$ for all $j \neq i$.  

\paragraph{Compute the bilinear form of Hessian.} According to the construction of $\bm D$ and \eqref{eq7:prop saddle}, we compute $\bm Z \bm D^T = \bm Z^{(i)}\bm D^{(i)^T} = \bm Z^{(i)} \bm q \bm u^T = \bm 0$. 
This, together with \eqref{eq:Hess R} and \eqref{eq9:prop saddle}, yields 
\begin{align*}
\nabla^2 R(\bm Z)\left[ \bm D, \bm D \right] = \alpha  \bm a^T\bm U^T \left(\bm I + \alpha \bm Z^{(i)}\bm Z^{(i)^T} \right)^{-1} \bm U\bm a \|\bm q\|^2 = \frac{\alpha}{\alpha\lambda_i + 1}\|\bm q\|^2,
\end{align*}
where the last equality is due to \eqref{eq6:prop saddle}. With abuse of notation, let 
\begin{align*}
R_c\left(\bm Z_{k}^{(i)} \right) = \frac{m_k}{2m} \sum_{i=1}^p \log\det\left( \bm I_n + \alpha_k \bm Z_k^{(i)}\bm Z_k^{(i)^T} \right). 
\end{align*}
Since $\bm Z_k^{(i)} = \beta_i \bm U_{k,1}^{(i)} \bm V_{k,1}^{(i)^T}$ and $\bm D_k^{(i)} = \bm u \bm q_k^{T}$, we compute for each $k \in [K]$,
\begin{align*}
\nabla^2 R_c\left(\bm Z_{k}^{(i)} \right)\left[\bm D_{k}^{(i)},\bm D_{k}^{(i)} \right] & =  \alpha \|\bm q_k\|^2 \bm u^T\bm X_{k}^{(i)} \bm u - \alpha\alpha_k \left(\bm u^T \bm X_k^{(i)}\bm Z_k^{(i)} \bm q_k \right)^2 -  \alpha\alpha_k\left( \bm u^T \bm X_k^{(i)} \bm u \bm q_k^T\bm Z_k^{(i)^T}\bm X_k^{(i)}\bm Z_k^{(i)} \bm q_k \right)  \\
& =  \alpha \|\bm q_k\|^2\left( 1 - \frac{\alpha_k}{\alpha_k\beta_i^2 + 1}\|\bm Z_k^{(i)^T}\bm u\|^2 \right) - \frac{\alpha\alpha_k}{(\alpha_k\beta_i^2+1)^2}\left(\bm u^T \bm Z_k^{(i)} \bm q_{k} \right)^2 \\
&\quad  -  \alpha\alpha_k \left( 1 - \frac{\alpha_k}{\alpha_k\beta_i^2 + 1}\|\bm Z_k^{(i)^T}\bm u\|^2 \right)\frac{ \beta_i^2\| \bm q_{k}\|^2}{\alpha_k\beta_i^2+1},  
\end{align*}
where $\bm X_{k}^{(i)} = \left(\bm I + \alpha_k\bm Z_k^{(i)} \bm Z_k^{(i)^T}\right)^{-1}$, the second equality follows from 
$$ \bm X_{k}^{(i)} = \left(\bm I + \alpha_k\beta_i^2 \bm U_{k,1}^{(i)}\bm U_{k,1}^{(i)^T}\right)^{-1} = \bm I -  {\alpha_k\beta_i^2}\bm U_{k,1}^{(i)}\bm U_{k,1}^{(i)^T}/({\alpha_k\beta_i^2 + 1})$$ 
due to \eqref{eq:matrix inv Z} in Lemma \ref{lem:matrix inv}, $\bm Z_{k}^{(i)} = \beta_i\bm U_{k,1}^{(i)} \bm V_{k,1}^{(i)^T}$, and \eqref{eq8:prop saddle}. Summing up the above equality for all $k \in [K]$ with $\alpha_k = K\alpha$ for all $k \in [K]$ yields
\begin{align*}
&\quad \sum_{k=1}^K \nabla^2 R_c\left(\bm Z_{k}^{(i)} \right)\left[\bm D_{k}^{(i)},\bm D_{k}^{(i)} \right] \\ 
& = \alpha\left(1 - \frac{\alpha_k \beta_i^2}{\alpha_k\beta_i^2+1} \right)\|\bm q\|^2 - \frac{\alpha\alpha_k}{\alpha_k\beta_i^2+1} \sum_{k=1}^K\|\bm q_k\|^2 \|\bm Z_k^{(i)^T}\bm u\|^2  \\
&\quad -  \frac{\alpha\alpha_k}{(\alpha_k\beta_i^2+1)^2}\sum_{k=1}^K \left(\bm u^T \bm Z_k^{(i)} \bm q_{k} \right)^2 + \frac{\alpha\alpha_k^2\beta_i^2}{\left(\alpha_k\beta_i^2+1 \right)^2} \sum_{k=1}^K\|\bm q_k\|^2 \|\bm Z_k^{(i)^T}\bm u\|^2 \\
& = \left(\frac{\alpha}{1+\alpha\lambda_i} - \lambda\right)\|\bm q\|^2 - \frac{\alpha\alpha_k}{(\alpha_k\beta_i^2+1)^2} \sum_{k=1}^K \left(\left(\bm u^T \bm Z_k^{(i)} \bm q_{k} \right)^2 +  \|\bm q_k\|^2 \|\bm Z_k^{(i)^T}\bm u\|^2\right),
\end{align*}
where the second equality follows from the definition of $\beta_i$ in \eqref{eq4:prop saddle}. Finally, we compute
\begin{align*}
\nabla^2 F(\bm Z)[\bm D, \bm D] & = \nabla^2 R(\bm Z)[\bm D, \bm D] - \sum_{j=1}^K \nabla^2R_c(\bm z_j^{(i)})[\bm d_j^{(i)}, \bm d_j^{(i)}] - \lambda\|\bm q\|^2 \\
& =    \frac{\alpha\alpha_k}{(\alpha_k\beta_i^2+1)^2} \sum_{k=1}^K \left(\left(\bm u^T \bm Z_k^{(i)} \bm q_{k} \right)^2 +  \|\bm q_k\|^2 \|\bm Z_k^{(i)^T}\bm u\|^2\right) > 0,
\end{align*}
where the inequality is due to $\|\bm q_2\| = \|\bm c_2\| \neq 0$ and 
\begin{align*}
\|\bm Z_2^{(i)^T}\bm u\|^2 = \beta_2\|\bm U_{2,1}^{(i)^T} \bm U\bm a\| \neq 0
\end{align*} 
due to  $\bm a \in \mathrm{span}(\bm U^T\bm U_{k,2}^{(i)})$.  
\end{proof}

Given a matrix $\bm Z \in \R^{d\times m}$, let $ \bm Z\bm Z^T = \bm Q \bm \Lambda \bm Q^T $ be an eigenvalue decomposition of $\bm Z\bm Z^T \in \S_+^d$, where $\bm Q \in \mathcal{O}^{d \times r}$ and $\bm \Lambda \in \R^{r\times r}$ is a  diagonal matrix with diagonal entries being positive eigenvalues of $\bm Z\bm Z^T$. Suppose that $\bm Z\bm Z^T$ has $p$ distinct positive eigenvalues, where $1 \le p \le r$. Let $\lambda_1 > \dots > \lambda_p > 0$ be its distinct eigenvalue values with the corresponding multiplicities being $h_1,\dots,h_{p} \in \mathbb{N}_+$, respectively. Obviously, we have $\sum_{i=1}^p h_i = r$. Therefore, we write 
\begin{align*}
\bm \Lambda = \mathrm{BlkDiag}\left(\lambda_1\bm I_{h_1},\dots, \lambda_p\bm I_{h_p}\right),\ \bm Q = \begin{bmatrix}
\bm Q_1,\dots, \bm Q_p
\end{bmatrix}, 
\end{align*}
where $\bm Q_i \in \mathcal{O}^{d \times h_i}$ for all $i \in [p]$.


\section{Proofs in \Cref{sec:main}}\label{sec:pf main}

\subsection{Proof of \Cref{thm:1}}\label{sec:pf thm1}

\begin{proof}[Proof of Theorem \ref{thm:1}]
(i) Suppose that each block of $\bm Z$ satisfying \eqref{eq:Zk opti}, (a), (b), and (c). It directly follows from (i) in \Cref{prop:local max} that $\bm Z$ is a local maximizer. Conversely, suppose that $\bm Z$ is a local maximizer. According to \eqref{eq:crit}, \eqref{eq:Zc}, \eqref{eq:Z12}, (ii) in \Cref{prop:local max} and \Cref{prop:saddle}, if $\bm Z \in \mathcal{Z}^c \cup \mathcal{Z}_2$, then $\bm Z$ is a strict saddle point. This, together with $\mathcal{X} = \mathcal{Z}^c \cup \mathcal{Z}_1 \cup \mathcal{Z}_2$ and the fact that $\bm Z$ is a local maximizer, implies that $\bm Z \in \mathcal{Z}_1$. Using this, \eqref{eq:Z12}, and \Cref{prop:set Z} yields that $\bm Z$ satisfying \eqref{eq:Zk opti}, (a), (b), and (c). 

(ii) According to (i) in Theorem \ref{thm:1}, suppose that the $k$-th block of a local maximizer $\bm Z$ admits the decomposition in \eqref{eq:Zk opti} satisfying (a), (b), and (c) for all $k \in [K]$. This, together with \eqref{eq1:prop local max} in the proof of \Cref{prop:local max}, yields that 
\begin{align}\label{eq:F}
    F(\bm Z) & = \frac{1}{2} \sum_{k=1}^K \sum_{i=1}^{r_k} \left(  \log\left( 1 + \alpha\overline{\sigma}_{k}^2 \right) - \frac{m_k}{m} \log\left( 1 + \alpha_k\overline{\sigma}_k^2 \right) - \lambda \overline{\sigma}_k^2  \right),
\end{align}
where $\overline{\sigma}_k$ is defined in \eqref{eq:sigma1} for each $k \in [K]$. Then, we define a function $g: \mathbb{N}_+ \times \R \rightarrow \R$ as 
\begin{align*}
    g(n, x) := \log(1+\alpha x) - \frac{n}{m}\log\left(1 + \frac{d x}{n\varepsilon^2} \right) - \lambda x. 
\end{align*}
One can verify that for all $n_1 \ge n_2$, we have $g(n_1, x) \le g(n_2, x)$ for each $x$. Therefore, we have for all $m_k \le m_l$,
\begin{align*}
    g(m_l, \overline{\sigma}_{l}^2) \le  g(m_k, \overline{\sigma}_{l}^2) \le  g(m_k, \overline{\sigma}_{k}^2),
\end{align*}
where the second inequality follows from $\overline{\sigma}_k^2$ is the maximizer of the function $g(m_k,x) = h_k(x)$ according to \eqref{eq:h}. This, together with \eqref{eq:F}, yields that $\bm Z$ is a global maximizer if and only if $\sum_{k=1}^K r_k = \min\{m,d\}$ and for all $k \neq l$ satisfying $m_k < m_l$ and $r_l > 0$, we have $r_k = \min\{m_k,d\}$. 
\end{proof}

\subsection{Proof of \Cref{prop:equi}}\label{sec:pf prop}

To prove \Cref{prop:equi}, we first need to characterize the global optimal solution set of Problem \eqref{eq:MCR1}. 

\begin{proposition}\label{prop:opt cons}
    Suppose that $m_1=\dots=m_K$ and \eqref{eq:epsilon} holds. It holds that $\bm Z = [\bm Z_1,\dots,\bm Z_K] \in \R^{d\times m}$ with $\bm Z_k \in \R^{d\times m_k}$ for each $k \in [K]$ is a global solution of Problem \eqref{eq:MCR1} if and only if for each $k \in [K]$,  
    \begin{align}\label{eq:Zk equi}
       \bm Z_k = \frac{m}{\min\{m,d\}} \bm U_k\bm V_k^T,
    \end{align}
    where $r_k = \min\{m,d\}/K$ for all $k \in [K]$, $\bm U_k \in \mathcal{O}^{d \times r_k}$ with $\bm U_k^T\bm U_l = \bm 0$ for all $l \neq k$, and $\bm V_k \in \mathcal{O}^{m_k \times r_k}$ for all $k \in [K]$. 
\end{proposition}
\begin{proof}
    According to \Cref{lem:MCR}, we have 
    \begin{align}\label{eq1:prop equi}
        R(\bm Z) - \sum_{k=1}^K R_c\left(\bm Z; \bm \pi^k \right) &\le \frac{1}{2} \sum_{k=1}^K \log\det\left( \bm I + \alpha\bm Z_k\bm Z_k^T \right) -  \sum_{k=1}^K \frac{m_k}{2m} \log\det\left( \bm I + \alpha_k\bm Z_k\bm Z_k^T \right) \nonumber \\
        & =: H(\bm Z),
    \end{align}
    where the inequality becomes equality if and only if $\bm Z_k^T\bm Z_l = \bm 0$ for all $k \neq l$. To simplify our development, let $r_k := \mathrm{rank}(\bm Z_k)$ denote the rank of $\bm Z_k \in \R^{d\times m_k}$, where $r_k \le \min\{d,m_k\}$ for each $k \in [K]$ and $\sum_{k=1}^K r_k \le \min\{d,m\}$, and 
    \begin{align}\label{eq:hk}
        h_k(\bm Z_k) := \frac{1}{2}\log\det\left( \bm I + \alpha\bm Z_k\bm Z_k^T \right) - \frac{m_k}{2m} \log\det\left( \bm I + \alpha_k\bm Z_k\bm Z_k^T \right),\ \forall k \in [K].
    \end{align}
    Moreover, let 
    \begin{align*} 
    \bm Z_k = \bP_k\bm{\Sigma}_k\bQ_k^T = \begin{bmatrix}
    \bP_{k,1} & \bP_{k,2} 
    \end{bmatrix} \begin{bmatrix}
    \tilde{\bm{\Sigma}}_k & \bm{0} \\
    \bm{0} & \bm{0} 
    \end{bmatrix} \begin{bmatrix}
    \bQ_{k,1}^T \\ \bQ_{k,2}^T 
    \end{bmatrix}
    \end{align*}
    be a singular value decomposition (SVD) of $\bm Z_k$, where $\tilde{\bm{\Sigma}}_k = \mathrm{diag}(\sigma_{k,1},\dots,\sigma_{k,r_k})$ with $\sigma_{k,1}\ge\dots\ge \sigma_{k,r_k} > 0$ being positive singular values of $\bZ_k$, $\bP_k \in \mO^d$ with $\bP_{k,1} \in \R^{d \times r_k}$ and $\bP_{k,2} \in \R^{d \times (d-r_k)}$, and $\bm{Q}_k \in \mO^{m_k}$ with $\bQ_{k,1} \in \R^{m_k\times r_k}$ and $\bQ_{k,2} \in \R^{m_k\times (m_k-r_k)}$. Substituting this SVD into \eqref{eq:hk}, together with $\| \bm Z_k \|_F^2 = m_k$, yields that to maximize $H(\bm Z)$, it suffices to study  for each $k \in [K]$, 
    \begin{align*}
        \max_{\sigma_{k,1},\dots,\sigma_{k,r_k}} \sum_{i=1}^{r_k} \log\left( 1 + \alpha\sigma_{k,i}^2 \right) - \sum_{i=1}^{r_k}\frac{m_k}{m} \log\left( 1 + \alpha_k \sigma_{k,i}^2 \right)\qquad \mathrm{s.t.}\ \sum_{i=1}^{r_k} \sigma_{k,i}^2 = m_k. 
    \end{align*}
    To simplify our development, let $x_{i} := \sigma_{k,i}^2 \ge 0$ for each $i \in [r_k]$. This, together with $m_k=m/K$, implies that it suffices to study 
    \begin{align}\label{eq:gx}
        \max_{x_{1},\dots,x_{r_k}} g(\bm x) &:= \sum_{i=1}^{r_k} \log\left( 1 + \alpha x_{i} \right)  -  \sum_{i=1}^{r_k}\frac{1}{K} \log\left( 1 + \alpha_k x_{i} \right)\qquad \nonumber \\
        &\mathrm{s.t.}\ \sum_{i=1}^{r_k} x_{i} = \frac{m}{K},\ \ x_i \ge 0,\ \forall i \in [r_k].
    \end{align} 
    This, together with \Cref{lem:KKT} and \eqref{eq:alp}, yields that the optimal solution for each $k \in [K]$ is 
    \begin{align}\label{eq1:prop opt cons}
        x_i^* = \frac{m}{r_k K},\ \forall i \in [r_k]. 
    \end{align}
    (i) Suppose that $m \le d$. Then, we have $r_k \le m/K$ for each $k \in [K]$ and $\sum_{k=1}^K r_k \le m$. This, together with \eqref{eq1:prop opt cons} and \Cref{lem:hs}, implies that $r_k = m/K$, and thus $x_i^* = 1$ for all $i \in [r_k]$ and $k \in [k]$. 
    
    (ii) Suppose that $m > d$. Then, we have $r_k \le \min\{d,m/K\}$ for each $k \in [K]$ and $\sum_{k=1}^K r_k \le d$. To compute the optimal function value, we consider the following problem: 
    \begin{align*}
        \max_{r_1,\dots,r_K \in \mathbb{Z}} &\sum_{k=1}^K r_k \left( \log\left(1 + \frac{m\alpha}{r_k K}\right) - \frac{1}{K} \log\left(1 + \frac{m\alpha}{r_k}\right)   \right) \nonumber \\
        &\mathrm{s.t.}\ \sum_{k=1}^K r_k = d,\ r_k \le \min\left\{d,\frac{m}{K}\right\},\ \forall k \in [K]. 
    \end{align*}
    Now, we study the following function: 
    \begin{align*}
        \phi(x) :=  x\left(\log\left(1 + \frac{m\alpha}{Kx} \right)  - \frac{1}{K}\log\left( 1 + \frac{m\alpha}{x} \right)\right),\ \text{where}\ x \in \left[1, \frac{m}{K}\right] 
    \end{align*}
    We compute 
    \begin{align*}
    \phi^{\prime}(x) & =  \log\left(1 + \frac{m\alpha}{Kx}\right) - \frac{1}{K}\log\left( 1 + \frac{m\alpha}{x} \right)   - \frac{m\alpha}{Kx+m\alpha} + \frac{m\alpha}{K(x+m\alpha)}, \\
    \phi^{\prime\prime}(x) & = -\frac{m\alpha/x}{Kx+m\alpha} + \frac{m\alpha / x}{K(x+m\alpha)} + \frac{Km\alpha}{(Kx+m\alpha)^2} - \frac{m\alpha}{K(x+m\alpha)^2} \\
    &= -\frac{m^2\alpha^2/x}{(Kx+m\alpha)^2} + \frac{m^2\alpha^2/x}{K(x+m\alpha)^2}.
    \end{align*}
    Since $x \in [1,m/K]$, we have $Kx^2 \le m^2/K \le m^2\alpha^2$ when $\alpha \ge 1/\sqrt{K}$, and thus $\phi^{\prime\prime}(x) \le 0$. Therefore, $\phi(x)$ is a concave function for all $x \in [1,m/K]$. Then, applying the Jensen inequality yields
    \begin{align*}
        \sum_{k=1}^K \frac{1}{K} f(r_k) \le f\left(  \sum_{k=1}^K r_k\right),
    \end{align*}
    where the inequality becomes equality if and only if $r_1=\dots=r_k=d/K$. This, together with \eqref{eq1:prop opt cons}, yields $x_i^* = m/d$. 

    According to (i) and (ii), we have $x_i^* = m/\min\{m,d\}$ and $r_k=\min\{m,d\}/K$. Therefore, we have $\bm Z_k = m/\min\{m,d\} \bm P_{k,1}\bm Q_{k,1}^T$, where $\bm P_{k,1} \in \mathcal{O}^{d\times r_k}$ and $\bm V_k \in \mathcal{O}^{m_k\times r_k}$ for each $k \in [K]$. Then, we complete the proof. 
    
\end{proof}

Based on the above proposition, we are ready to prove \Cref{prop:equi}. 
\begin{proof}[Proof of \Cref{prop:equi}]
    Let $\hat{\bm Z} = [\bm Z_1,\dots,\bm Z_k]$ denote the optimal solution of Problem \eqref{eq:MCR1}. According to \Cref{prop:opt cons}, it suffices to study the following two cases. 
        
    (i) Suppose that $m < d$. Using this and \eqref{eq:Zk equi}, we have $\hat{\bm Z}_k = \bm U_k\bm V_k^T$ and $\hat{r}_k = m/K$ for each $k \in [K]$. Moreover, according to \Cref{thm:1}, if $m_k = m/K$ for each $k \in [K]$ and $\lambda$ satisfies \eqref{eq:lambda m<d}, one can verify that the global solutions of Problem \eqref{eq:MCR}  satisfy \eqref{eq:Zk opti} with $\overline{\sigma}_1 = \dots = \overline{\sigma}_K =  1$ and $\sum_{k=1}^K r_k = m$. Since $r_k \le m_k$ for each $k \in [K]$, we have $r_k = m/K$ for each $k \in [K]$. Therefore, Problem \eqref{eq:MCR1} and Problem \eqref{eq:MCR} have the same global solution set.

    (ii) Suppose that $m \ge d$. Using this and \eqref{eq:Zk equi}, we have $\hat{\bm Z}_k = m\bm U_k\bm V_k^T/d$ and $\hat{r}_k = d/K$ for each $k \in [K]$. Moreover, according to \Cref{thm:1}, if $m_k = m/K$ for each $k \in [K]$ and $\lambda$ satisfies \eqref{eq:lambda m>d}, one can verify that the global solutions of Problem \eqref{eq:MCR}  satisfy \eqref{eq:Zk opti} with $\overline{\sigma}_1 = \dots = \overline{\sigma}_K =  m/d$ and $\sum_{k=1}^K r_k = d$. Therefore, the global solution set of Problem \eqref{eq:MCR1} is a subset of that of Problem \eqref{eq:MCR}. 
\end{proof}

\begin{lemma}\label{lem:KKT}
Suppose that $m, K$ are integers such that $m/K$ is a positive integer, $r\le m/K$ is an integer, and $\alpha > 0$ is a constant. Consider the following optimization problem
\begin{align}\label{eq:Fx}
    \min_{\bm x \in \R^r} \sum_{i=1}^r -\log\left( 1 + \alpha x_i \right) + \sum_{i=1}^r \frac{1}{K} \log(1+K\alpha x_i) \quad \mathrm{s.t.}\ \sum_{i=1}^r x_i = \frac{m}{K},\ x_i \ge 0,\ \forall i \in [r]. 
\end{align}
If 
\begin{align}\label{eq:alp low}
    \alpha \ge 6 K^{\frac{1}{K-1}}\exp(1)\left(1 + \frac{1}{\sqrt{K}} \right)^{\frac{m}{K-1}},
\end{align}
the optimal solution is
\begin{align}\label{eq:opt x}
    x_i^* = \frac{m}{rK},\ \forall i \in [r].
\end{align}
\end{lemma}
\begin{proof}
If $r=1$, it is trivial to see that \eqref{eq:opt x} is the optimal solution. Therefore, it suffices to study $r \ge 2$. To simplify our development, let $f(x):= -\log(1+\alpha x) + \log(1+K\alpha x)/K$ and $F(\bm x) := \sum_{i=1}^r f(x_i)$. Then, one can verify that for all $x \ge 0$,  
\begin{align}\label{eq0:lem KKT}
f^\prime(x) = -\frac{\alpha}{1+\alpha x} + \frac{\alpha}{1+K\alpha x} < 0,\quad f^{\prime\prime}(x) = \frac{\alpha^2}{(1+\alpha x)^2} - \frac{K\alpha^2}{(1+K\alpha x)^2}. 
\end{align} 
Introducing dual variables $\lambda$ associated with the constraint $\sum_{i=1}^r x_i = m/K$ and $\mu_i$ associated with the constraint $x_i \ge 0$ for each $i \in [r]$, we write the Lagrangian as follows
\begin{align}\label{eq:Lag}
    \mathcal{L}(\bm x; \lambda, \bm \mu) = \sum_{i=1}^r f(x_i) + \lambda\left( \sum_{i=1}^r x_i -\frac{m}{K} \right) - \sum_{i=1}^r \mu_ix_i.
\end{align}
Then, we write the KKT system as follows:      
\begin{align}\label{eq:KKT}
 -\frac{\alpha}{1+\alpha x_i} + \frac{\alpha}{1+K\alpha x_i} + \lambda - \mu_i = 0,\ x_i \mu_i = 0,\ x_i \ge 0,\ \mu_i \ge 0,\ \forall i \in [r],\ \sum_{i=1}^r x_i = \frac{m}{K}. 
\end{align} 
Now, let $\mathcal{S} := \{i \in [r]: x_i > 0 \}$ denote the support set of a KKT point $\bm x \in \R^r$ and $s := |\mathcal{S}|$ denote the cardinality of the support set, where $1 \le s \le r$. This, together with \eqref{eq:KKT}, implies that for each $i \in \mathcal{S}$, 
\begin{align}\label{eq:KKT 1}
    -\frac{\alpha}{1+\alpha x_i} + \frac{\alpha}{1+K\alpha x_i} + \lambda = 0,\  \sum_{i \in \mathcal{S}} x_i = \frac{m}{K}.  
\end{align}
This is equivalent to the following quadratic equation:
\begin{align}\label{eq1:lem KKT}
   K\lambda \alpha x_i^2 - \left(  (K-1)\alpha - (K+1)\lambda \right)x_i + \frac{\lambda}{\alpha} = 0. 
\end{align}
We compute 
\begin{align}\label{eq:eta delta}
    \Delta = \eta^2 - 4K\lambda^2, \text{where}\ \eta := (K-1)\alpha - (K+1)\lambda.  
\end{align}
Note that for all $i \in \mathcal{S}$, we have $x_i > 0$, and thus $\mu_i = 0$. This, together with $K \ge  2$ and the first equation in \eqref{eq:KKT 1}, implies $\lambda > 0$.
Consequently, the quadratic equation \eqref{eq1:lem KKT} has a positive root if and only if $\eta \ge 0$ and $\Delta \ge 0$. This implies
\begin{align}\label{eq2:lem KKT}
0 < \lambda \le \frac{\sqrt{K} - 1}{\sqrt{K} + 1}\alpha. 
\end{align}
Then, the solution of Problem \eqref{eq1:lem KKT} is $x_i \in \{\overline{x}, \underline{x}\}$ for each $i \in \mathcal{S}$, where
\begin{align}\label{eq3:lem KKT}
\overline{x} = \frac{\eta + \sqrt{\Delta}}{2K\lambda\alpha},\quad \underline{x} = \frac{\eta - \sqrt{\Delta}}{2K\lambda\alpha}.
\end{align}
Now, we discuss the KKT points that could potentially be optimal solutions. Let $\bm x \in \R^r$ be a KKT point satisfying $x_i \in \{\overline{x}, \underline{x}\}$ for each $i \in \mathcal{S}$, where $s \in \{1,2,\dots,r\}$. In particular, when $s=1$, we have $x_i = m/K$ for all $i \in \mathcal{S}$. In the following, we consider $s \in \{2,\dots,r\}$.

{\bf Case 1}. Suppose that $x_i=x_j$ for all $i,j \in \mathcal{S}$. This, together with $\sum_{i\in \mathcal{S}} x_i = m/K$ and $x_i \in \{\overline{x}, \underline{x}\}$ for each $i \in \mathcal{S}$, yields 
\begin{align}\label{eq:solu 1}
x_i = \frac{m}{s K},\ \forall i \in \mathcal{S}. 
\end{align}

{\bf Case 2}. Suppose that there exists $i\neq j \in \mathcal{S}$ such that $x_i \neq x_j$. This, together with $x_i \in \{\overline{x}, \underline{x}\}$, implies $\overline{x} > \underline{x}$. According to \eqref{eq0:lem KKT}, we have $f^{\prime\prime}(x) = 0$ at $\hat{x} = 1/(\alpha\sqrt{K})$. Then, we obtain that $f^\prime(x)$ is strictly decreasing in $[0,\hat{x}]$ and strictly increasing in $[\hat{x},\infty]$. Then, one can further verify that $\underline{x} < \hat{x} < \overline{x}$. This, together with \eqref{eq0:lem KKT}, implies 
\begin{align}\label{eq4:lem KKT}
f^{\prime\prime}(\underline{x}) < 0,\ f^{\prime\prime}(\overline{x}) > 0. 
\end{align}
For ease of exposition, let $l(\bm x) = |\{i \in \mathcal{S}: x_i = \underline{x}\}|$ be the number of entries of $\bm x$ that equal to $\underline{x}$. Then, we claim that any optimal solution $\bm x^\star$ satisfies $l(\bm x^\star) \le 1$. Now, we prove this claim by contradiction. Without loss of generality, we assume that $x_i^\star = \overline{x}$ for all $i=1,\dots,r-l$ and $x_i^\star = \underline{x}$ for all $i=r-l+1,\dots,r$ with $l \ge 2$. This, together with \eqref{eq4:lem KKT} and $l \ge 2$, yields
\begin{align}\label{eq6:lem KKT}
f^{\prime\prime}(x_{r-1}^\star) < 0,\quad f^{\prime\prime}(x_r^\star) < 0. 
\end{align}
Using the second-order necessary condition for constraint optimization problems (see, e.g., \citep[Theorem 12.5]{nocedal1999numerical}) and $x_i^\star \ge 0$ for all $i \in [r]$, we obtain 
\begin{align}\label{eq5:lem KKT}
\sum_{i=1}^r f^{\prime\prime}(x_i^\star)v_i^2 \ge 0,\ \forall \bm v \in \R^r\ \text{s.t.}\ \sum_{i=1}^r v_i = 0. 
\end{align}
Then, we take $\bm v \in \R^r$ such that $v_1=\dots=v_{r-2} = 0$ and $v_{r-1}= -v_r \neq 0$. Substituting this into \eqref{eq5:lem KKT} yields
\begin{align*}
f^{\prime\prime}(x_{r-1}^\star) + f^{\prime\prime}(x_{r}^\star) \ge 0,
\end{align*}
which contradicts \eqref{eq6:lem KKT}. Therefore, $\bm x^\star$ cannot be an optimal solution. Then, we prove the claim. In this case, we can write the KKT point that could be an optimal solution as follows: There exists $i\in \mathcal{S}$ such that 
\begin{align}\label{eq:solu 2}
    x_i = \underline{x},\ x_j = \overline{x},\ \forall j \neq i. 
\end{align}

Now, we compute the function values for the above two cases, i.e., \eqref{eq:solu 1} and \eqref{eq:solu 2}, and compare them to determine which one is the optimal solution. To simplify our analysis, let $h(x) := f(x)/x$. For \eqref{eq:solu 1} in {\bf Case 1}, we compute
\begin{align}\label{eq:F1}
    F_1 =  s f\left(\frac{m}{sK}\right) = \frac{m}{K}h\left( \frac{m}{sK} \right). 
\end{align}
For \eqref{eq:solu 2} in {\bf Case 2}, we compute
\begin{align}\label{eq:F2}
    F_2 =  (s-1) f\left(\overline{x}\right) + f\left(\underline{x}\right) \ge (s-1)f\left(\frac{m}{(s-1)K}\right) + f(\underline{x}) = \frac{m}{K}h\left(\frac{m}{(s-1)K}\right) +  f(\underline{x}),
\end{align}
where the inequality is because $\overline{x} \le {m}/\left({(s-1)K}\right)$ and $f(x)$ is strictly dereasing in $[0,\infty)$. For all $x \in \left[{m/(sK)}, {m}/\left({(s-1)K}\right)\right]$, we compute 
\begin{align}\label{eq8:lem KKT}
    h^\prime(x) & = \frac{f^\prime(x)x - f(x)}{x^2} = -\frac{1}{x}\left( \frac{\alpha}{1+\alpha x} - \frac{\alpha}{1+K\alpha x} \right) + \frac{1}{x^2}\left( \log(1+\alpha x) - \frac{1}{K} \log(1+K\alpha x) \right) \notag\\
    & \ge -\frac{(K-1)\alpha^2}{(1+\alpha x)(1+K\alpha x)} + \frac{1}{x^2}\left(\log(1+\alpha) - \frac{1}{K}\log(1+\alpha K)\right) \notag\\
    & \ge -\frac{K - 1}{Kx^2}\left( 1 - \log(1+\alpha) + \frac{\log K}{K-1} \right) \ge -\frac{(K - 1)Ks^2}{m^2}\left( 1 + \frac{\log K}{K-1} - \log(1+\alpha) \right),
\end{align}
where the first inequality follows from $\log(1+\alpha x) -  \log(1+K\alpha x)/K \ge \log(1+\alpha) - \log(1+\alpha K)/K$ due to $x \ge m/(s K) \ge 1$, the second inequality uses $\log(1+\alpha) - \log(1+\alpha K)/K = (K-1)\log(1+\alpha)/K + \log\left( (1+\alpha)/(1+\alpha K) \right)/K \ge  \left((K-1)\log(1+\alpha) - \log K \right)/K$, and the last inequality is because of $x \ge {m}/(sK)$. According to \eqref{eq:alp low}, we have 
\begin{align}\label{eq12:lem KKT}
1 + \frac{\log K}{K-1} - \log(1+\alpha) = \log\left( \frac{{K}^{\frac{1}{K-1}}\exp(1)}{1+\alpha} \right)  < 0. 
\end{align}
Using the mean-value theorem, there exists $x \in \left({m/(sK)}, {m}/\left({(s-1)K}\right)\right)$ such that 
\begin{align}\label{eq11:lem KKT}
h\left( \frac{m}{(s-1)K} \right) - h\left( \frac{m}{sK} \right) = h^\prime(x)\left( \frac{m}{(s-1)K} 
 - \frac{m}{sK} \right) \ge \frac{(K - 1)s }{m(s-1)}\left( \log(1+\alpha) - 1 - \frac{\log K}{K-1} \right),
\end{align}
where the inequality follows from \eqref{eq8:lem KKT}. Now, we are devoted to bounding $f(\underline{x})$. According to \eqref{eq:eta delta} and \eqref{eq2:lem KKT}, we have
\begin{align}\label{eq9:lem KKT}
\underline{x} = \frac{\eta - \sqrt{\Delta}}{2K\lambda\alpha} = \frac{4K\lambda^2}{2K\lambda\alpha(\eta + \sqrt{\Delta})} \le \frac{2\lambda}{\alpha \eta}. 
\end{align}
This, together with the fact that $f(x)$ is decreasing in $(0,\infty)$, yields 
\begin{align*}
f(\underline{x}) & \ge f\left(  \frac{2\lambda}{\alpha \eta} \right) = -\log\left(1+\frac{2\lambda}{\eta}\right) + \frac{1}{K}\log\left(1+\frac{2K\lambda}{\eta} \right) \ge  -\log\left(1+\frac{2\lambda}{\eta}\right) \ge -\log\left(1 + \frac{1}{\sqrt{K}}\right),
\end{align*} 
where the last inequality uses $\eta = (K-1)\alpha - (K+1)\lambda \ge 2\sqrt{K}\lambda$ due to $(K-1)\alpha \ge (\sqrt{K}+1)^2\lambda$ by \eqref{eq2:lem KKT}. This, together with \eqref{eq:F1}, \eqref{eq:F2}, and \eqref{eq11:lem KKT}, yields
\begin{align*}
    F_2 - F_1 & = \frac{m}{K}\left( h\left( \frac{m}{(s-1)K} \right) - h\left( \frac{m}{sK} \right) \right) + f(\underline{x}) \\
    & \ge \frac{(K - 1)s }{m(s-1)}\left( \log(1+\alpha) - 1 - \frac{\log K}{K-1} \right)  -\log\left(1 + \frac{1}{\sqrt{K}}\right) \\
    & \ge \frac{K-1}{m} \log\left( \frac{1+\alpha}{{K}^{\frac{1}{K-1}}\exp(1)} \right) -\log\left(1 + \frac{1}{\sqrt{K}}\right) > 0, 
\end{align*}
where the last inequality follows from \eqref{eq:alp low}. This implies that the optimal solution takes the form of \eqref{eq:solu 1} for some $s \in [r]$. Consequently, the function value of \eqref{eq:solu 1} for each $s \in [r]$ is 
\begin{align*}
    s \left( -\log\left( 1 + \frac{\alpha m}{sK} \right) + \frac{1}{K}\log\left(1 + \frac{\alpha m}{s}\right)  \right). 
\end{align*}
This, together with \Cref{lem:hs}, implies that when the optimal solution takes the form of \eqref{eq:solu 1} with $s=r$, Problem \eqref{eq:Fx} achieves its global minimum. Then, we complete the proof. 
\end{proof}

\begin{lemma}\label{lem:hs}
Consider the setting in \Cref{lem:KKT} and the following function 
\begin{align}\label{eq:h(s)}
h(s) := s\left(\frac{1}{K} \log\left(1+\frac{m\alpha}{s}\right) - \log\left( 1 +  \frac{m\alpha}{sK} \right)\right),
\end{align}
where $s \in [1,r]$ and $\alpha$ satisfies \eqref{eq:alp low}. Then, $h(s)$ is decreasing in $s \in [1,r]$. 
\end{lemma} 
\begin{proof}
For ease of exposition, let $\beta := m\alpha$ and $x := 1/s \in [1/r, 1]$. According to \eqref{eq:alp low}, we have 
\begin{align*}
    \alpha \ge  6 K^{\frac{1}{K-1}}\exp(1)\left(1 + \frac{1}{\sqrt{K}} \right)^{\frac{2m}{K-1}} > 1 \ge  \frac{r\sqrt{K}}{m}.  
\end{align*}
This implies $\beta \ge r\sqrt{K}$. Then, we study 
\begin{align}
h(s) = g(x) = \frac{1}{x}\left( \frac{1}{K} \log\left( 1 + \beta x \right) - \log\left( 1 + \frac{\beta x}{K} \right)\right). 
\end{align}
Note that showing $h(s)$ is decreasing in $s \in [1,r]$ is equivalent to proving $g(x)$ is increasing in $x \in [1/r,1]$.
Now, we compute
\begin{align}\label{eq:g(x)}
g^\prime(x) & = \frac{1}{x}\left( \frac{\beta}{K(1+\beta x)} - \frac{\beta}{K+\beta x}  \right) -\frac{1}{x^2}\left( \frac{1}{K} \log\left( 1 + \beta x \right) - \log\left( 1 + \frac{\beta x}{K} \right) \right)  = -\frac{1}{x^2} \phi(x),
\end{align}
where 
\begin{align*}
\phi(x) := \frac{1}{K} \log\left( 1 + \beta x \right) - \log\left( 1 + \frac{\beta x}{K} \right) +    \beta x \left( \frac{1}{K+\beta x} - \frac{1}{K(1+\beta x)} \right). 
\end{align*}
Then, it suffices to show $\phi(x) \le 0$ for all $x \in [K/m,1]$ due to $1/r \ge K/m$. Towards this goal, we compute
\begin{align*}
\phi^\prime(x) & =  \frac{\beta}{K(1+\beta x)} - \frac{\beta}{K+\beta x} +  \beta \left( \frac{1}{K+\beta x} - \frac{1}{K(1+\beta x)} \right) +  \beta x \left( \frac{-\beta}{(K+\beta x)^2} + \frac{\beta}{K(1+\beta x)^2} \right)\\
& = -x\beta^2 \frac{(K-1)(\beta^2 x^2 - K)}{(K+\beta x)^2K^2(1+\beta x)^2} \le 0,
\end{align*}
where the inequality follows from $\beta^2 x^2 \ge \beta^2 K^2/m^2 \ge K$ due to $x \in [K/m,1]$ and $\beta = \alpha m > 2m > m/\sqrt{K}$. Therefore, $\phi(x)$ is decreasing in $[K/m,1]$. Next, we compute
\begin{align*}
\phi\left( \frac{K}{m} \right) & = \frac{1}{K} \log\left( 1 + \frac{\beta K}{m} \right) - \log\left( 1 + \frac{\beta}{m} \right) + \frac{\beta }{m} \left( \frac{1}{1 + {\beta }/{m}} - \frac{1}{1+ {\beta K}/{m}} \right) \\
& = \frac{1}{K} \log\left( 1 +\alpha K \right) - \log\left( 1 + \alpha \right) + \alpha \left( \frac{1}{1 + \alpha } - \frac{1}{1+ \alpha K} \right) \\
& \le \frac{1}{K} \log\left( 1 +\alpha K \right) - \log\left( 1 + \alpha \right) + 1 \le \frac{1}{2} \log\left( 1 + 2\alpha \right) - \log\left( 1 + \alpha \right) + 1 \le 0,
\end{align*}
where the second equality is due to $\beta=\alpha m$, the second inequality holds because $\log\left( 1 +\alpha K \right)/K$ is decreasing as $K \ge 2$ increases when $\alpha \ge 2$, and the last inequality follows from $\alpha \ge 15$ by \eqref{eq:alp low}. 
This, together with the fact that $\phi$ is decreasing in $[K/m,1]$, yields $\phi(x) \le \phi\left( {K}/{m} \right)  \le 0$. Using this and \eqref{eq:g(x)}, we obtain $g^\prime(x) \ge 0$ in  $[K/m,1]$. Therefore, $g(x)$ is increasing  in  $[K/m,1]$. Then, we complete the proof. 
\end{proof}

\subsection{Proof of and \Cref{thm:2}}\label{sec:pf thm2}

\begin{proof}[Proof of \Cref{thm:2}]
According to (i) of \Cref{prop:local max}, if $\bm Z$ is a critical point but not a local maximizer, we have $\bm Z \in \mathcal{Z}_2 \cup \mathcal{Z}^c$. This, together with (ii) of \Cref{prop:local max} and \Cref{prop:saddle}, yields that $\bm Z$ is a strict saddle point. 
\end{proof}

\section{Additional Experimental Setups and Results}\label{app sec:exp}

In this section, we provide additional implementation details and experimental results under different parameter settings for \Cref{subsec:exp1,subsec:exp2} in \Cref{app subsec:exp 1,app subsec:exp 2}, respectively.  

\subsection{Implementation Details and Additional Results in \Cref{subsec:exp1}}\label{app subsec:exp 1}

\paragraph{Training setups.}  In this subsection, we employ full-batch gradient descent (GD) for solving Problem \eqref{eq:MCR}. Here, we use the Gaussian distribution to randomly initialize GD. More precisely, we randomly generate an initial point $\bm Z^{(0)}$ whose entries are i.i.d. sampled from the standard normal distribution, i.e. $z_{ij}^{(0)} \overset{i.i.d.}{\sim} \mathcal{N}(0,1)$. We fix the learning rate of GD as $10^{-1}$ in the training. We terminate the algorithm when the gradient norm at some iterate is less $10^{-5}$. 

\begin{figure*}[t]
\begin{center}
	\begin{subfigure}{0.32\textwidth}
    	\includegraphics[width = 0.9\linewidth]{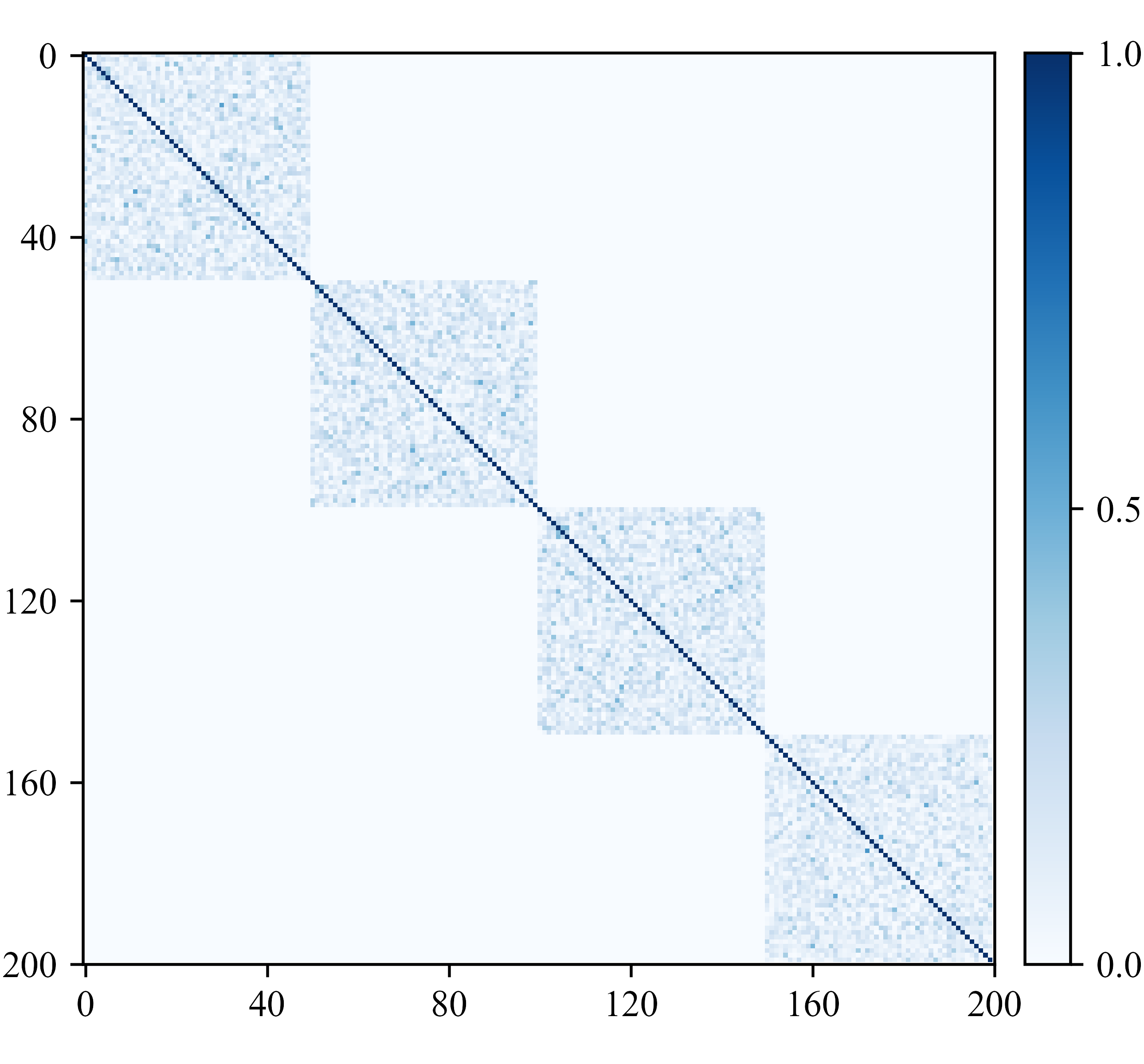}\vspace{-0.1in}
    \caption{Experiment 1} 
    \end{subfigure} 
    \begin{subfigure}{0.32\textwidth}
    	\includegraphics[width = 0.9\linewidth]{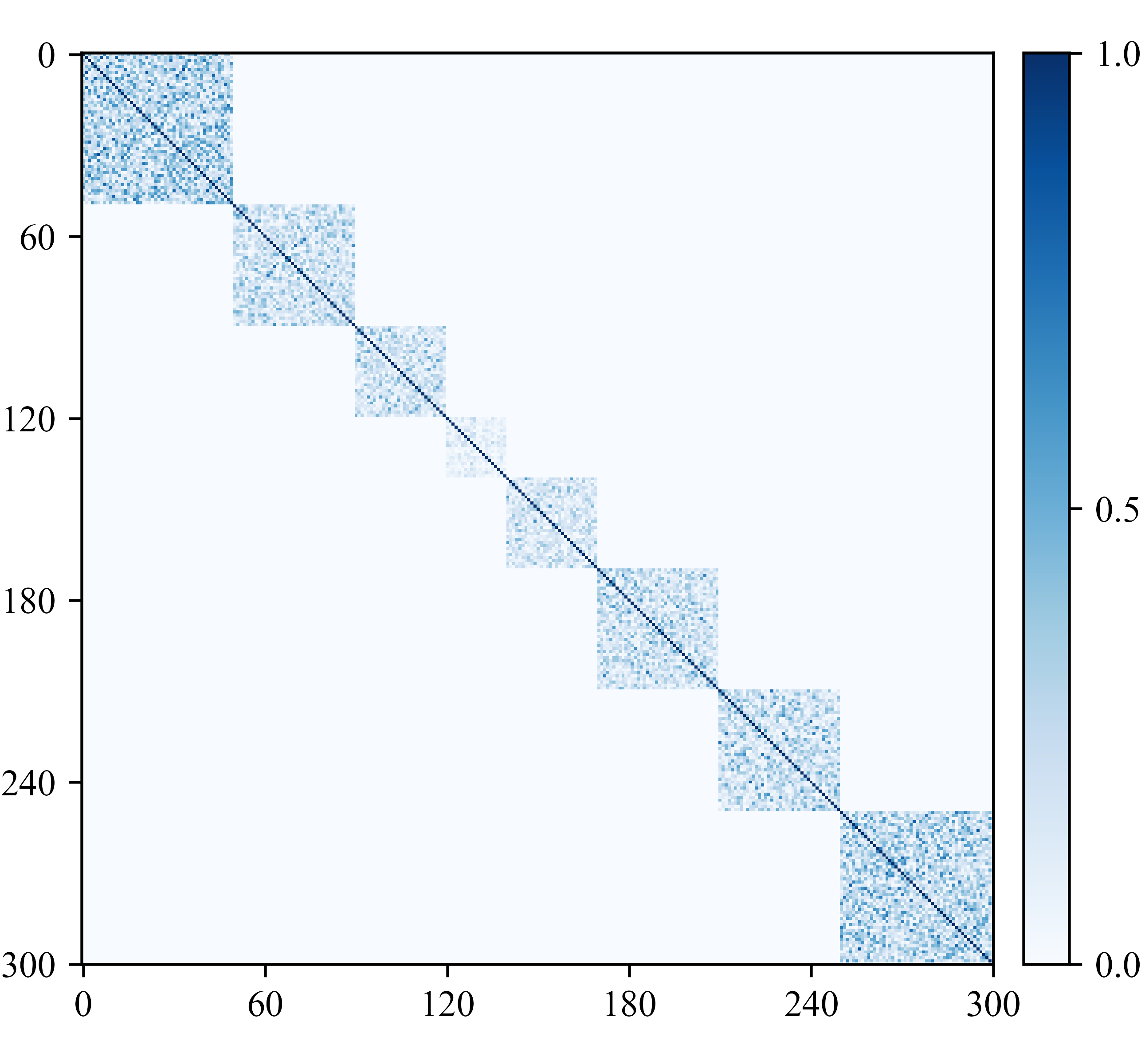}\vspace{-0.1in}
    \caption{Experiment 2} 
    \end{subfigure}\vspace{-0.05in} 
    \begin{subfigure}{0.32\textwidth}
    	\includegraphics[width = 0.9\linewidth]{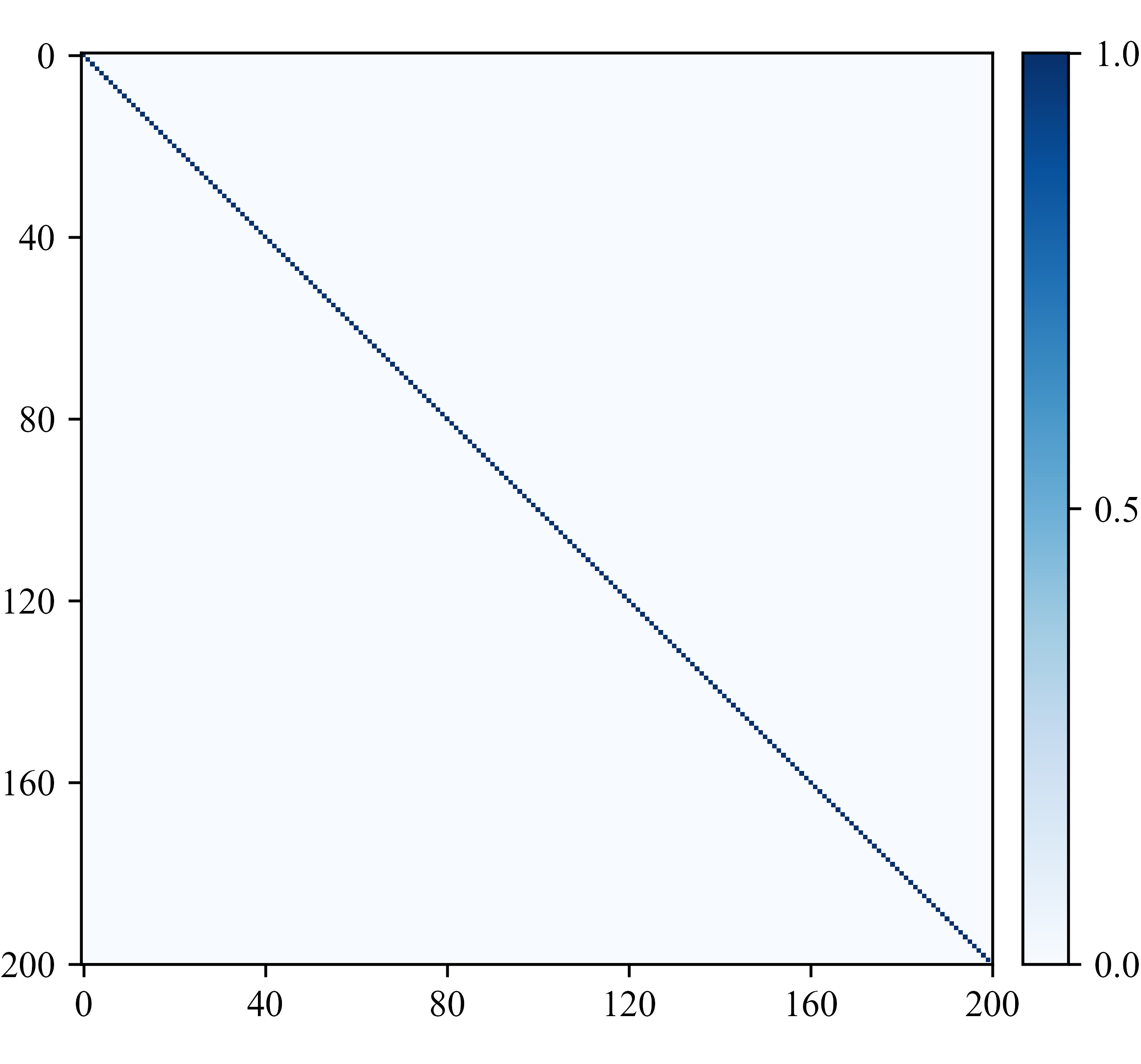}\vspace{-0.1in}
    \caption{Experiment 3} 
    \end{subfigure}\vspace{-0.05in} 
    \caption{\textbf{Heatmap of cosine similarity between pairwise features under different settings.}}\vspace{-0.15in} 
    \label{fig:heatmap}
\end{center}
\end{figure*}

In addition to the results presented in \Cref{subsec:exp1}, we perform additional experiments under different settings as follows. To support our theorems, we visualize the heatmap of learned features and plot the number and magnitude of singular values in each class. Unless specified otherwise, we use the training setups introduced above in the following experiments.  

\begin{itemize}
    \item {\bf Experiment 1 on balanced data.} In this experiment, we consider that the number of samples in each class is same and set the parameters in Problem \eqref{eq:MCR} as follows: the dimension of features $d = 100$, the number of classes $K = 4$, the number of samples in each class is $m_1=m_2=m_3=m_4=50$, the regularization parameter $\lambda=0.1$, and the quantization error $\epsilon = 0.5$. We visualize the heatmap between pairwise features of $\bm Z$ obtained by GD in \Cref{fig:heatmap}(a) and the number and magnitude of singular values in each class in \Cref{fig1:singu}. 
    \item {\bf Experiment 2 on data with more classes.} In this experiment, we set the parameters in Problem \eqref{eq:MCR} as follows: the dimension of features $d = 100$, the number of classes $K = 8$, the number of samples in each class is $m_1=50,m_2=40,m_3=30,m_4=20,m_5=30,m_6=40,m_7=40,m_8=50$, the regularization parameter $\lambda=0.1$, and the quantization error $\epsilon = 0.5$.  We visualize the heatmap between pairwise features of $\bm Z$ obtained by GD in \Cref{fig:heatmap}(b) and the number and magnitude of singular values in each class in \Cref{fig2:singu}. 
    \item {\bf Experiment 3 on data where the dimension $d$ is larger than the number of samples $m$.} In this experiment, we set the parameters in Problem \eqref{eq:MCR} as follows: the dimension of features $d = 300$, the number of classes $K = 4$, the number of samples in each class is $m_1=50,m_2=50,m_3=40,m_4=60$, the regularization parameter $\lambda=0.01$, and the quantization error $\epsilon = 5$. Note that in this experiment, we set the learning rate as $1$. We visualize the heatmap between pairwise features of $\bm Z$ obtained by GD in \Cref{fig:heatmap}(c) and the number and magnitude of singular values in each class in \Cref{fig3:singu}. 
\end{itemize}

 According to the results in Figures \ref{fig:heatmap}(a), \ref{fig:heatmap}(b), \ref{fig1:singu}, and \ref{fig2:singu} of Experiments 1 and 2, we observe that when the number of samples is larger than its dimension, i.e., $m \ge d$, the learned features via the MCR$^2$ principle are within-class compressible and between-class discriminative in both balanced and unbalanced data sets. Moreover, the dimension of the space spanned by these features is maximized such that $\sum_{k=1}^K r_k = \min\{m,d\}$. This directly supports \Cref{thm:1}. By comparing the function value returned by GD and that computed by the closed-form in \Cref{thm:1}, we found that GD with random initialization will always converge to a global maximizer of Problem \eqref{eq:MCR} when the data is balanced, while it will always converge to a local maximizer of Problem \eqref{eq:MCR} when data is unbalanced. This directly supports \Cref{thm:2}. 

 According to the results in Figures \ref{fig:heatmap}(c) and \ref{fig3:singu} of Experiment 3, we observe that when the number of samples is smaller than its dimension, i.e., $m \le d$, the learned features via the MCR$^2$ principle are orthogonal to each other and the dimension of each subspace is equal to the number of samples, i.e., $r_k = m_k$ for each $k \in [K]$. This exactly supports \Cref{thm:1}. Indeed, when $d \ge m$ and $r_k = m_k$ for each $k \in [K]$, it follows from \Cref{thm:1} that $\bm V_k = \bm I$ and thus $\bm Z_k = \overline{\sigma}_k \bm U_k$ for each $k \in [K]$ for each local maximizer. Therefore, this also supports \Cref{thm:2} as GD with random initialization converges to a local maximizer. 

\begin{figure*}[t]
    \centering
    \includegraphics[width = \linewidth]{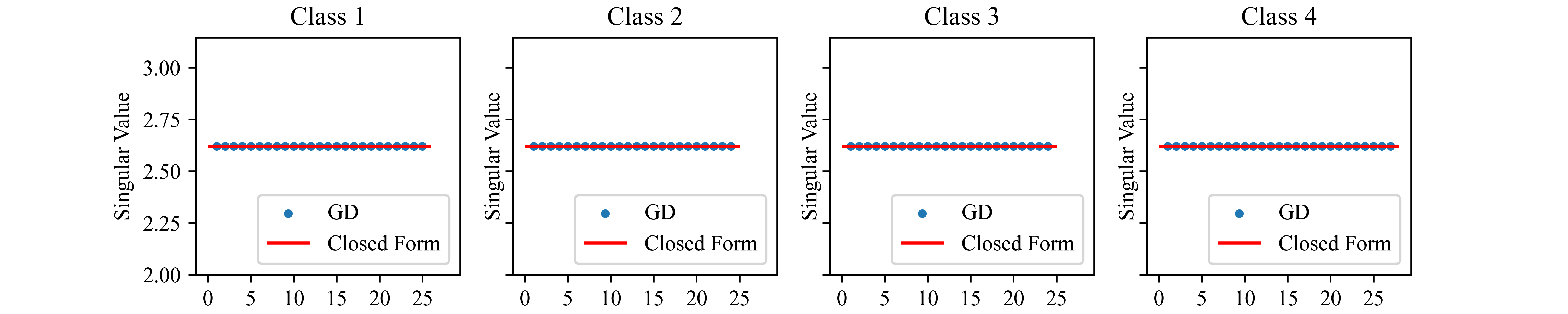}\vspace{-0.15in}
    \caption{\textbf{The number and magnitude of singular values in each subspace in Experiment 1.} The blue dots are plotted based on the singular values by applying SVD to the solution returned by GD, and the red line is plotted according to the closed-form solution in \eqref{eq:Zk opti}. The number of singular values in each subspace is $25, 24, 24, 27$, respectively.}\vspace{-0.05in} \label{fig1:singu}
\end{figure*}

\begin{figure*}[t]
    \centering
    \includegraphics[width = \linewidth]{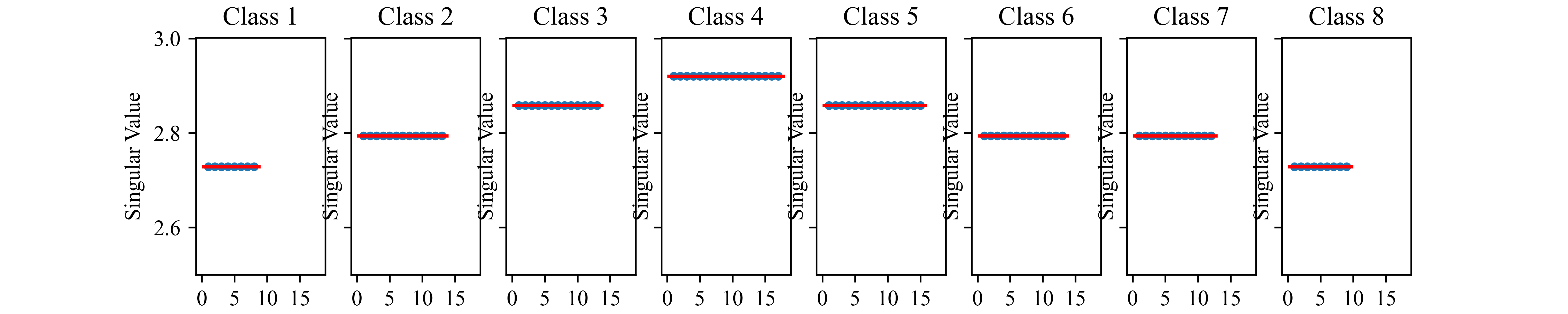}\vspace{-0.15in}
    \caption{\textbf{The number and magnitude of singular values in each subspace in Experiment 2.} The blue dots are plotted based on the singular values by applying SVD to the solution returned by GD, and the red line is plotted according to the closed-form solution in \eqref{eq:Zk opti}. The number of singular values in each subspace is $8, 13, 13, 17, 15, 13, 12, 9$, respectively.} \label{fig2:singu}
\end{figure*}

\begin{figure*}[!htbp]
    \centering
    \includegraphics[width = \linewidth]{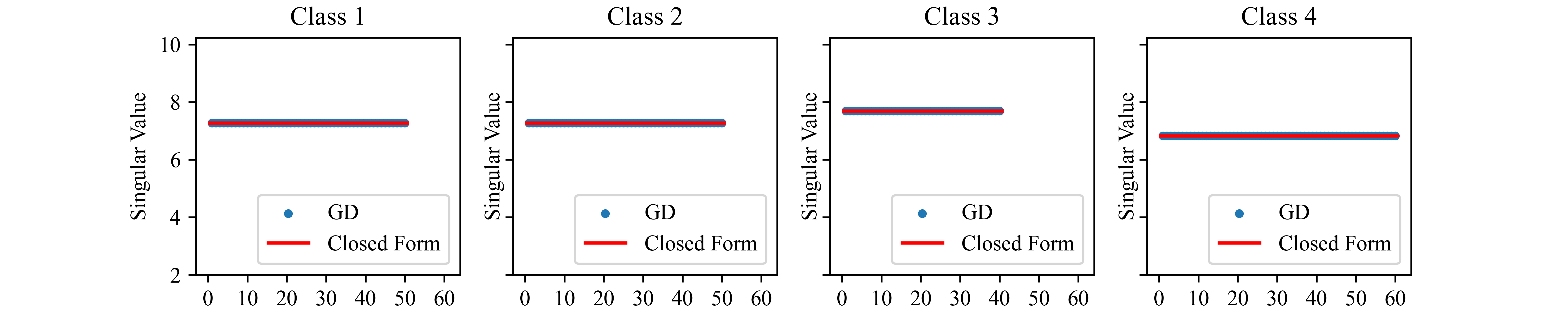}\vspace{-0.15in}
    \caption{\textbf{The number and magnitude of singular values in each subspace in Experiment 3.} The blue dots are plotted based on the singular values by applying SVD to the solution returned by GD, and the red line is plotted according to the closed-form solution in \eqref{eq:Zk opti}. The number of singular values in each subspace is $50, 50, 40, 60$, respectively.}\vspace{-0.05in} \label{fig3:singu}
\end{figure*}

\begin{figure*}[t]
\begin{center}
    \begin{subfigure}{0.24\textwidth}
    	\includegraphics[width = 1\linewidth]{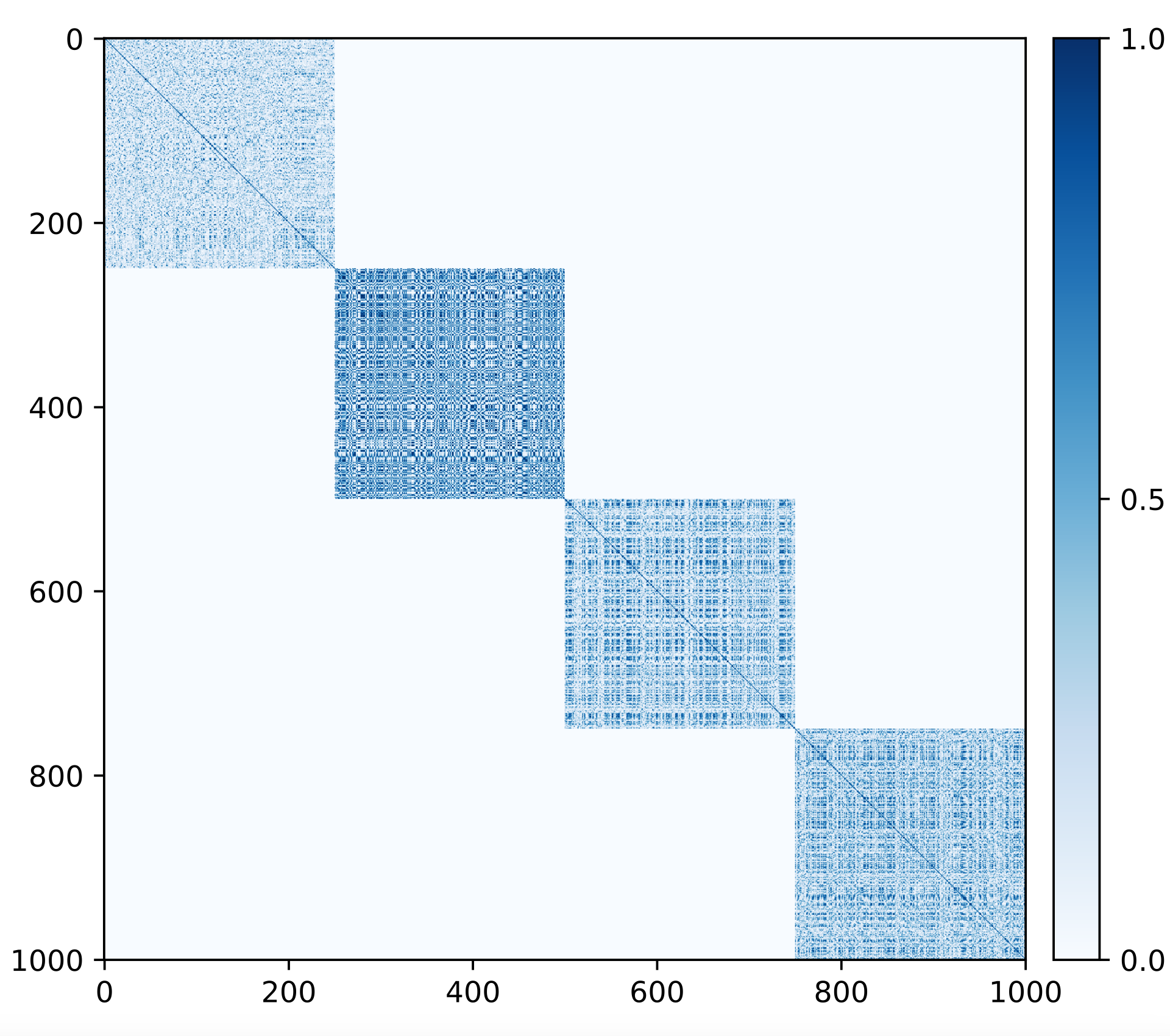}\vspace{-0.05in}
    \caption*{\footnotesize MNIST: $m=1000,K=4$} 
    \end{subfigure} 
    \begin{subfigure}{0.24\textwidth}
    	\includegraphics[width = 1\linewidth]{figures/MNIST_K=6_M=1500.png}\vspace{-0.05in}
    \caption*{\footnotesize MNIST: $m=1500,K=6$} 
    \end{subfigure}
    \begin{subfigure}{0.24\textwidth}
    	\includegraphics[width = 1\linewidth]{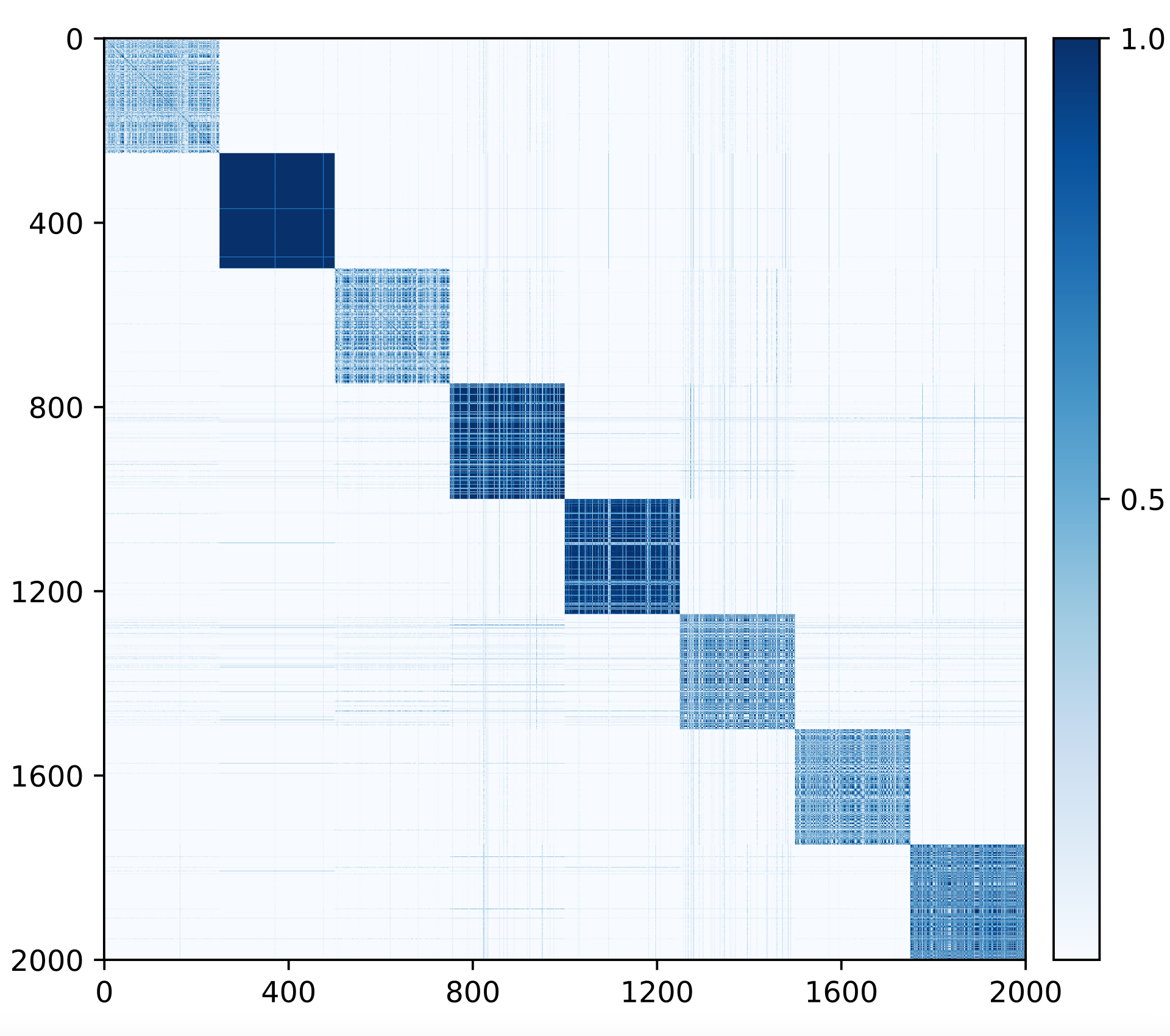}\vspace{-0.05in}
    \caption*{\footnotesize MNIST: $m=2000,K=8$} 
    \end{subfigure} 
    \begin{subfigure}{0.24\textwidth}
    	\includegraphics[width = 1\linewidth]{figures/MNIST_K=10_M=2500.png}\vspace{-0.05in}
    \caption*{\footnotesize MNIST: $m=2500,K=10$} 
    \end{subfigure}
    
    \begin{subfigure}{0.24\textwidth}
    	\includegraphics[width = 1\linewidth]{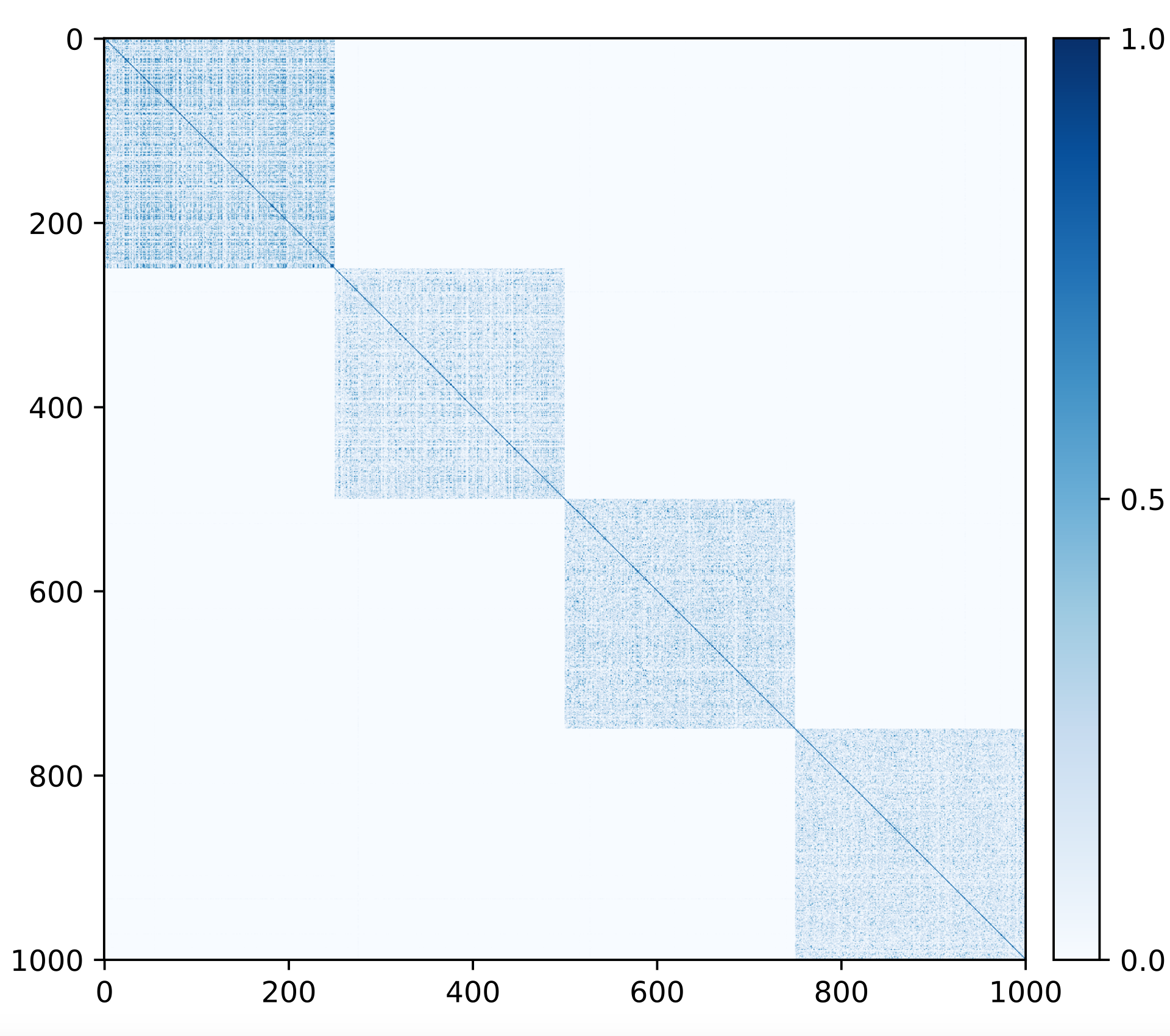}\vspace{-0.1in}
    \caption*{\footnotesize CIFAR: $m=1000,K=4$} 
    \end{subfigure} 
    \begin{subfigure}{0.24\textwidth}
    	\includegraphics[width = 1\linewidth]{figures/CIFAR10_K=6_M=1500.png}\vspace{-0.05in}
    \caption*{\footnotesize CIFAR: $m=1500,K=6$} 
    \end{subfigure}
    \begin{subfigure}{0.24\textwidth}
    	\includegraphics[width = 1\linewidth]{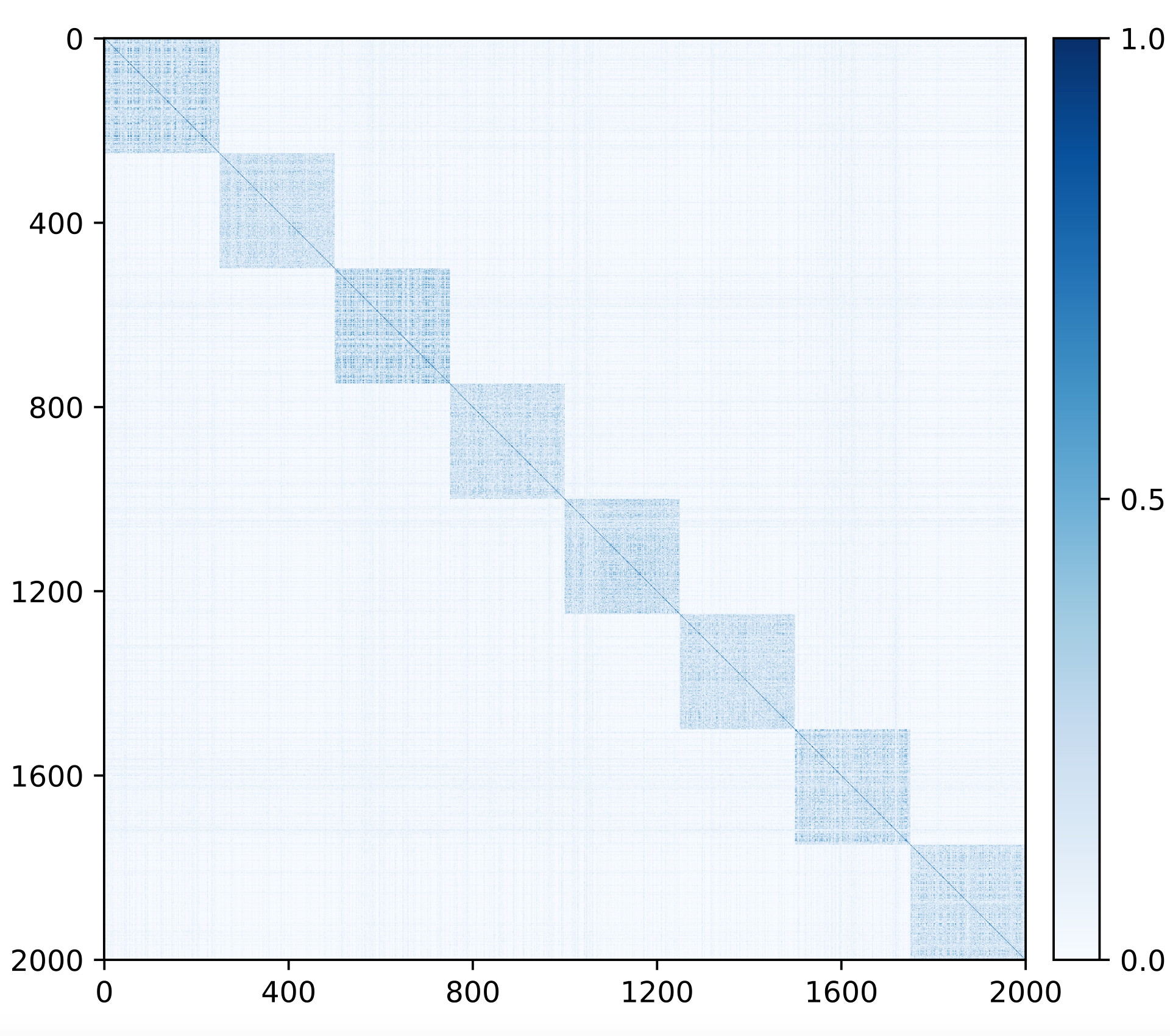}\vspace{-0.05in}
    \caption*{\footnotesize CIFAR: $m=2000,K=8$} 
    \end{subfigure} 
    \begin{subfigure}{0.24\textwidth}
    	\includegraphics[width = 1\linewidth]{figures/CIFAR10_K=10_M=2500.png}\vspace{-0.05in}
    \caption*{\footnotesize CIFAR: $m=2500,K=10$} 
    \end{subfigure}
    \caption{\textbf{Heatmap of cosine similarity among learned features by training deep networks on MNIST and CIFAR-10.} We train network parameters by optimizing the regularized MCR$^2$ objective \eqref{eq:MCR} on $m$ samples split equally among $K$ classes of MNIST and CIFAR-10. In the figure, the darker pixels represent higher cosine similarity between features. In particular, when the $(i,j)$-th pixel is close to $0$ (very light blue), the features $i$ and $j$ are approximately orthogonal.} \vspace{-0.15in} 
    \label{fig:5} 
\end{center}
\end{figure*}

\subsection{Implementation Details and Additional Results in \Cref{subsec:exp2}}\label{app subsec:exp 2}

\paragraph{Network architecture and training setups for MNIST.} In the experiments on MNIST, we employ a 4-layer multilayer perception (MLP) network with ReLU activation as the feature mapping with the intermediate dimension 2048 and output dimension 32. 
In particular, each layer of MLP networks consists of a linear layer and layer norm layer followed by ReLU activation in the implementation. We train the network parameters via Adam by optimizing the MCR$^2$ function. For the Adam settings, we use a momentum of 0.9, a full-batch size, and a dynamically adaptive learning rate initialized with $5\times 10^{-3}$, modulated by a CosineAnnealing learning rate scheduler \cite{loshchilov2016sgdr}. We terminate the algorithm when it reaches 3000 epochs. 

\paragraph{More experimental results on MNIST.} Besides the numerical results in \Cref{table:1}, we also plot the heatmap of the cosine similarity between pairwise columns of the features in $\bm Z$ obtained by training deep networks in \Cref{fig:5}. We observe that the features from different classes are nearly orthogonal to each other, while those from the same classes are highly correlated. This supports our results in \Cref{thm:1}.

\paragraph{Network architecture and training setups for CIFAR-10.} In the experiments on CIFAR10, we employ a 2-layer multilayer perceptron (MLP) network with ReLU activation as the feature mapping with the intermediate dimension 4096 and output dimension 128. 
In particular, each layer of MLP networks consists of a linear layer and layer norm layer followed by ReLU activation in the implementation. We train the network parameters via Adam by optimizing the MCR$^2$ function. For the Adam settings, we use a momentum of 0.9, a full-batch size, and a dynamically adaptive learning rate initialized with $5\times 10^{-3}$, modulated by a CosineAnnealing learning rate scheduler \cite{loshchilov2016sgdr}. We terminate the algorithm when it reaches 4000 epochs. 

\paragraph{More experimental results on CIFAR-10.} Besides the numerical results in \Cref{table:1}, we also plot the heatmap of the cosine similarity between pairwise columns of the features in $\bm Z$ obtained by training deep networks in \Cref{fig:5}. We observe that the features from different classes are nearly orthogonal to each other, while those from the same classes are highly correlated. This supports our results in \Cref{thm:1}.

\end{appendix}

\end{document}